\documentclass{aamas2016}

\usepackage{graphicx} 
\usepackage{caption}

\usepackage{amsmath, amssymb}

\usepackage{algorithm}
\usepackage{algpseudocode}

\input{Definitions.tex}

\usepackage[colorlinks,linkcolor=red,citecolor=blue,urlcolor=blue]{hyperref}

\title{Investigating Practical Linear Temporal Difference Learning \!\!\!}

\numberofauthors{2}

\author{
\alignauthor
Adam White\\
       \affaddr{Department of Computer Science}\\
       \affaddr{Indiana University}\\
       \affaddr{Bloomington, IN 47405, USA}\\
       \email{adamw@indiana.edu}
\alignauthor
Martha White\\
       \affaddr{Department of Computer Science}\\
       \affaddr{Indiana University}\\
       \affaddr{Bloomington, IN 47405, USA}\\
       \email{martha@indiana.edu}
       }

\newcommand\defeq{\mathrel{\overset{\makebox[0pt]{\mbox{\normalfont\tiny\sffamily def}}}{=}}}

\newcommand{\dmu}{{d^\mu}}

\newcommand{\qw}{\mathbf w}
\newcommand{\E}{\mathbb{E}}
\newcommand{\fs}{\mathbf{x}}
\newcommand{\h}{\mathbf h}
\newcommand{\vx}{\mathbf x}
\def\e{{\mathbf e}}
\newcommand{\emu}{\e^\mu}

\newcommand{\A}{\mathcal{A}}
\newcommand{\Ss}{\mathcal{S}}
\newcommand{\R}{\mathbb{R}}

\newcommand{\MSPBE}{{\tt MSPBE}}

\newcommand{\citet}[1]{\cite{#1}}

\begin{document}

\maketitle

\begin{abstract}
Off-policy reinforcement learning has many applications including: learning from demonstration, learning multiple goal seeking policies in parallel, and representing predictive knowledge. Recently there has been an proliferation of new policy-evaluation algorithms that fill a longstanding algorithmic void in reinforcement learning: combining robustness to off-policy sampling, function approximation, linear complexity, and temporal difference (TD) updates. 
This paper contains two main contributions. 
First, we derive two new hybrid TD policy-evaluation algorithms, which fill a gap in this collection of algorithms.
Second, we perform an empirical comparison to elicit which of these new linear TD methods should be preferred in different situations, and make concrete suggestions about practical use. 
%
\end{abstract}

\keywords{Reinforcement learning; temporal difference learning; off-policy learning}
\section{Introduction}
Until recently, using temporal difference (TD) methods to approximate a value function from off-policy samples was potentially unstable without resorting to quadratic (in the number of features) computation and storage, even in the case of linear approximations. Off-policy learning involves learning an estimate of total future reward that we would expect to observe if the agent followed some target policy, while learning from samples generated by a different behavior policy. This off-policy, policy-evaluation problem, when combined with a policy improvement step, can can be used to model many different learning scenarios, such as learning from many policies in parallel \cite{sutton2011horde}, learning from demonstrations \cite{argall2009asurvey}, learning from batch data \cite{lin1992self}, or simply learning about the optimal policy while following an exploratory policy, as in the case of Q-learning \cite{watkins1989learning}. In this paper, we focus exclusively on the off-policy, policy evaluation problem, commonly referred to as value function approximation or simply the {\em prediction problem}. Over the past decade there has been an proliferation of new linear-complexity, policy-evaluation methods designed to be convergent in the off-policy case.

These novel algorithmic contributions have focused on different ways of achieving stable off-policy prediction learning.
The first such methods were the gradient TD family of algorithms that perform approximate stochastic gradient descent on the mean squared projected Bellman error (MSPBE).
The primary drawback of these methods is the requirement for a second set of learned weights, a second step size parameter, and potentially high variance updates due to importance sampling. Empirically the results have been mixed,  
with some results indicating that TD can be superior in on-policy settings\cite{sutton2009fast}, and others concluding the exact opposite \cite{dann2014policy}.

Later, provisional TD (PTD) was introduced \cite{sutton2014anew} to rectify the issue that the bootstrap parameter $\lambda$, used in gradient TD methods\cite{maei2011gradient} does not correspond well with the same parameter used by conventional TD learning \cite{sutton1988learning}. Specifically, 
for $\lambda=1$, 
gradient TD methods do not correspond to any known variant of off-policy Monte Carlo. The PTD algorithm fixes this issue, and in on-policy prediction is exactly equivalent to the conventional TD algorithm. PTD does not use gradient corrections, and is only guaranteed to converge in the tabular off-policy prediction setting. Its empirical performance relative to TD and gradient TD, however, is completely unknown.

Recently Sutton et al.~\citet{sutton2015anemphatic} observed that conventional TD does not correct its update based on the notion of a follow-on distribution.
This distributional mis-match provides another way to understand the off-policy divergence of conventional off-policy TD. 
They derive the Emphatic TD (ETD) algorithm that surprisingly achieves convergence \cite{yu2015onconvergence} without the need for a second set of weights, like those used by gradient TD methods. Like gradient TD methods, however, it seems that this algorithm also suffers from high variance due to importance sampling. Hallak et al.~\cite{hallak2015generalized} introduced a variant ETD that utilizes a scaling parameter $\beta$, which is meant to reduce the magnitude of the follow-on trace. Comparative empirical studies for ETD and ETD($\beta$), however,  have been limited. 

The most recent contribution to this line of work explores a mirror-prox approach to minimizing the MSPBE \cite{mahadevan2012sparse,mahadevan2014proximal,liu2015finite}. The main benefit of this work was that it enabled the first finite sample analysis of an off-policy TD-based method with function approximation, and the application of advanced stochastic gradient optimizations. Liu et al. \cite{liu2015finite} introduced two mirror-prox TD algorithms, one based on the GTD2 algorithm \cite{sutton2009fast} the other based on TDC \cite{sutton2009fast}\footnote{The GTD2 and TDC algorithms are gradient TD methods that do not use eligibility traces; $\lambda = 0$.} 
and showed that these methods outperform their base counter-parts on Baird's counterexample \cite{baird1995residual}, but did not extend their new methods with eligibility traces.

A less widely known approach to the off-policy prediction problem is based on algorithms that perform precisely TD updates when the data is sampled on-policy, and corrected gradient-TD style updates when the data is generated off-policy. The idea is to exploit the supposed superior efficiency of TD in on-policy learning, while maintaining robustness in the off-policy case. These ``hybrid" TD methods were introduced for state value-function based prediction \cite{maei2011gradient}, and state-action value-function based prediction \cite{hackman2012thesis}, but have not been extended to utilize eligibility traces, nor compared with the recent developments in linear off-policy TD learning (many developed since 2014).

Meanwhile a separate but related thread of algorithmic development has sought to improve the operation of eligibility traces used in both on- and off-policy TD algorithms. This direction is based on another nonequivalence observation: the update performed by the forward view variant of the conventional TD is only equivalent to its backward view update at the end of sampling trajectories. 
The proposed true-online TD (TO-TD) prediction algorithm \cite{vanseijen2014true}, and true-online GTD (TO-GTD) prediction algorithm \cite{vanhasselt2014off}
remedy this issue, and have been shown to outperform conventional TD and gradient TD methods respectively on chain domains.
The TO-TD algorithm requires only a modest increase in computational complexity over TD, however, the
TO-GTD algorithm is significantly more complex to implement and requires three eligibility traces compared to GTD. Nevertheless, both TO-TD and TO-GTD achieve linear complexity, and can be implemented in a completely incremental way.
  
Although there asymptotic convergence properties of many of these methods has been rigorously characterized, but empirically there is still much we do not understand about this now large collection of methods. A frequent criticism of gradient TD methods, for example, is that they are hard to tune and not well-understood empirically. It is somewhat disappointing that perhaps the most famous application of reinforcement learning---learning to play Atari games \cite{mnih2015human}---uses potentially divergent off-policy Q-learning. In addition, we have very little understanding of how these methods compare in terms of learning speed, robustness, and parameter sensitivity.   
By clarifying some of the empirical properties of these algorithms, we hope 
to promote more wide-spread adoption of these theoretically sound and computationally efficient algorithms.

This paper has two primary contributions. First, we introduce a novel extension of hybrid methods to eligibility traces resulting in two new algorithms,
HTD($\lambda$) and true-online HTD($\lambda$). Second, we provide an empirical study of TD-based prediction learning with linear function approximation. The conclusions of our experiments are surprisingly clear:
\begin{enumerate}[leftmargin=0.2cm,itemindent=.5cm,labelwidth=\itemindent,labelsep=-0.1cm,align=left,itemsep=-.2cm]
\item GTD($\lambda$) and TO-GTD($\lambda$) should be preferred if robustness to off-policy sampling is required \\ 
\item Between the two GTD($\lambda$) should be preferred if computation time is at a premium\\
\item Otherwise, TO-ETD($\lambda,\beta$) was clearly the best across our experiments except on Baird's counterexample.
\end{enumerate}

\section{Background}
This paper investigates the problem of estimating the discounted sum of future rewards {\em online} and with function approximation. In the context of reinforcement learning we take online to mean that the agent makes decisions, the environment produces outcomes, and the agent updates its parameters in a continual, real-time interaction stream. We model the agent's interaction as Markov decision process defined by a countably infinite set of states $\Ss$, a finite set of actions $\A$, and a scalar discount function $\gamma: \Ss \rightarrow \R$. The agent's observation of the current situation is summarized by the feature vector $\vx(S_t) \in \R^d$, where $S_t \in \Ss$ is the current state and $d\ll|\Ss|$. On each time step $t$, the agent selects an action according to it's {\em behavior policy} $A_t \sim \mu(S_t, \cdot)$, where $\mu: \Ss \times \A\rightarrow[0,1]$.
The environment then transitions into a new state $S_{t+1} \sim P(S_t, A_t, \cdot)$, and emits a scalar reward $R_{t+1} \in \R$. The agent's objective is to evaluate a fixed {\em target policy} $\pi: \Ss \times \A\rightarrow[0,1]$, or estimate the expected return for policy $\pi$: 
\begin{align*}
v^\pi(s) &\defeq \E [G_t | S_t = s, A_t \sim \pi]\\ 
\text{ for return } G_t &\defeq \sum_{i=0}^\infty \left(\prod_{j=1}^{i} \gamma_{t+j}\right) R_{t+i + 1} \hspace{0.5cm} \triangleright \gamma_j \defeq \gamma(s_j)
.
\end{align*}
where $v^\pi(s)$ is called the {\em state-value function} for policy $\pi$.

All the methods evaluated in this study perform temporal difference updates, and most utilize eligibility traces. The TD($\lambda$) algorithm is the prototypical example of these concepts and is useful for understanding all the other algorithms discussed in the remainder of this paper. TD($\lambda$) estimates $v^\pi$ as a linear function of the weight vector $\qw \in \R^d$, where the estimate 
is formed as an inner product between the weight vector and the features of the current state: 
$\qw^\top \vx(s) \approx v^\pi(s)$. The algorithm maintains a memory trace of recently experienced features, called the eligibility trace $\e\in\R^d$, allowing updates to assign credit to previously visited states. The TD($\lambda$) algorithm requires linear computation and storage $O(d)$, and can be implemented incrementally as follows:
\begin{align*}
\delta_t &\leftarrow R_{t+1} + \gamma_{t+1} \qw_t^\top \vx(S_{t+1}) - \qw_t^\top \vx(S_{t})\\  
\e_t &\leftarrow \lambda_t \gamma_t \e_{t-1} + \fs(S_t)\\
\Delta \qw &\leftarrow \alpha\delta_t\e_t \hspace{2.0cm} \triangleright \qw_{t+1} \leftarrow \qw_t + \Delta \qw
.
\end{align*}

In the case when the data is generated by a behavior policy, $\mu$, 
with $\pi \neq \mu$, 
we say that the data is generated off-policy. In the off-policy setting we must estimate $v^\pi$ with samples generated by selecting actions according to $\mu$. This setting can cause the TD($\lambda$) algorithm to diverge. The GTD($\lambda$) algorithm solves the divergence issue by minimizing the MSPBE, resulting in a stochastic gradient descent algorithm that looks similar to TD($\lambda$), with some important differences. GTD($\lambda$) uses importance weights, $\rho_t \defeq \frac{\pi(s,a)}{\mu(s,a)}\in \R$ in the eligibility trace to reweight the data and obtain an unbiased estimate of $\E[G_t]$. Note, in the policy iteration case---not studied here---it is still reasonable to assume knowledge of $\pi(s,a)$ for all $s\in \Ss,a\in \A$; for example when $\pi$ is near greedy with respect to the current estimate of the state-action value function. The GTD($\lambda$) has a auxiliary set of learned weights, $\h \in \R^d$, in addition to the primary weights $\qw$, 
which maintain a quasi-stationary estimate of a part of the MSPBE.
Like the TD($\lambda$) algorithm, GTD($\lambda$) requires only linear computation and storage and can be implemented fully incrementally as follows: 
\begin{align*}\vspace{-1cm}
\delta_t &\leftarrow R_{t+1} + \gamma_{t+1} \qw_t^\top \vx(S_{t+1}) - \qw_t^\top \vx(S_{t})\\  
\e_t &\leftarrow \rho_t(\lambda_t \gamma_t \e_{t-1} + \fs(S_t)) \hspace{.75cm} \triangleright {\text{weighted by}~\rho_t}\\
\Delta \qw &\leftarrow \alpha\delta_t\e_t  - \underbrace{\alpha\gamma_{t+1}(1-\lambda_{t+1})(\e_t^\top \h_t) \vx(S_{t+1})}_{\text{correction term}}\\
\Delta \h &\leftarrow \alpha_\h[ \delta_t\e_t  - (\vx(S_{t})^\top \h_t) \vx(S_{t}) ] \hspace{.75cm} \triangleright {\text{auxiliary weights}}
\vspace{-1cm}
\end{align*}
The auxiliary weights also make use of a step-size parameter, $\alpha_\h$ which is usually not equal to $\alpha$.

Due to space constraints we do not describe the other TD-based linear learning algorithms found in the literature and investigated in our study. We provide each algorithm's pseudo code in the appendix, and in the next section describe two new off-policy, gradient TD methods, before turning to empirical questions.

\section{HTD derivation}
Conventional temporal difference updating can be more data efficient than gradient temporal difference updating, but the correction term used by gradient-TD methods helps prevent divergence.  
Previous empirical studies\cite{sutton2009fast} demonstrated situations (specifically on-policy) where linear TD(0) can outperform gradient TD methods, and others \cite{hackman2012thesis} demonstrated that Expected Sarsa(0) can outperform multiple variants of the GQ(0) algorithm, even under off-policy sampling. On the other hand, TD($\lambda$) can diverge on small, though somewhat contrived counterexamples.

The idea of hybrid-TD methods is to achieve sample efficiency closer to TD($\lambda$) during on-policy sampling, while ensuring non-divergence under off-policy sampling.
To achieve this, a hybrid algorithm could do conventional, uncorrected TD updates when the data is sampled on-policy, and use gradient corrections when the data is sampled off-policy. This approach was pioneered by Maei \citet{maei2011gradient}, leading to the derivation of the Hybrid Temporal Difference learning algorithm, or HTD(0). Later, Hackman\cite{hackman2012thesis} produced a hybrid version of the GQ(0) algorithm, estimating state-action value functions rather than state-value functions as we do here. In this paper, we derive the first hybrid temporal difference method to make use of eligibility traces, called HTD($\lambda$).

The key idea behind the derivation of HTD learning methods is to modify the gradient of the MSPBE to produce a new learning algorithm. Let $\E_\mu$ represent the expectation according to samples generated under the behavior policy,
$\mu$. The MSPBE\cite{sutton2009fast} can be written as 
\begin{equation*}\label{eq2}
 \MSPBE(\qw) 
= \underbrace{\E_\mu[\delta_t\e_t]^{\top}}_{-A_\pi \qw + b_\pi}\underbrace{\E_\mu[\fs(S_t)\fs(S_t)^{\top}}_{C}]^\inv \E_\mu[\delta_t\e_t]
,
\end{equation*}
where 
$\e_t = \rho_t(\lambda_t \gamma_t \e_{t-1} + \fs(S_t))$ and
 \begin{align}
A_\pi &\defeq \E_\mu[ \e_t (\fs(S_t) - \gamma_{t+1} \fs(S_{t+1}) )^\top] \label{eq_tracesample}\\
 &=   \sum_{s_t \in \mathcal{S}}  \dmu(s_t)  \sum_{a_t \in \mathcal{A}} \underbrace{\mu(s_t,a_t) \rho_t }_{\pi(s_t,a_t)} (\gamma_t\lambda \E_\mu[\e_{t-1} | s_t] + \vx(s_t))  \nonumber \\
 &\hspace{1.4cm} \sum_{s_{t+1} \in \mathcal{S}} P(s_t,a_t,s_{t+1}) (\fs(s_{t}) - \gamma_{t+1} \fs(s_{t+1})) ^\top \nonumber
\end{align}
 \begin{align}
b_{\pi} &\defeq  \E_\mu[R_{t+1}\e_t]  \nonumber\\
&= \sum_{s_t \in \mathcal{S}}  \dmu(s_t)  \sum_{a_t \in \mathcal{A}} \pi(s_t,a_t) (\gamma_t\lambda \E_\mu[\e_{t-1} | s_t] + \vx(s_t)) \nonumber\\ 
&\hspace{2.5cm}   \sum_{s_{t+1} \in \mathcal{S}} \pi(s_t,a_t) P(s_t,a_t,s_{t+1}) r_{t+1} \nonumber
.
\end{align}
%
 Therefore, the relative importance given to states in the MSPBE is weighted by the stationary distribution of
 the behavior policy, $d_\mu: \Ss \rightarrow \R$, (since it is generating samples), but the transitions
 are reweighted to reflect the returns that $\pi$ would produce. 
 
The gradient of the MSPBE is:
\begin{align}\label{eq2}
&-\frac{1}{2}\nabla_{\qw} \MSPBE(\qw) = -A^\top_\pi C^{-1} (-A_\pi\qw + b_\pi)
.
\end{align}
%
Assuming $A_{\pi}^{-1}$ is non-singular,
we get the TD-fixed point solution:
\begin{align}\label{eq3}
&0 = -\frac{1}{2}\nabla_{\qw} {\MSPBE}(\qw) \implies -A_{\pi} \qw + b_{\pi} = 0
. 
\end{align}
The value of $\qw$, for which \eqref{eq3} is zero, is the solution found by linear TD($\lambda$) and LSTD($\lambda$) where $\pi=\mu$. 
The gradient of the MSPBE yields an incremental learning rule with the following general form (see \cite{bertsekas1996neuro}):
\begin{align}
\qw_{t+1} \leftarrow \qw_t + \alpha(M\qw_t + b), \label{eq_general}
\end{align}
where $M =  -A_{\pi}^\top C^{-1}A_{\pi}$ and $b = A^{\top}_{\pi}C^{-1} b_{\pi}$. 
The update rule, in the case of TD($\lambda$), will yield stable convergence if $A_\pi$ is positive definite (as shown by Tsitsiklis and van Roy \citet{tsitsiklis1997ananalysis}). In off-policy learning, we require $A_{\pi}^\top C^{-1}A_{\pi}$ to be positive definite to satisfy the conditions of the ordinary differential equation proof of convergence \cite{maei2010gq}, which holds because $C^{-1}$ is positive definite and therefore $A_{\pi}^\top C^{-1}A_{\pi}$ is positive definite, because $A_{\pi}$ is full rank (true by assumption). See Sutton et al. \citet{sutton2015anemphatic} for a nice discussion on why the $A_{\pi}$ matrix must be positive definite to ensure stable, non-divergent iterations.   
The $C$ matrix in Equation \eqref{eq3},  can be replaced by any positive definite matrix and the fixed point will be unaffected, but the rate of convergence will almost surely change.

Instead of following the usual recipe for deriving GTD, let us try replacing $C^{-1}$ with 
\begin{align*}
{A_{\mu}^{-\top}} \defeq  \E_{\mu}[(\fs(S_t)- \gamma_t\fs(S_{t+1})) {\emu_t}^{\top} ],
\end{align*}
where $\emu$ is the regular on-policy trace for the behavior policy (i.e., no importance weights)
\begin{align*}
\emu_t = \gamma_t\lambda\emu_{t-1} + \vx(S_t)
.
\end{align*}
%
The matrix ${A_{\mu}^{-\top}}$ is a positive definite matrix (proved by Tsitsiklis and van Roy~\citet{tsitsiklis1997ananalysis}).
Plugging ${A_{\mu}^{-\top}}$ into \eqref{eq2} results in the following expected update:
\begin{align} 
&\frac{1}{\alpha}\mathbb{E}[\Delta\qw_t] 
= A^{\top}_{\pi} {A_{\mu}^{-\top}} (-A_{\pi} \qw_t + b_\pi) \nonumber\\
&= (A_{\mu}^{\top} - A_{\mu}^{\top} + A_{\pi}^{\top}) {A_{\mu}^{-\top}} (-A_{\pi} \qw_t + b_\pi) \nonumber\\
&= (A_{\mu}^{\top} {A_{\mu}^{-\top}})(-A_{\pi} \qw_t + b_\pi) + (A_{\pi}^{\top} - A_{\mu}^{\top}){A_{\mu}^{-\top}} (-A_{\pi} \qw_t + b_\pi) \nonumber\\
&= (-A_{\pi} \qw_t + b_\pi) + (A_{\pi}^{\top} - A_{\mu}^{\top}) {A_{\mu}^{-\top}} (-A_{\pi} \qw_t + b_\pi) \nonumber\\
&= (-A_{\pi} \qw_t + b_\pi) \ + \label{htd_dev}\\
&\E_\mu \left[ \left( \fs(S_t) - \gamma_{t+1}\fs(S_{t+1}) \right) \left(\e_t - \emu_t \right)^\top\right] {A_{\mu}^{-\top}}(-A_{\pi} \qw_t + b_\pi) \nonumber
\end{align}
As in the derivation of GTD($\lambda$) \cite{maei2011gradient}, 
let the vector $\h_t$ form a quasi-stationary estimate of the final term,
%
$${A_{\mu}^{-\top}}(-A_{\pi} \qw_t + b_\pi)
.
$$ 
%
Getting back to the primary weight update, we can sample the first term using the fact that
$(-A_{\pi} \qw_t + b_\pi) = \E_\mu [\delta_t\e_t ]$ (see \cite{maei2011gradient}) and use \eqref{eq_tracesample}
to get the final stochastic update
\begin{align}
\Delta\qw_t \leftarrow \alpha\left(\delta_t \e_t + (\vx_t - \gamma_{t+1}\vx_{t+1})\left(\e_t - \emu_t \right)^\top \h_t\right)
. 
\label{HTD}
\end{align}
Notice that when the data is generated on-policy ($\pi = \mu$),
$\e_t = \emu_t$, and thus the correction term disappears and we are left with precisely linear TD($\lambda$). 
When $\pi \ne \mu$, the TD update is corrected as in GTD: 
unsurprisingly, the correction is slightly different but has the same basic form.
 
To complete the derivation, we must derive an incremental update rule for $\h_t$.
We have a linear system, because
\begin{align*} 
\h_t &= {A_{\mu}^{-\top}}(-A_{\pi} \qw_t + b_\pi) 
\implies A^{\top}_{\mu} \h_t = -A_{\pi} \qw_t + b_\pi.
 \end{align*}
Following the general expected update in \eqref{eq_general},
\begin{align} 
  \h_{t+1} \leftarrow \h_t + \alpha_\h \left((-A_{\pi} \qw_t + b_\pi) - A^{\top}_{\mu} \h_t\right) \label{eq_htd}
  \end{align}
  which converges if $A^{\top}_{\mu}$ is positive definite for any fixed $\qw_t$
  and $\alpha_\h$ is chosen appropriately (see Sutton et al.'s recent paper\cite{sutton2015anemphatic}
  for an extensive discussion of convergence in expectation).
To sample this update, recall
\begin{align*} 
A^{\top}_{\mu} \h_t = \mathbb{E}_{\mu}[(\fs(S_t)-\fs(S_{t+1})){\emu_t}^{\top} ] \h_t 
 \end{align*}
giving stochastic update rule for $\h_t$:
\begin{align*} 
\Delta \h_t &\leftarrow \alpha_\h\left[ \delta_t\e_t  - (\vx_t - \gamma_{t+1} \vx_{t+1}){\emu_t}^\top \h_t \right]
.
 \end{align*}
 As in GTD, $\alpha\in\R$ and $\alpha_\h\in\R$ are step-size parameters, and $\delta_t \defeq R_{t+1} + \gamma_{t+1} \qw_t^{\top} \vx_{t+1} - \qw_t^{\top} \vx_t$. 
This hybrid-TD algorithm should converge under off-policy sampling using a proof technique similar to the one used for GQ($\lambda$) (see Maei \& Sutton's proof~\cite{maei2010gq}), but we leave this to future work. The HTD($\lambda$) algorithm is completely specified by the following equations:
\begin{align*} 
\e_t &\leftarrow \rho_t (\lambda_t \gamma_t \e_{t-1} + \vx_t) \\
\emu_t &\leftarrow \lambda_t \gamma_t \emu_{t-1} + \vx_t\\
\Delta \qw_t &\leftarrow \alpha\left[ \delta_t\e_t  + (\gamma_{t+1} \vx_{t+1} - \vx_t )(\emu_t - \e_t)^\top \h_t \right]\\
\Delta \h_t &\leftarrow \alpha_\h\left[ \delta_t\e_t  + (\gamma_{t+1} \vx_{t+1} - \vx_t){\emu}_t^\top \h_t \right]
 \end{align*}
 %
This algorithm
 can be made more efficient by exploiting the
 common terms in $\Delta \qw_t$ and $\Delta \h_t$, as shown in the appendix. 
   
\section{True online HTD}
Recently, a new forward-backward view equivalence has been proposed for online TD methods,
resulting in true-online TD \cite{vanseijen2014true} and true-online GTD \cite{vanhasselt2014off} algorithms.
The original forward-backward equivalence was for offline TD($\lambda$)\footnote{The idea of defining a forward view objective and then converting this computationally impractical forward-view into an efficiently implementable algorithm using traces is extensively treated in Sutton and Barto's introductory text \cite{sutton1998reinforcement}.}.
To derive a forward-backward equivalence under online updating,
a new truncated return was proposed,
which uses the online weight vector that changes into the future,
\begin{align*}
G_{k,t}^{\lambda,\rho} \defeq \rho_k (R_{k+1} + \gamma_{k+1} [ (1-\lambda_{k+1}) \vx_{t+1}^\top {\qw_k} + \lambda_{k+1} G_{k+1,t}^{\lambda,\rho}])
,
\end{align*}
with $G_{t,t}^{\lambda,\rho} \defeq \rho_t \vx_{t}^\top \qw_{t-1}$. 
A forward-view algorithm can be defined
that computes $\qw_{k}$ online assuming access to future samples, and then an exactly equivalent incremental backward-view
algorithm can be derived that does not require access to future samples. This framework was used to derive the TO-TD algorithm for the on-policy setting, and TO-GTD for the more general off-policy setting. 
This new true-online equivalence is not only interesting theoretically, but
also translates into improved prediction and control performance \cite{vanseijen2014true,vanhasselt2014off}.

In this section, we derive a true-online variant of HTD($\lambda$).
When used on-policy HTD($\lambda$) behaves similarly to TO-TD($\lambda$).
Our goal in this section is to combine the benefits of both hybrid learning and true-online traces in a single algorithm.
 We proceed with a similar derivation to TO-GTD($\lambda$) \cite[Theorem 4]{vanhasselt2014off},
with the main difference appearing in the update of the auxiliary weights.
Notice that the primary weights $\qw$, and the auxiliary weights $\h$, of HTD($\lambda$) have a similar structure.
Recall from \eqref{htd_dev}, the modified gradient of the MSPBE, or expected primary-weight update can be written as:
\begin{align*} 
\frac{1}{\alpha}\mathbb{E}[\Delta\qw_t] 
&= (-A_{\pi} \qw_t + b_\pi) \  \\
&+ \E_\mu \left[ \left( \fs(S_t) - \gamma_{t+1}\fs(S_{t+1}) \right) \left(\e_t - \emu_t \right)^\top\right] \h_t
\end{align*}
Similarly, we can rewrite the expected update of the auxiliary weights by plugging $A_{\mu}^\top$ into \eqref{eq_htd}:
\begin{align*} 
\frac{1}{\alpha_h}\mathbb{E}[\Delta\h_t] 
&= (-A_{\pi} \qw_t + b_\pi) \ \\
&+ \E_\mu \left[ \left( \fs(S_t) - \gamma_{t+1}\fs(S_{t+1}) \right) {\emu_t}^\top\right] \h_t
\end{align*}
%
As in the derivation of TO-GTD \cite[Equation 17,18]{vanhasselt2014off}), 
for TO-HTD we will sample the second part of the update using a backward-view
and obtain forward-view samples for $(-A_{\pi} \qw_t + b_\pi)$.
%
The resulting TO-HTD($\lambda$) algorithm is completely specified by the following equations
\begin{align} 
\e_t &\leftarrow \rho_t (\lambda_t \gamma_t \e_{t-1} + \vx_t)\nonumber\\
\emu_t &\leftarrow \lambda_t \gamma_t \emu_{t-1} + \vx_t\nonumber\\
\e^o_t &\leftarrow \rho_t(\lambda_t \gamma_t \e^o_{t-1} + \alpha_t (1-\rho_t \gamma_t \lambda_t \vx_t^\top \e^o_{t-1}) \vx_t) \nonumber\\
\mathbf{d} &= \delta_t \e^o_t + (\e^o_t - \alpha_t \rho_t \vx_t)(\qw_t - \qw_{t-1})^\top \vx_t \label{eq_tohtd}\\
\qw_{t+1} &\leftarrow \qw_t + \mathbf{d}  + \alpha_t (\gamma_{t+1} \vx_{t+1} - \vx_t)(\emu_t - \e_t )^\top \h_t\nonumber\\
\h_{t+1} &\leftarrow  \h_t + \mathbf{d}  + \alpha_\h (\gamma_{t+1} \vx_{t+1} - \vx_t){\emu}_t^\top \h_t\nonumber
 \end{align}
 In order to prove that this is a true-online update, we use the constructive theorem due to van Hasselt et al.~\cite{vanhasselt2014off}.
 \begin{theorem}[True-online HTD($\lambda$)]
 For any $t$, the weight vectors $\qw_t^t, \h_t^t$ as defined by the forward view 
 \begin{align*}
 \qw_{k+1}^t &= \qw_k^t + \alpha_k(G_{k,t}^{\lambda,\rho} - \rho_k \vx_t^\top \qw_k^t)\vx_k \\
 & \hspace{1.0cm}+ \alpha_k (\vx_t - \gamma_{t+1} \vx_{t+1})(\e_t - \emu_t)^\top \h_k^t
\end{align*}
 \begin{align*}
  \h_{k+1}^t &= \h_k^t + \alpha_{h,k}(G_{k,t}^{\lambda,\rho} - \rho_k \vx_t^\top \qw_k^t)\vx_k \\
 & \hspace{1.0cm}+ \alpha_{h,k} (\vx_t - \gamma_{t+1} \vx_{t+1}){\emu_t}^\top \h_k^t
 \end{align*}
 are equal to $\qw_t$, $\h_t$ as defined by the backward view in \eqref{eq_tohtd}.
 \end{theorem}
 \begin{proof}
 We apply \cite[Theorem 1]{vanhasselt2014off}.
 The substitutions are
\begin{align*}
&\eta_t = \rho_t \alpha_t\\
&\mathbf{g}_{w,k} = \alpha_k (\vx_k - \gamma_{k+1} \vx_{k+1})(\e_k - \emu_k)^\top \h_k\\
&\mathbf{g}_{h,k} = \alpha_{h,k} (\vx_k - \gamma_{k+1} \vx_{k+1}){\emu_k}^\top \h_k\\
&Y_t^t = \qw_{t-1}^\top \vx_t\\
&Y_k^t = R_{k+1} + \gamma_{k+1}(1-\lambda_{k+1} \rho_{k+1}) \qw_k^\top \vx_{k+1} + \gamma_{k+1} \lambda_{k+1} G_{k+1,t}^{\lambda,\rho}
 \end{align*}
where $\mathbf{g}_{w,k}$ is called $\vx_k$ in van Hasselt's Theorem 1~\cite{vanhasselt2014off}. The proof then follows through in the same way as in van Hasselt's Theorem 4~\cite{vanhasselt2014off}, where we apply Theorem 1
to $\qw$ and $\h$ separately.
 \end{proof}
 
Our TO-HTD(0) algorithm is equivalent to HTD(0), but TO-HTD($\lambda$) is not equivalent to TO-TD($\lambda$) under on-policy sampling.
To achieve the later equivalence, replace $\delta_t \defeq R_{t+1} + \gamma_{t+1}\qw_{t}^\top\vx_{t+1} + \qw_{t-1}^\top\vx_{t}$ and $\mathbf{d} \defeq \delta_t \e^o_t - \alpha_t \rho_t \vx_t(\qw_t - \qw_{t-1})^\top \vx_t$. We opted for the first equivalence for two reasons. 
In preliminary experiments, TO-HTD($\lambda$) described in Equation \eqref{eq_tohtd} already exhibited similar
performance compared to TO-TD($\lambda$), and so designing for the second equivalence was unnecessary. Further, TO-GTD($\lambda$) was derived to ensure equivalence between TO-GTD(0) and GTD(0); this choice, therefore, better parallels that equivalence. 
 
 Given our two new hybrid methods, and the long list of existing linear prediction algorithms we now focus on how these algorithms perform in practice. 
\section{Experimental study}
Our empirical study focused on three main aspects: (1) early learning performance with different feature representations, (2) parameter sensitivity, and, (3) efficacy in on and off-policy learning.
The majority of our experiments were conducted on random MDPs (variants of those used in previous studies\cite{mahmood2015off, geist2014off}). Each random MDP contains 30 states, and three actions in each state. From each state, and for each action, the agent can transition to one of four next states, assigned randomly from the entire set without replacement. Transition probabilities for each MDP instance are randomly sampled from $[0,1]$ and the transitions were normalized to sum to one. The expected reward for each transition is also generated randomly in $[0,1]$ and the reward on each transition was sampled without noise. Two transitions are randomly selected to terminate: $\gamma(s_i,s_j) = 0$ for $i\ne j$. 
Each problem instance is held fixed during learning. 

We experimented with three different feature representations. The first, a {\em tabular} representation where each state is represented with a binary vector with a single one corresponding the current state index. This encoding allows perfect representation of the value function with no generalization over states. The second representation is computed by taking the tabular representation and {\em aliasing} five states to all have the same feature vector, so the agent cannot differentiate these states. These five states were selected randomly without replacement for each MDP instance. The third representation is a dense {\em binary} encoding where the feature vector for each state is the binary encoding of the state index, and thus the feature vector for a 30 state MDP requires just five components. Although the binary representation appears to exhibit an inappropriate amount of generalization,
 we believe it to be more realistic that a tabular representation, because access to MDP state is rare in real-world domains (e.g., 
 a robotic with continuous sensor values). The binary representation should be viewed as an 
 approximation to the poor, and relatively low-dimensional (compared to the number of states in the world) representations common in real applications. All feature encoding we normalized. Experiments conducted with the binary representation use $\gamma=0.99$, and the rest use $\gamma=0.9$.

To generate policies with purposeful behavior, we forced the agent to favor a single action in each state. The target policy is generated by randomly selecting an action and assigning it probability 0.9 (i.e., $\pi(s,a_i)=0.9$) in each state, and then assigning the remaining actions the remaining probability evenly. In the off-policy experiments the behavior policy is modified to be slightly different than the target policy, by selecting the same base action, but instead assigning a probability of 0.8 (i.e., $\mu(s,a_i)=0.8$) . This choice ensures that the policies are related, but guarantees that $\rho_t$ is never greater than 1.5 thus avoiding inappropriately large variance due to importance sampling\footnote{See the recent study by Mahmood \& Sutton \cite{mahmood2015off} for an extensive treatment of off-policy learning domains with large variance due to importance sampling.}. 

Our experiment compared 12 different linear complexity value function learning algorithms, including: GTD($\lambda$), HTD($\lambda$), true-online GTD($\lambda$), true-online HTD($\lambda$), true-online ETD($\lambda$), true-online ETD($\lambda,\beta$), PTD($\lambda$), GTD2 - mp($\lambda$), TDC - mp($\lambda$), linear off-policy TD(0), TD($\lambda$), true- online TD($\lambda$). The later two being only applicable in on-policy domains, and the two mirror-prox methods are straight-forward extensions (and described in the appendix) of the GTD2-mp and TDC-mp methods \cite{mahadevan2014proximal} to handle traces ($\lambda>0$). We drop the $\lambda$ designation of each method in the figure labels to reduce clutter. 

Our results were generated by performing a large parameter sweep, averaged over many independent runs, for each random MDP instance, and then averaging the results over the entire set of MDPs. We tested 14 different values of the step-size parameter $\alpha \in \{0.1\times 2^j|j=-8,-7,...,6\}$, seven values of $\eta \in\{2^j|j=-4,-2,-1,0,1,2,4\}$ ($\alpha_\h \defeq \alpha\eta$), and 20 values of $\lambda =$ $\{0,0.1,...,0.9,0.91,...,1.0\}$. We intentionally precluded smaller values of $\alpha_\h$ from the parameter sweep because many of the gradient TD methods simply become their on-policy variants as $\alpha_\h$ approaches zero, whereas in some off-policy domains values of $\alpha_\h > \alpha$ are required to avoid divergence \cite{white2015thesis}. We believe this range of $\eta$ fairly reflects how the algorithms would be used in practice if avoiding divergence was a priority. The $\beta$ parameter of TO-ETD($\lambda,\beta$) was set equal to $0.5\gamma_t$. Each algorithm instance, defined by one combination of $\alpha,\eta$, and $\lambda$ was evaluated using the mean absolute value error on each time step,
\begin{align*}
\epsilon_t \defeq \sum_{s\in\Ss}d_\mu(s)\left|\frac{\vx(s)^\top\qw_t - V^*(s)}{V^*(s)}\right|,
\end{align*} 
averaged over 30 MDPs, each with 100 runs. Here $V^*:\Ss\rightarrow\R$ denotes the {\em true} state-value function, which can be easily computed with access to the parameters of the MDP.

The graphs in Figures \ref{figure_offpolicy} and \ref{figure_onpolicy} include 
(a) learning curves with $\alpha,\eta$, and $\lambda$ selected to minimize the mean absolute value error, for each of the three different feature representations, and
(b) parameter sensitivity graphs for $\alpha,\eta$, and $\lambda$, in which the mean absolute value error is plotted against the parameter value, while the remaining two parameters are selected to minimize mean absolute value error. 
These graphs are included across feature representations, for on and off-policy learning.
Across all results the parameters are selected to optimize performance over the last half of the experiment to ensure stable learning throughout the run.


To analyze large variance due to importance sampling and off-policy learning we also investigated Baird's counterexample \cite{baird1995residual}, a simple MDP that causes TD learning to diverge. This seven state MDP uses a target policy that is very different from the behavior policy, a feature representation that allows perfect representation of the value function, but also causes inappropriate generalization. We used the variant of this problem described by Maei \cite{maei2011gradient} and White \cite[Figure 7.1]{white2015thesis}. 
We present results with the root mean squared error
\footnote{In this counterexample the mean absolute value error is not appropriate because the optimal values for this task are zero. The MSPBE is often used as a performance measure, but the MSPBE changes with $\lambda$; for completeness, we include results with the MSPBE in the appendix.},
$$
\epsilon_t \defeq \sum_{s\in\Ss}d_\mu(s)\left(\vx(s)^\top\qw_t - V^*(s) \right)^2,
$$
in Figure \ref{figure_offpolicy}.
The experiment was conducted in the same way was the random MDPs, except we did not average over MDPs---there is only one---and we used different parameter ranges. We tested 11 different values of the step-size parameter $\alpha \in \{0.1\times2^j|j=-10,-9,...,-1,0\}$, 
12 values of $\eta \in \{2^j | j = -16, -8, \ldots, -2,-1,0,1,2,\ldots,32\}$ 
($\alpha_\h \defeq \alpha\eta$), and the same 20 values of $\lambda$. We did not evaluate TD(0) on this domain because the algorithm will diverge and that has been shown many times before.

In addition to performance results in Figures \ref{figure_offpolicy} and \ref{figure_onpolicy}, Table \ref{table_runtimes} summarizes the runtime comparison
for these algorithms.
Though the algorithms are all linear in storage and computation, they do differ in both implementation and runtime,
particularly due to true-online traces. The appendix contains several plots of runtime verses value error illustrating the trade-off between computation and sample complexity for each algorithm.
  Due to space constraints, we have included the aliased tabular representation results for on-policy learning in the appendix, 
since they are similar to the tabular representation results in on-policy learning.

\begin{table*}[htbp!]
\scriptsize    
\vspace{-0.5cm}
\begin{center}
  \begin{tabular}{c| c | c | c | c | c | c | c | c | c | c  | c | c}
    \hline
    & TD(0) & TD($\lambda$) & TO-TD & PTD & GTD & TO-ETD & TO-ETD($\beta$) & HTD & TO-GTD & GTD-MP & TDC-MP & TO-HTD\\ 
    \hline
   \hspace{-0.3cm}On-policy & 120.0 & 132.7 & 150.1 & 172.4 & 204.6 & 287.8 & 286.0 & 311.8 & 366.2 & 467.4 & 466.2 & 466.0\\
       \hline
   \hspace{-0.3cm}Off-policy & 108.3 & - & - & 158.7 & 175.2 & 249.65  & 254.7 & 267.5 & 316.2 & 407.8 & 395.7 & 403.3\\
   \hline
  \end{tabular}
  \vspace{-0.6cm}
\end{center}
  \caption{Average runtime in microseconds for 500 steps of learning, averaged over 30 MDPS, with 100 runs each, with 30-dimensional tabular features.}\label{table_runtimes}
\end{table*}

\newcommand{\gwidth}{0.333\textwidth}


\newcommand{\addlabel}[1]{\hspace{-2.0cm}\parbox{0.5cm}{\vspace{-5.0cm}$#1$}\hspace{-2.0cm}}

\begin{figure*}
\centering
\hspace*{-1.2cm}
\begin{tabular}{ccc}
  \vspace{-.2cm}
Tabular features & Aliased Tabular features &  Binary features \\
  \vspace{-.3cm}
\includegraphics[width=\gwidth]{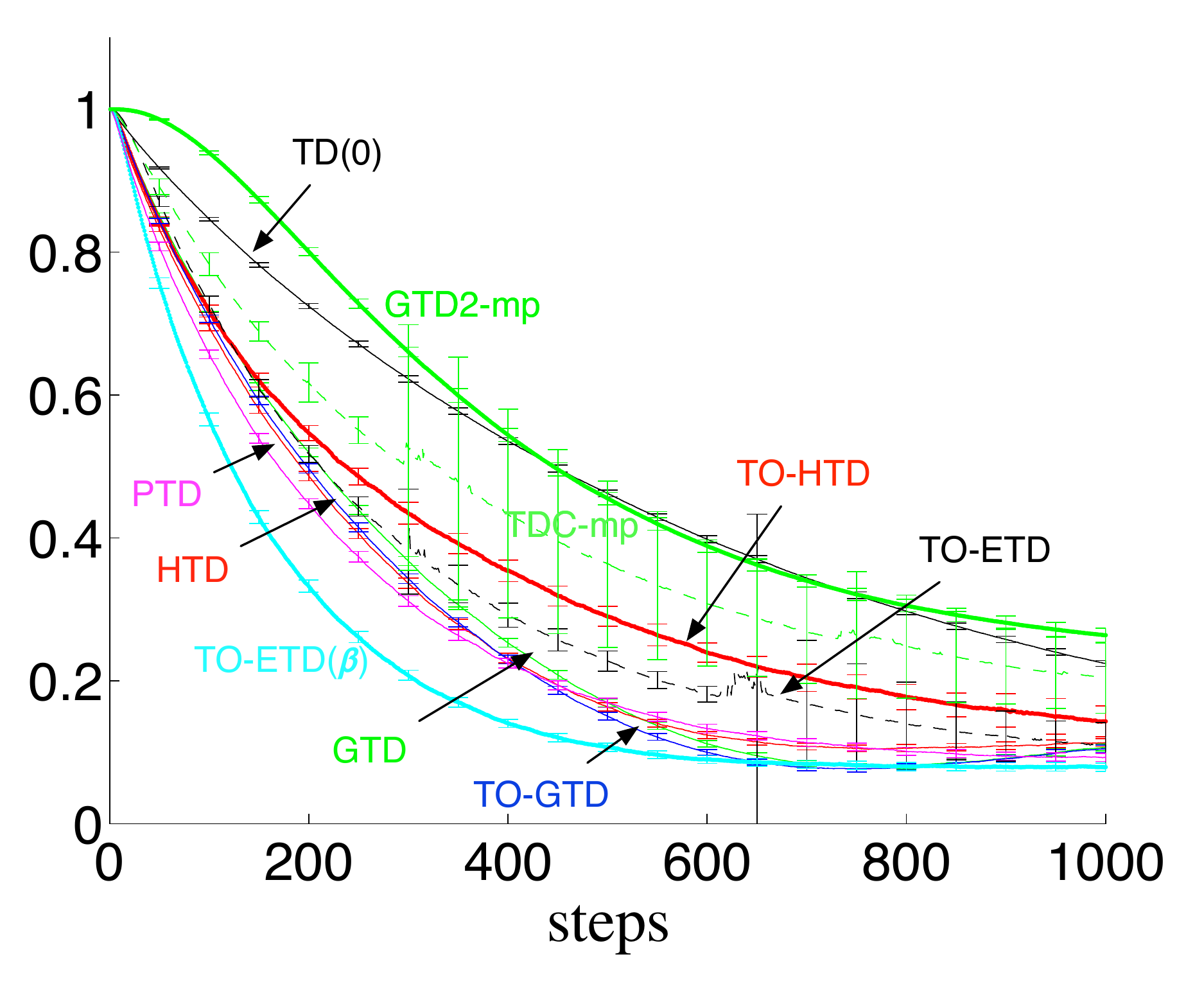} &   \includegraphics[width=\gwidth]{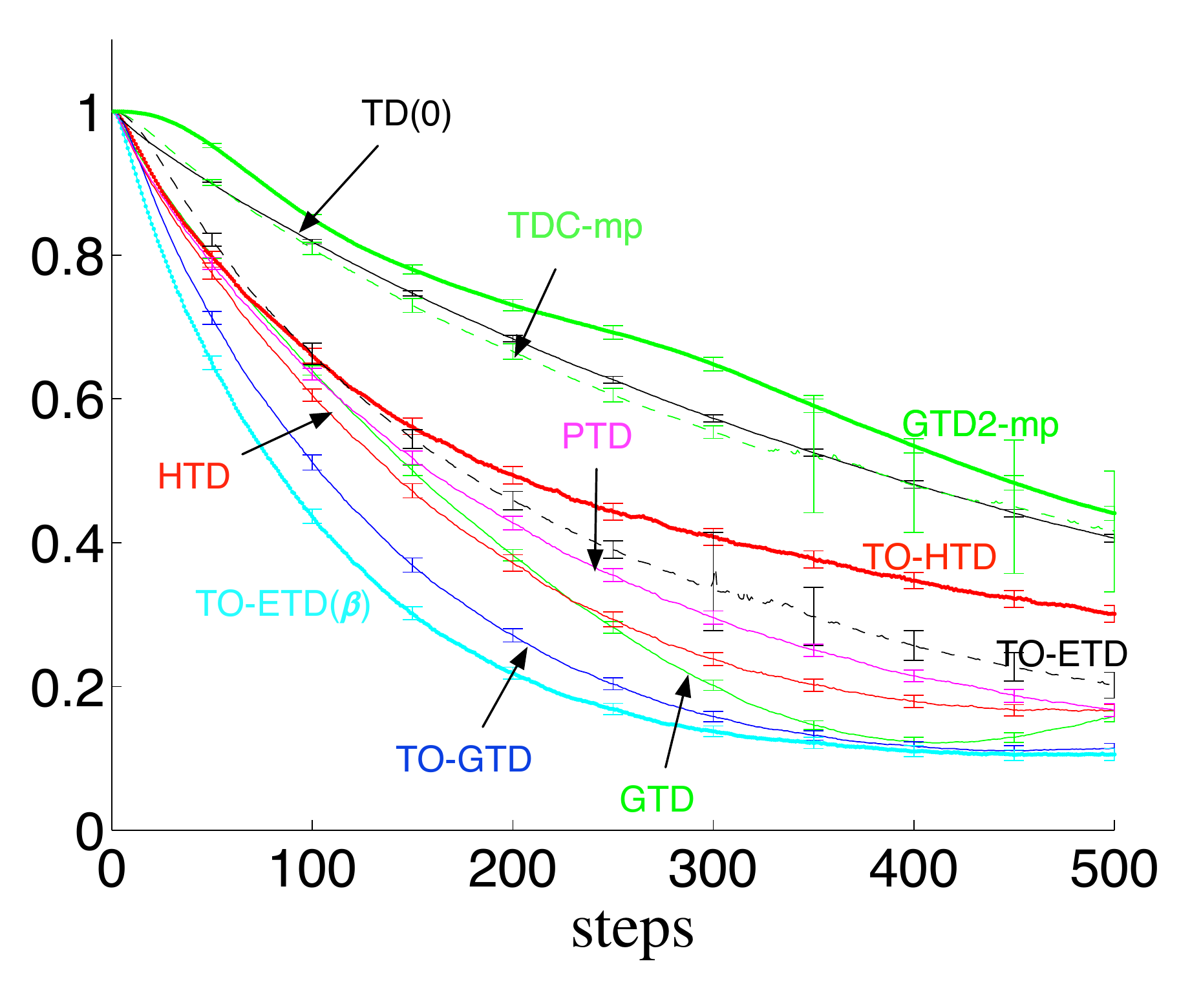} & \includegraphics[width=\gwidth]{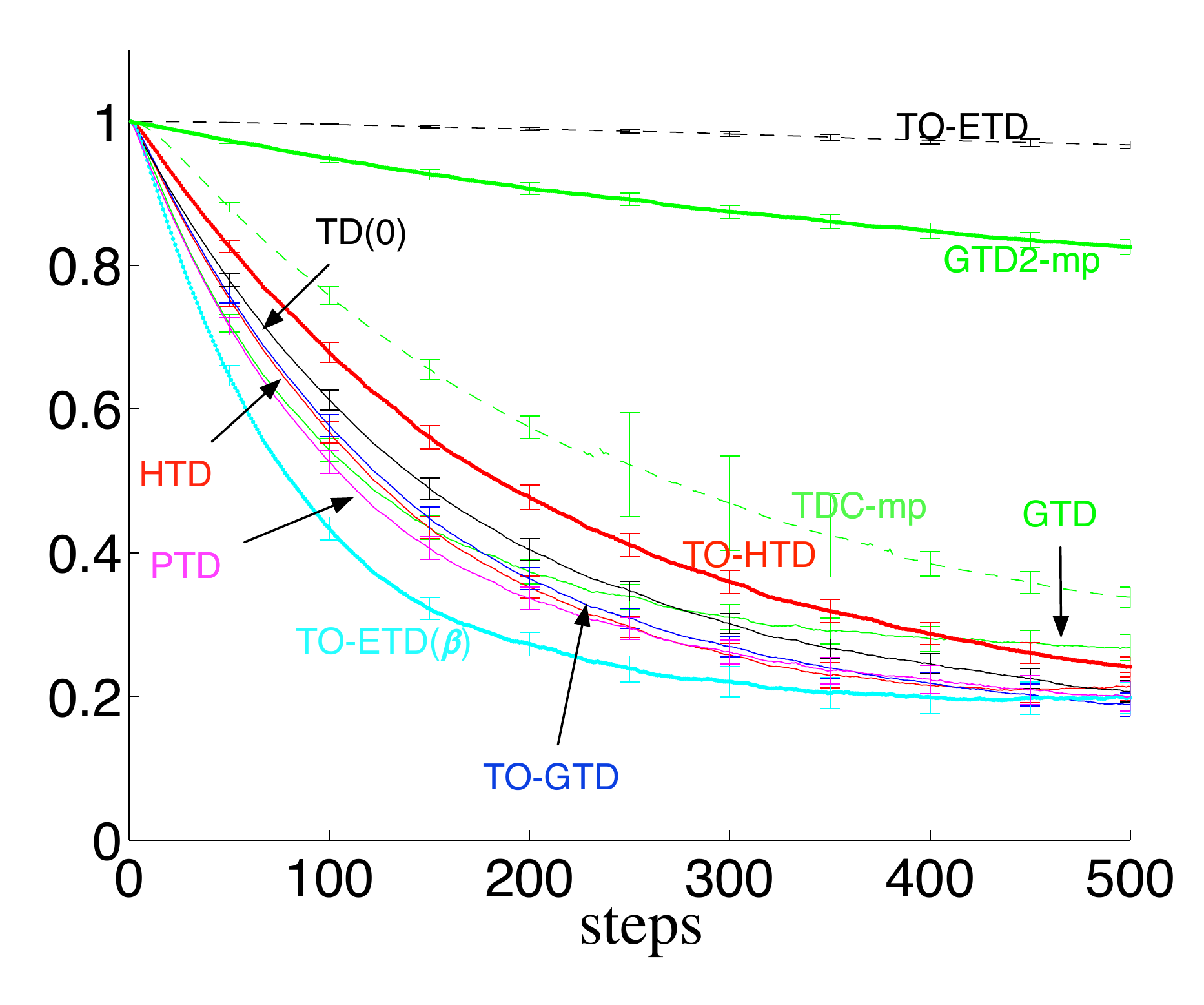} \\
  \vspace{-.3cm}
\includegraphics[width=\gwidth]{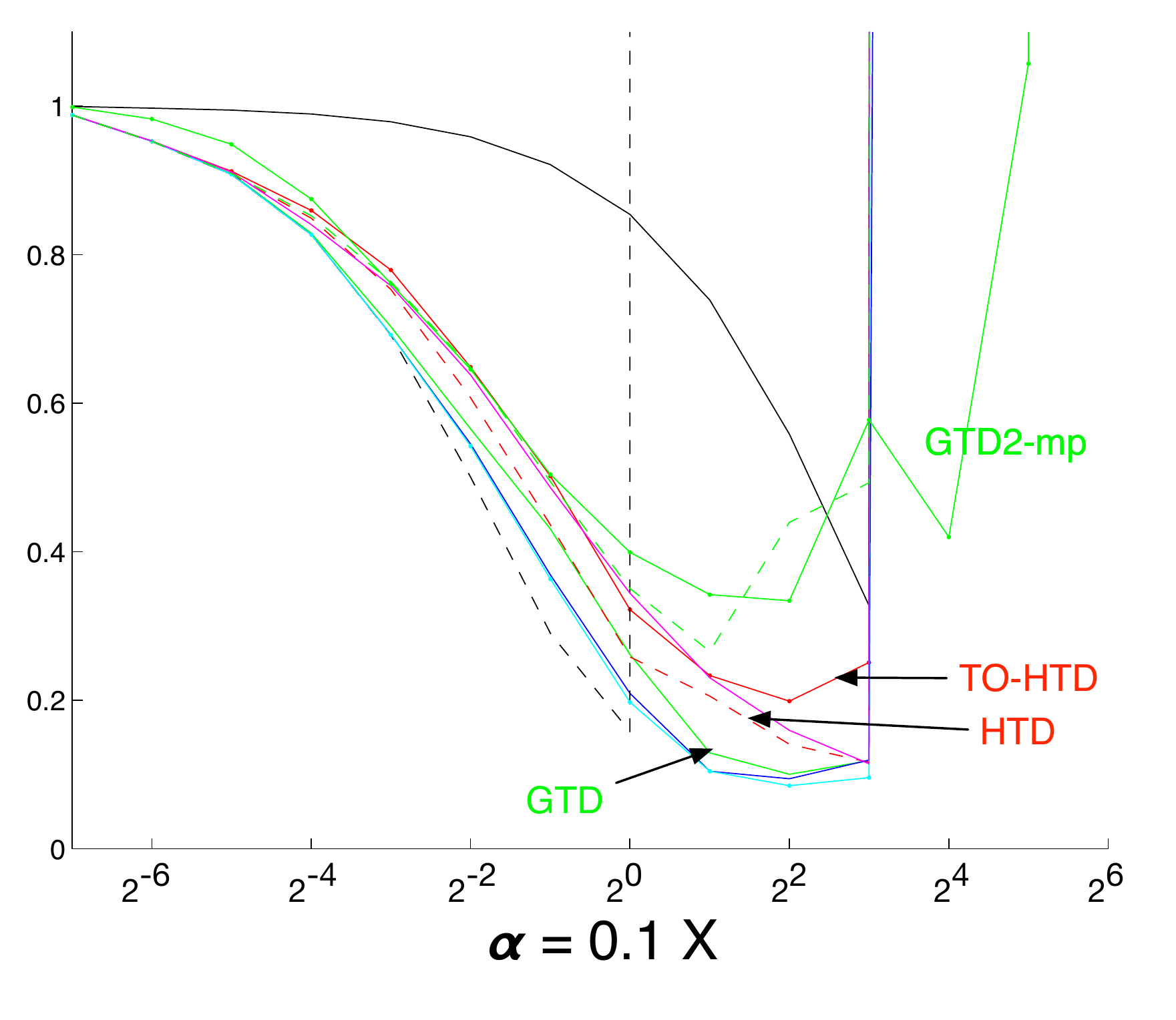} &   \includegraphics[width=\gwidth]{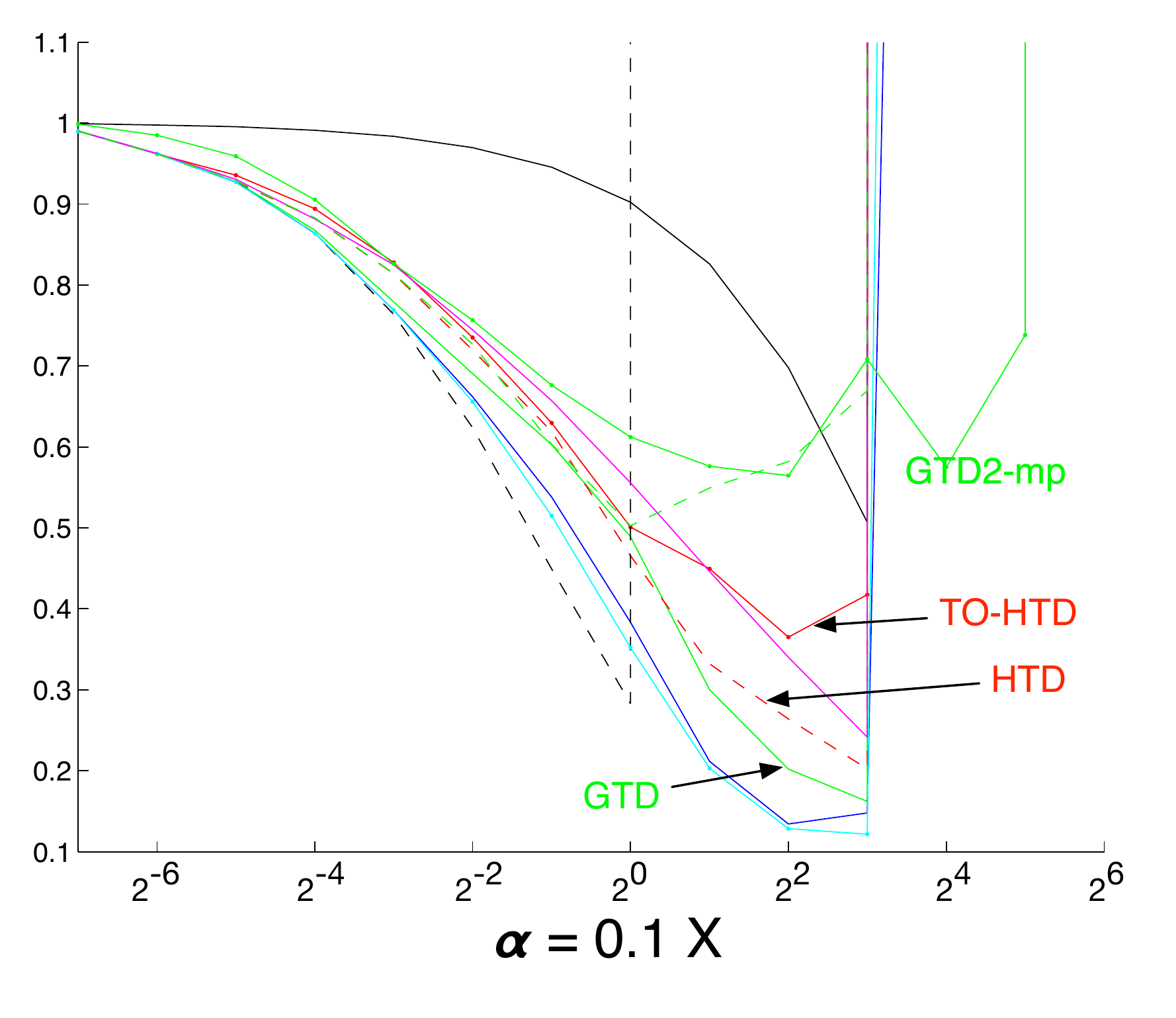} & \includegraphics[width=\gwidth]{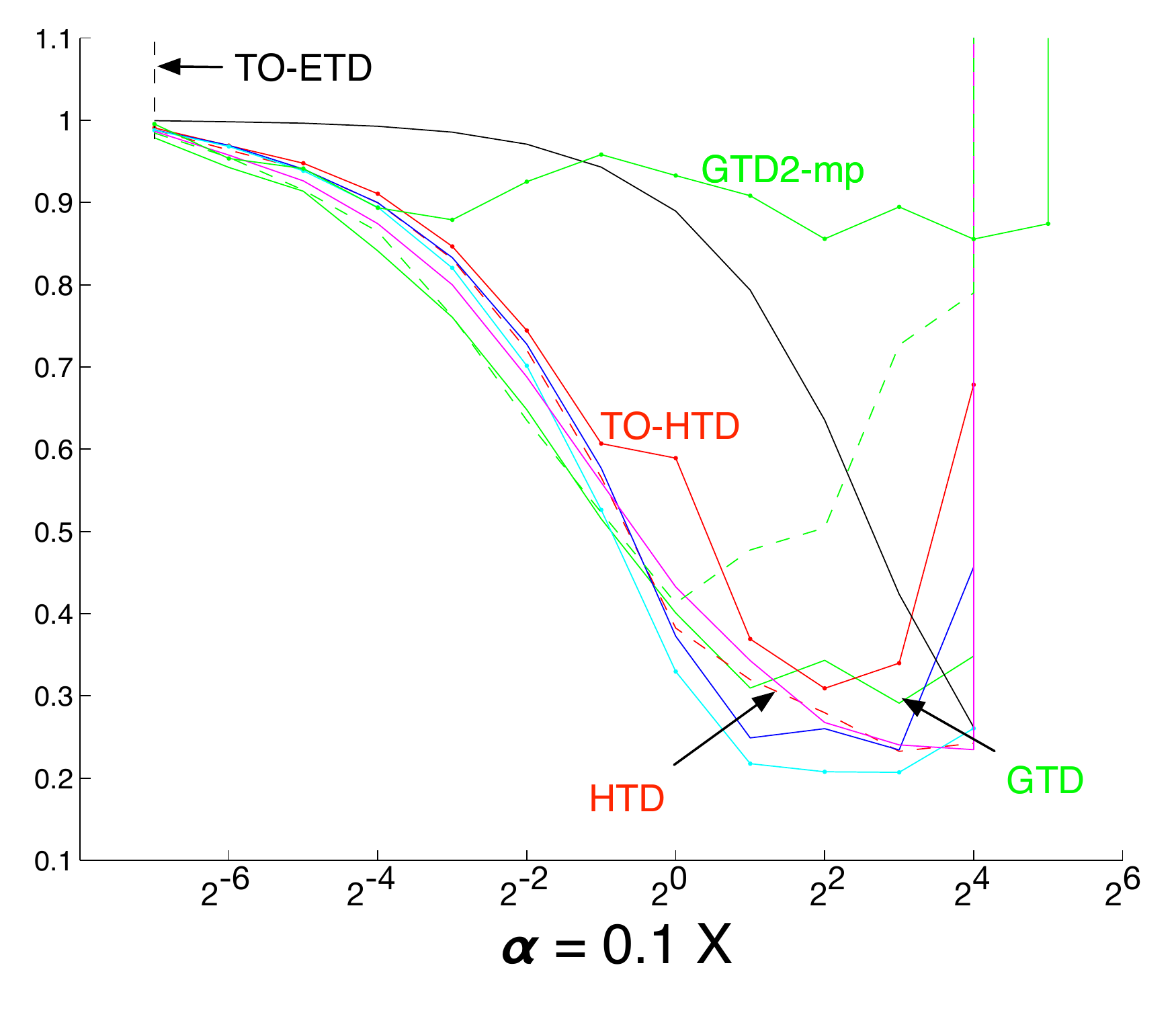} \\
    \vspace{-.3cm}

\includegraphics[width=\gwidth]{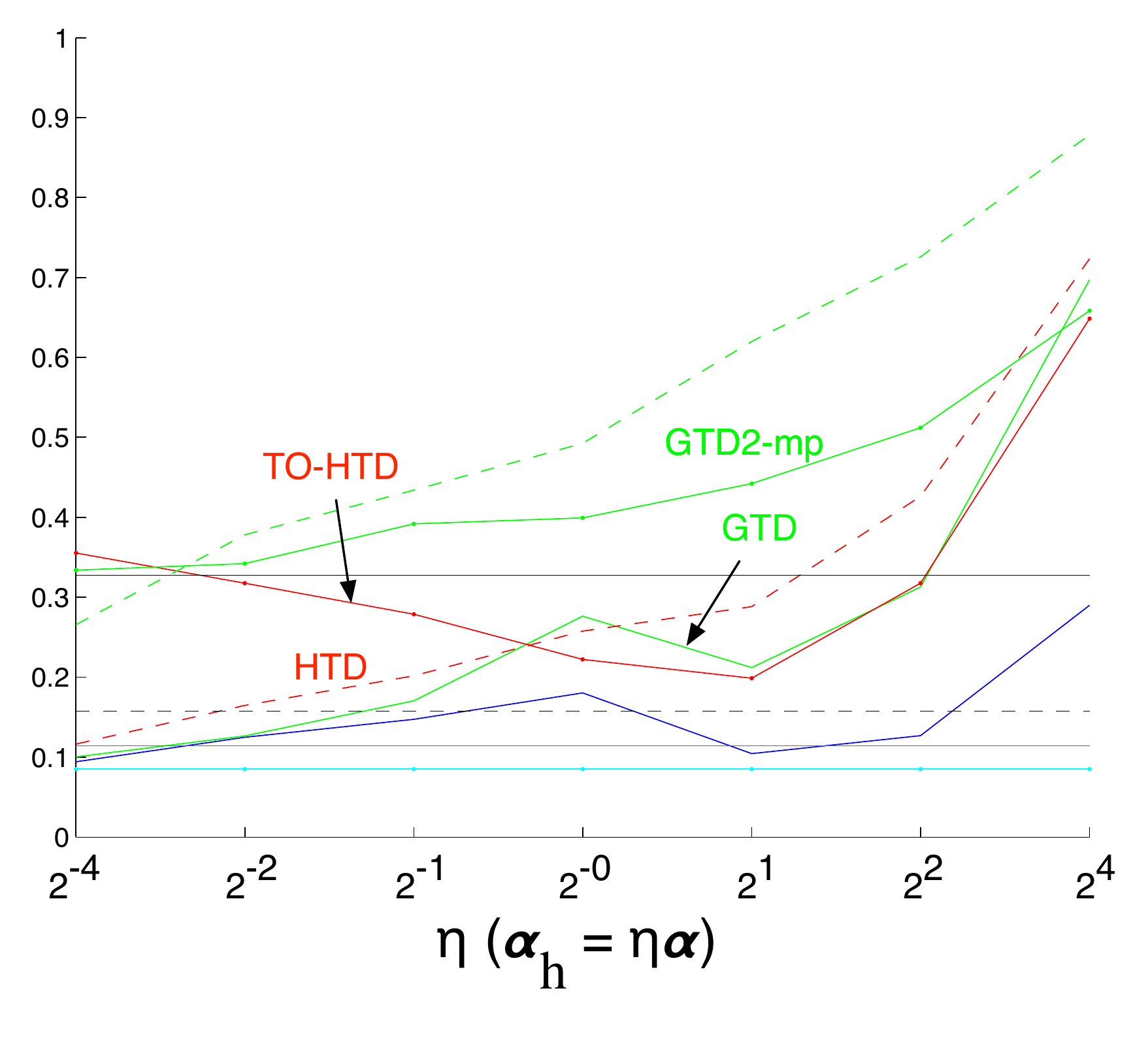} &   \includegraphics[width=\gwidth]{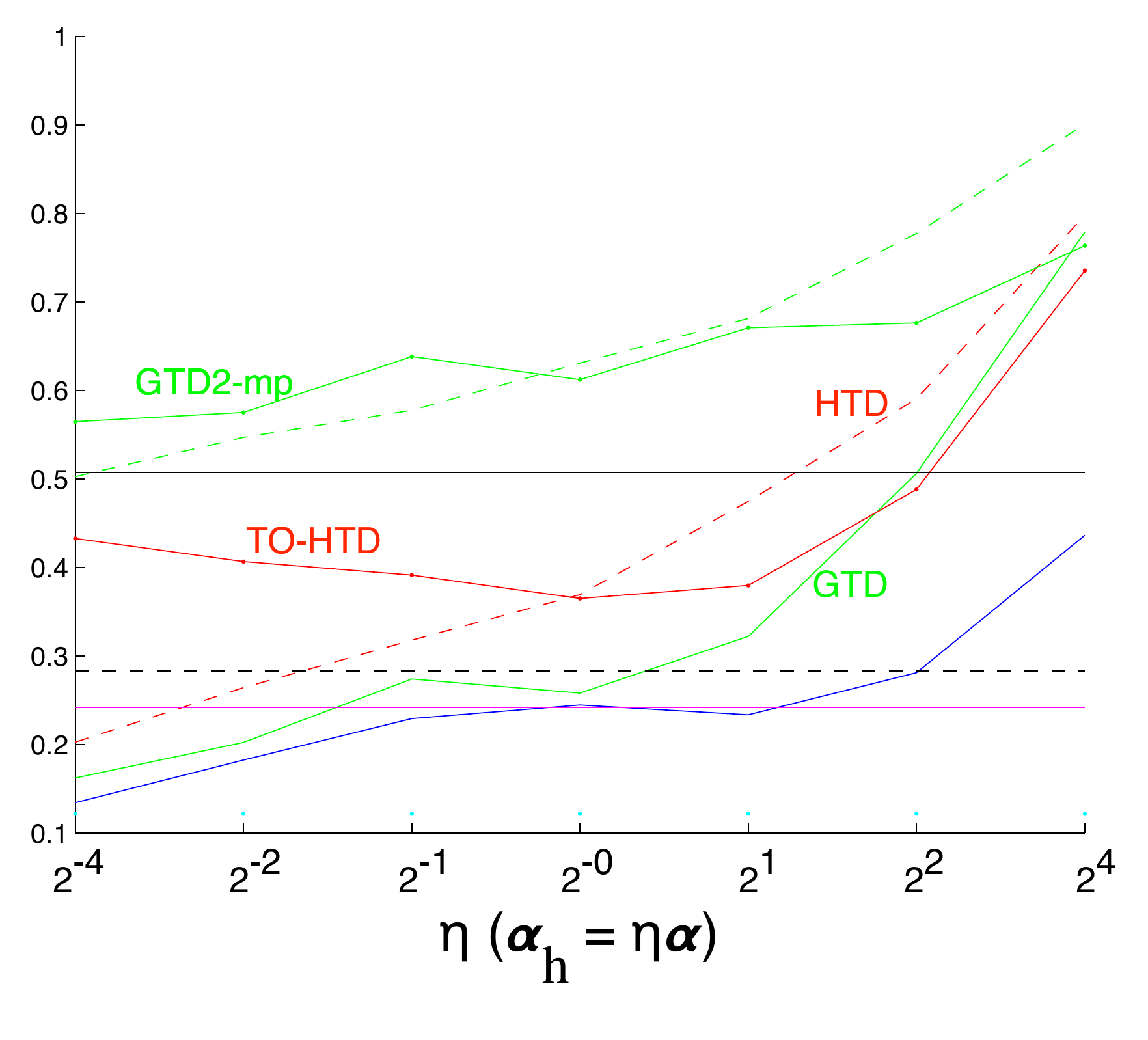} & \includegraphics[width=\gwidth]{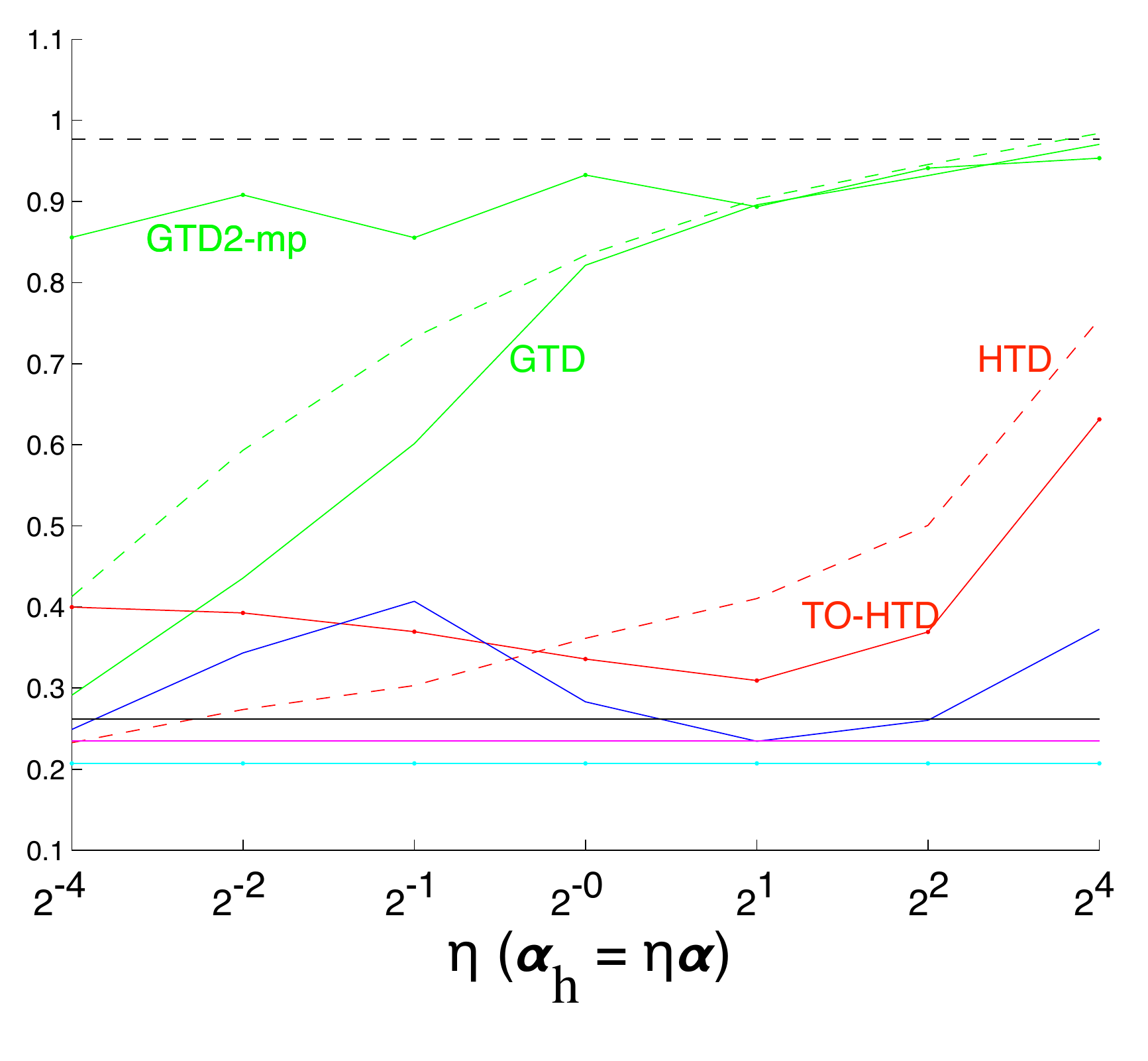} \\
  \includegraphics[width=\gwidth]{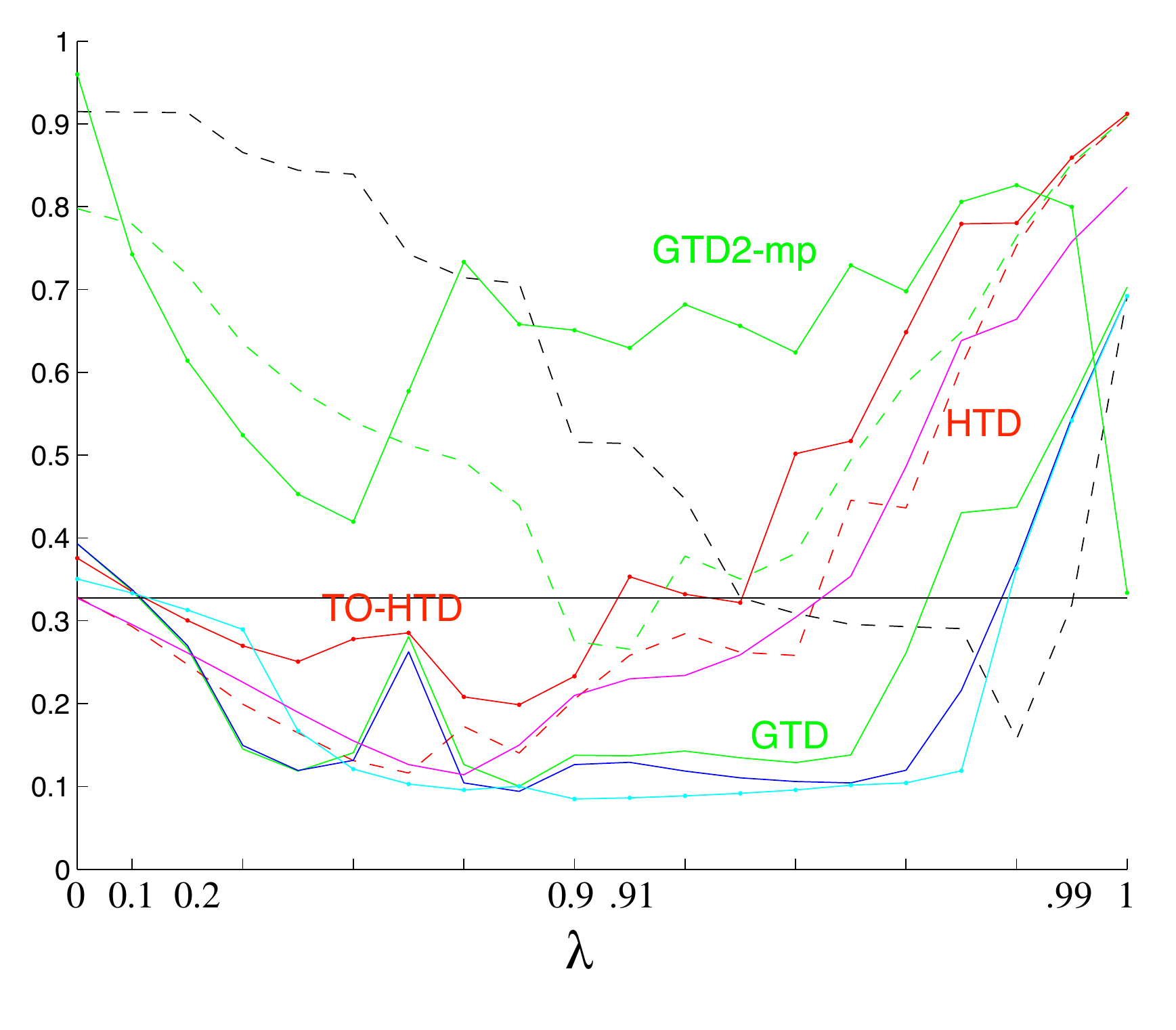} &   \includegraphics[width=\gwidth]{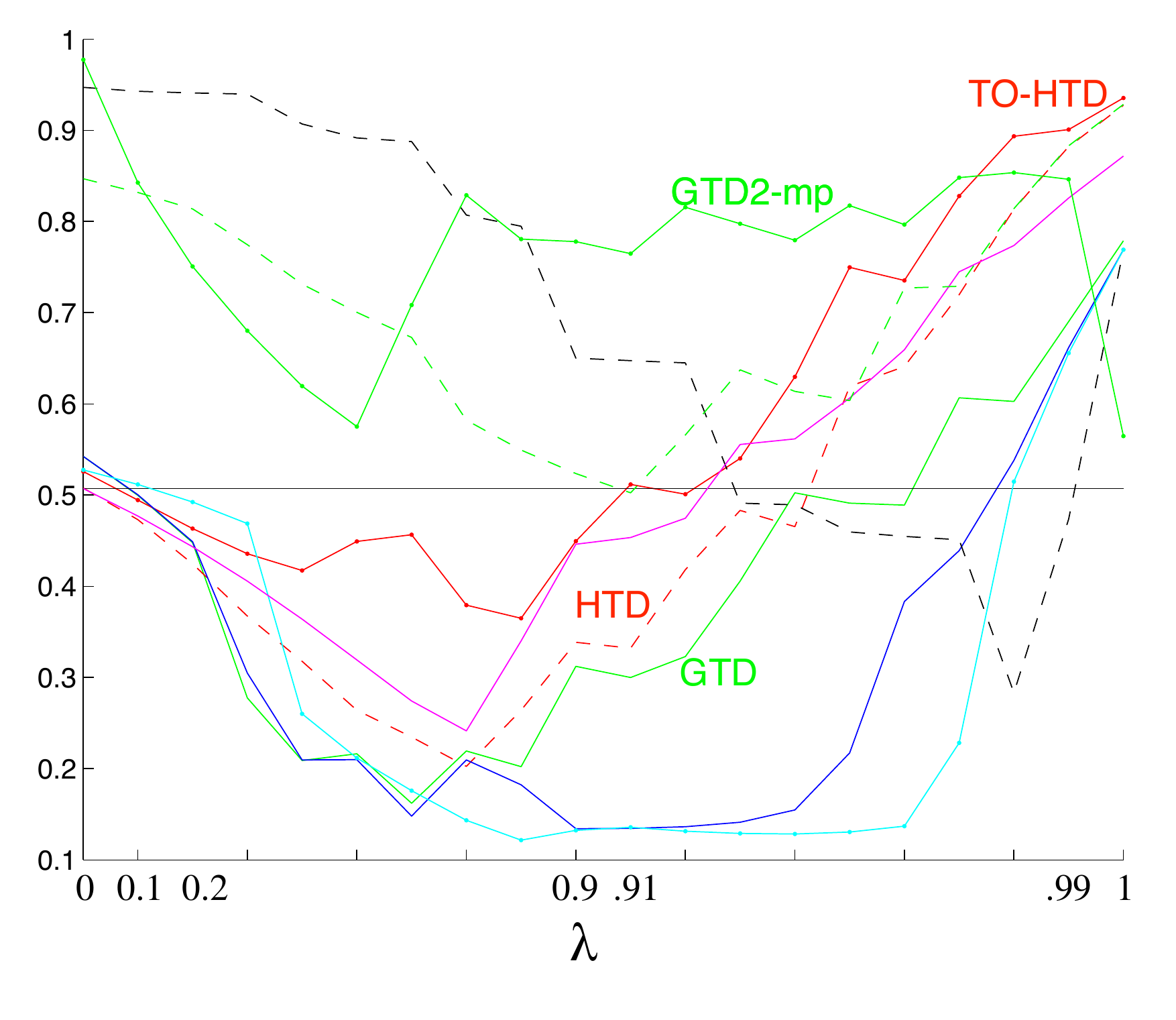} & \includegraphics[width=\gwidth]{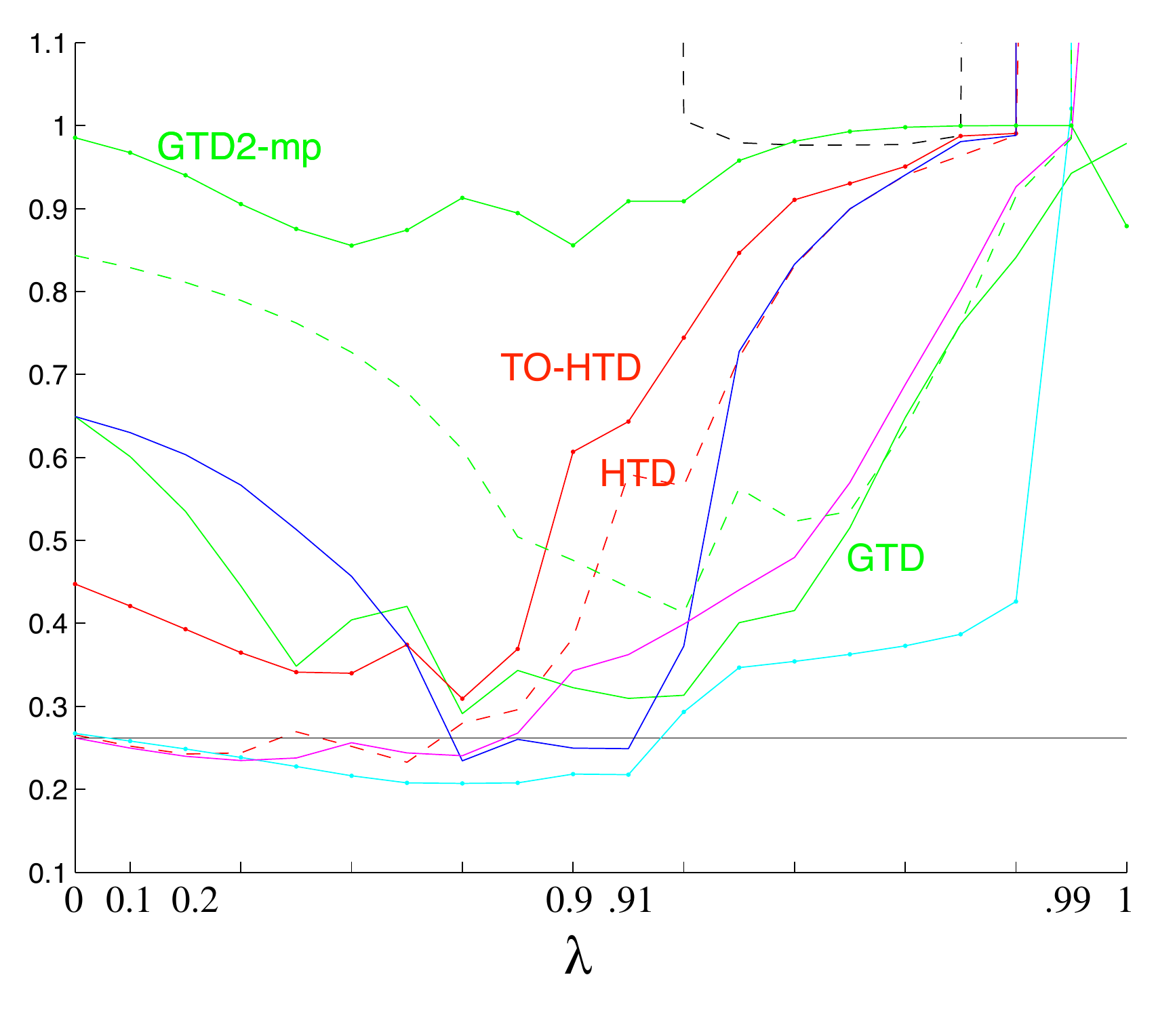} \\
 \end{tabular}
  \vspace{-.5cm}
 \caption{ \textbf{Off-policy} performance on random MDPs with three different representations. All plots report mean absolute value error averaged over 100 runs and 30 MDPs. The plots are organized in columns left to right corresponding to tabular, aliased, and binary features. The plots are organized in rows from top to bottom corresponding to learning curves, $\alpha$,  $\eta$, and $\lambda$ sensitivity. The error bars are standard errors ($s/\sqrt(n)$) computed from 100 independent runs.}\label{figure_offpolicy}
 \end{figure*}
 
 \begin{figure*}
\centering
\hspace*{-1.2cm}
\begin{tabular}{ccc}
Tabular features & Binary features  & Baird's counterexample\\
\includegraphics[width=\gwidth]{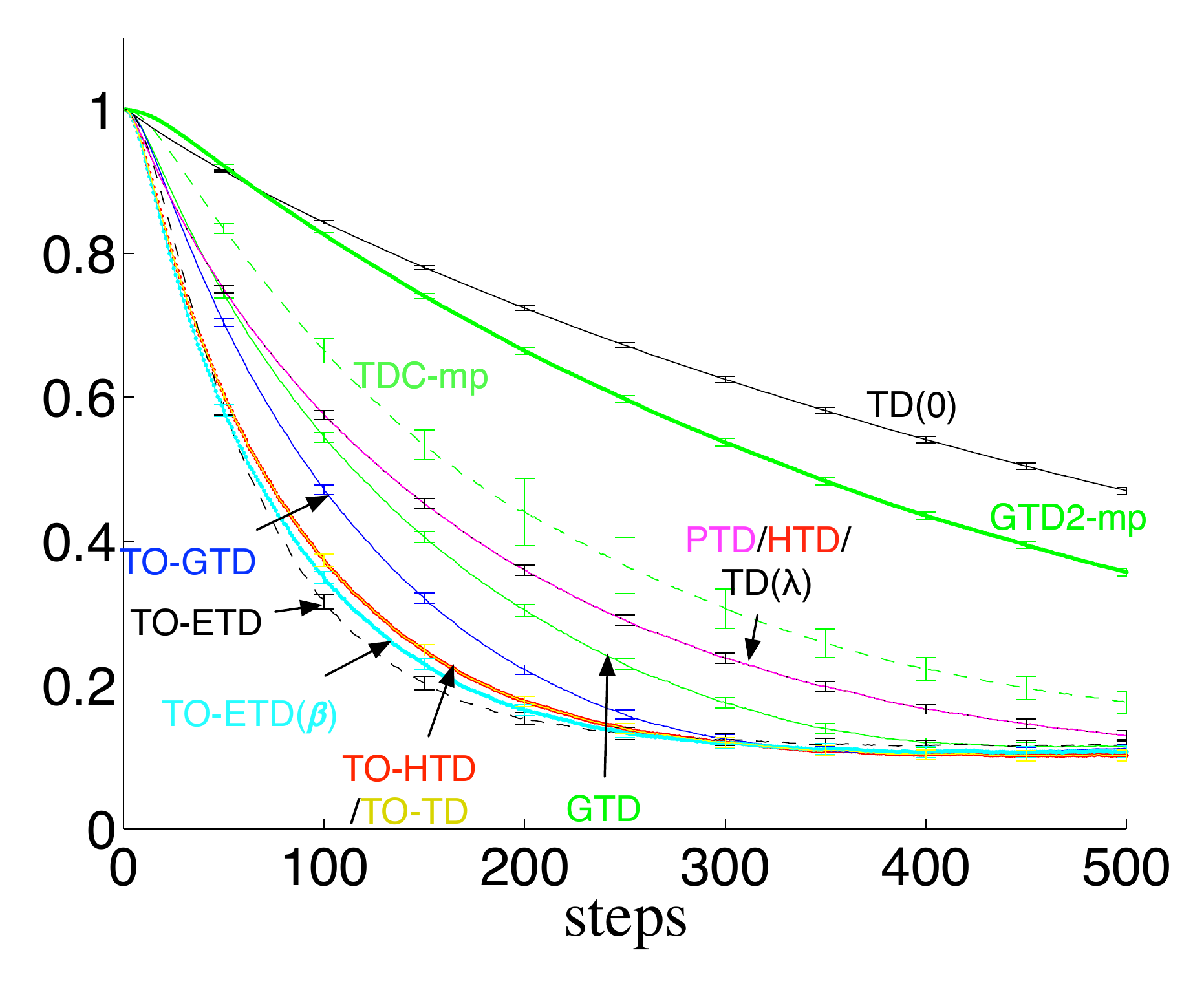} & \includegraphics[width=\gwidth]{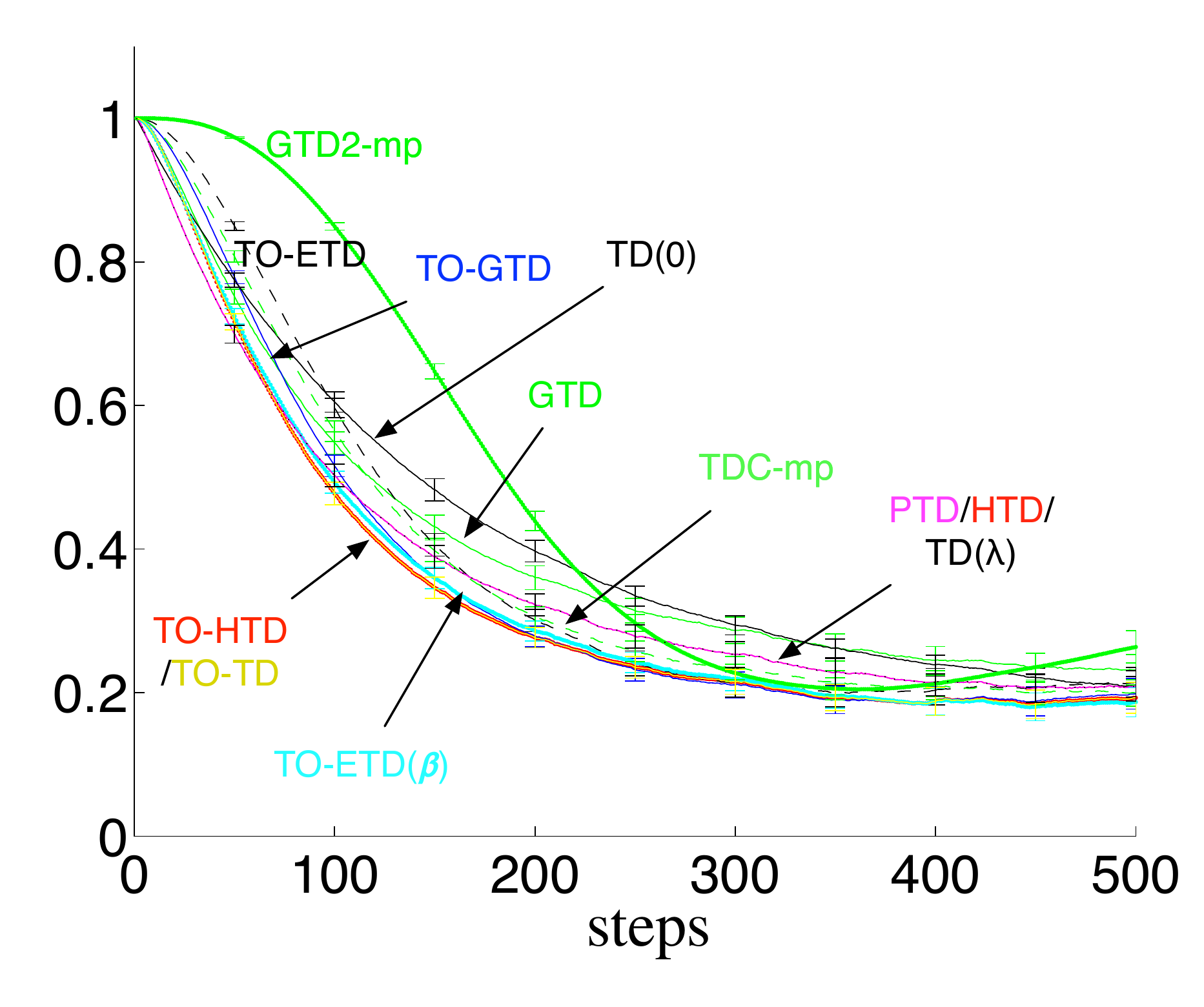} &    \includegraphics[width=\gwidth]{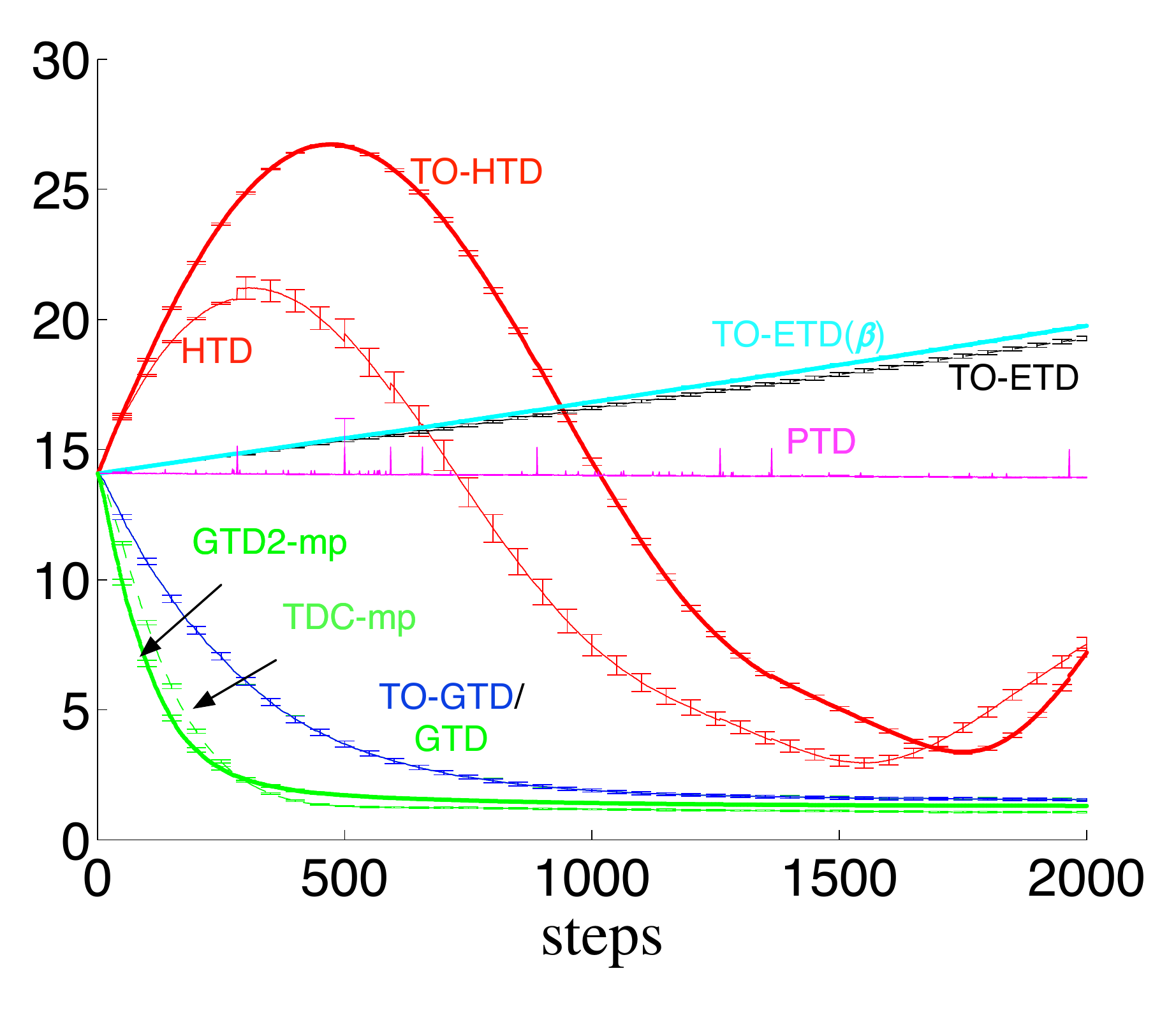}\\
\includegraphics[width=\gwidth]{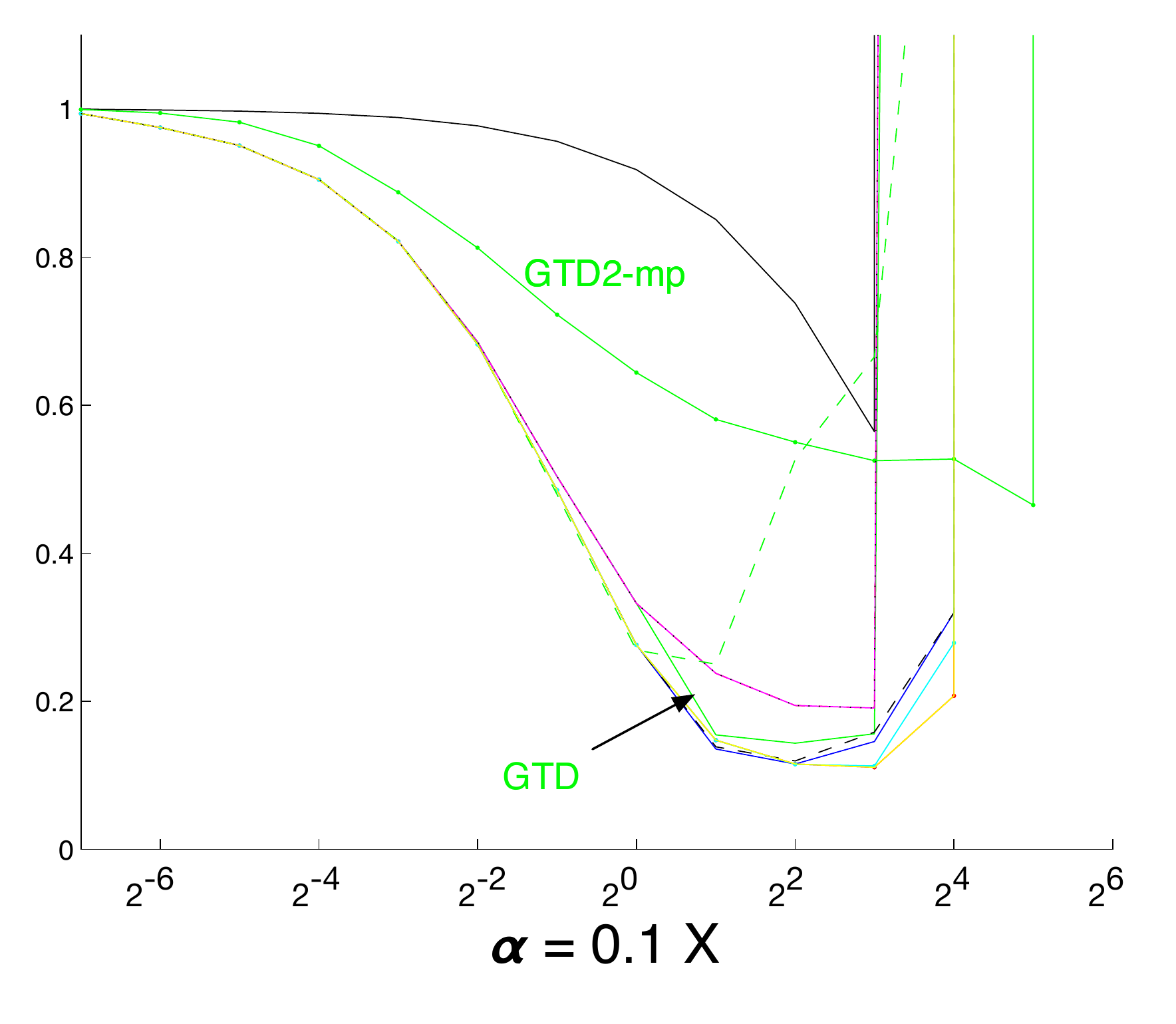} & \includegraphics[width=\gwidth]{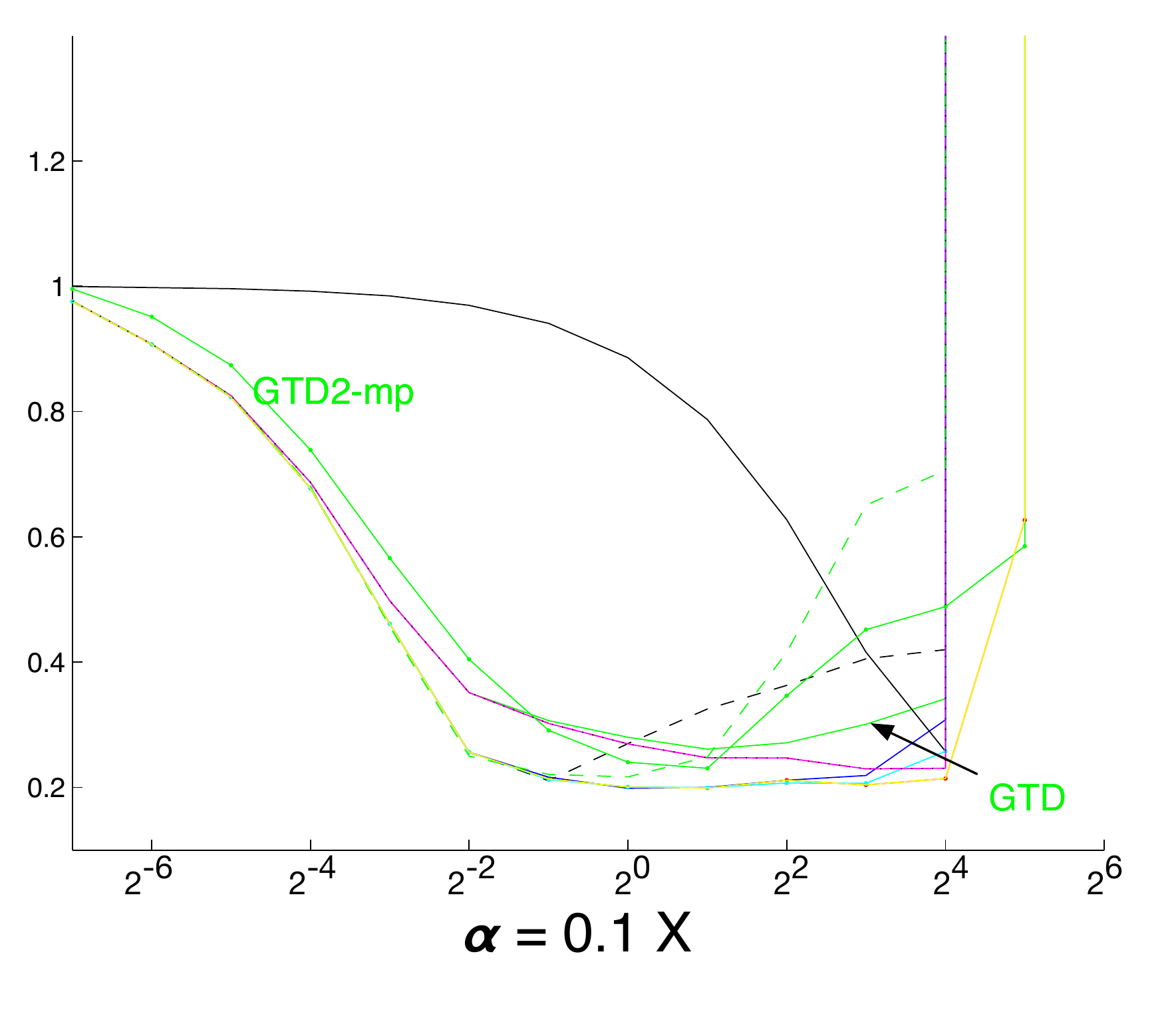} &  \includegraphics[width=\gwidth]{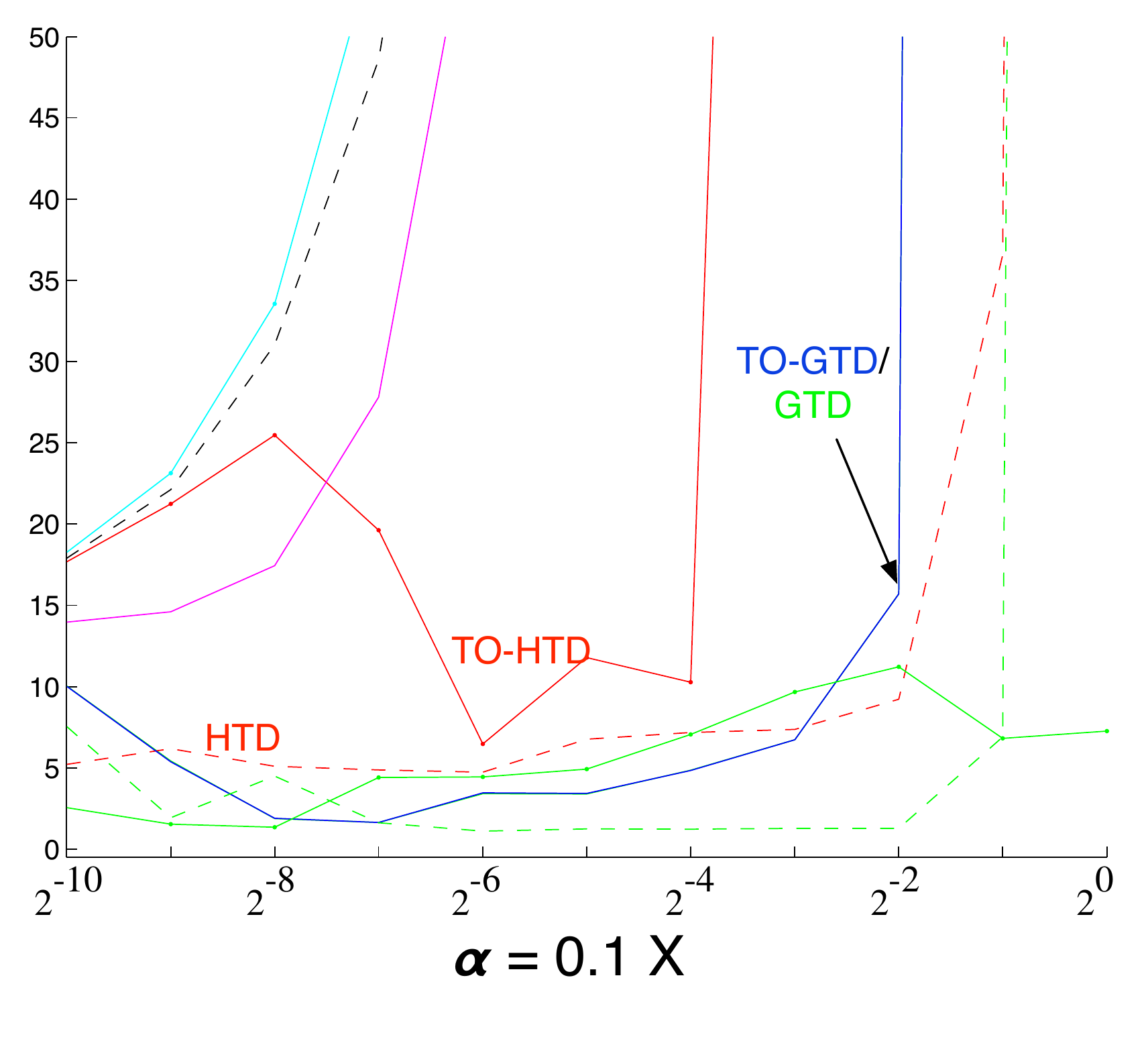}\\
\includegraphics[width=\gwidth]{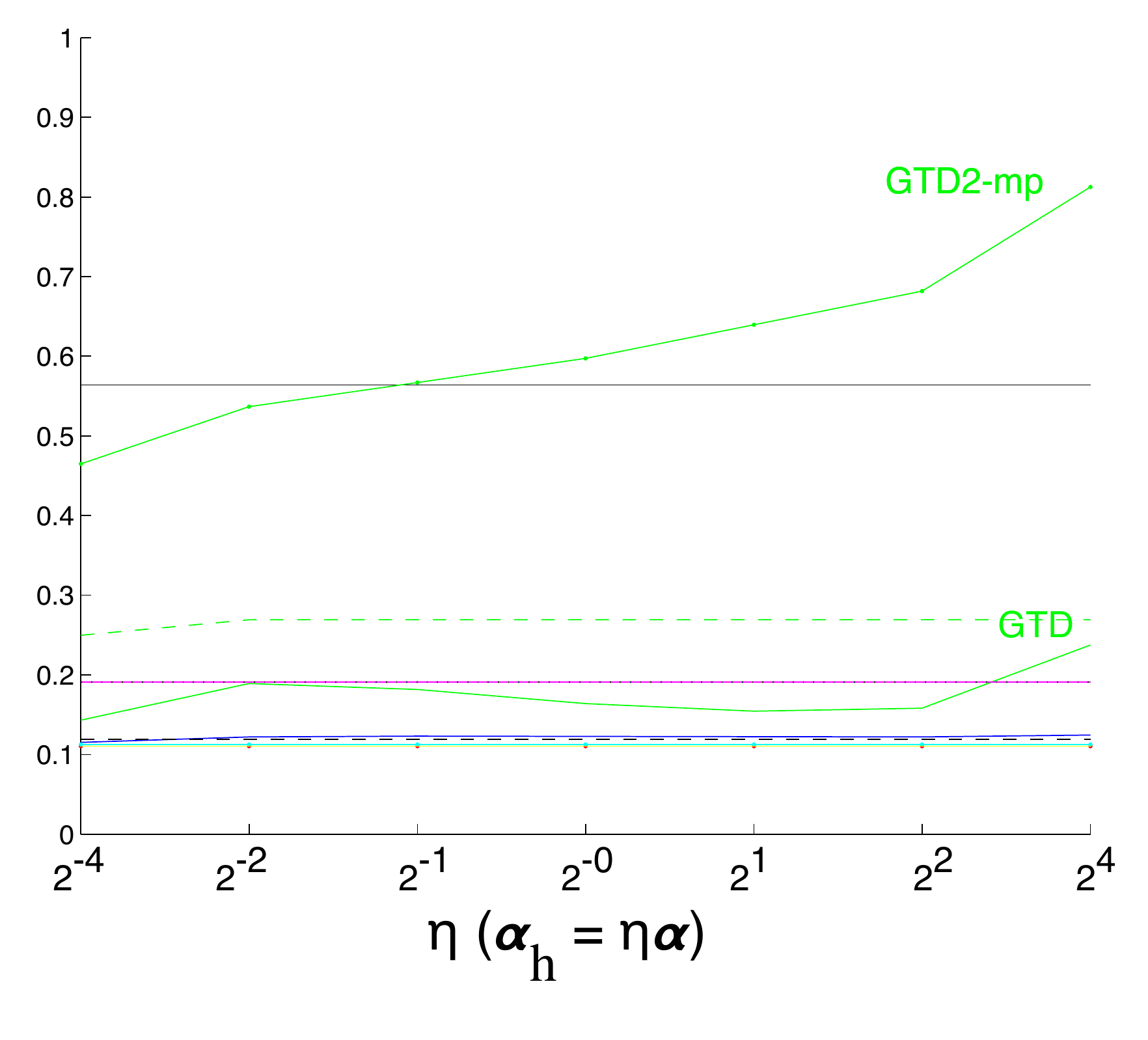} & \includegraphics[width=\gwidth]{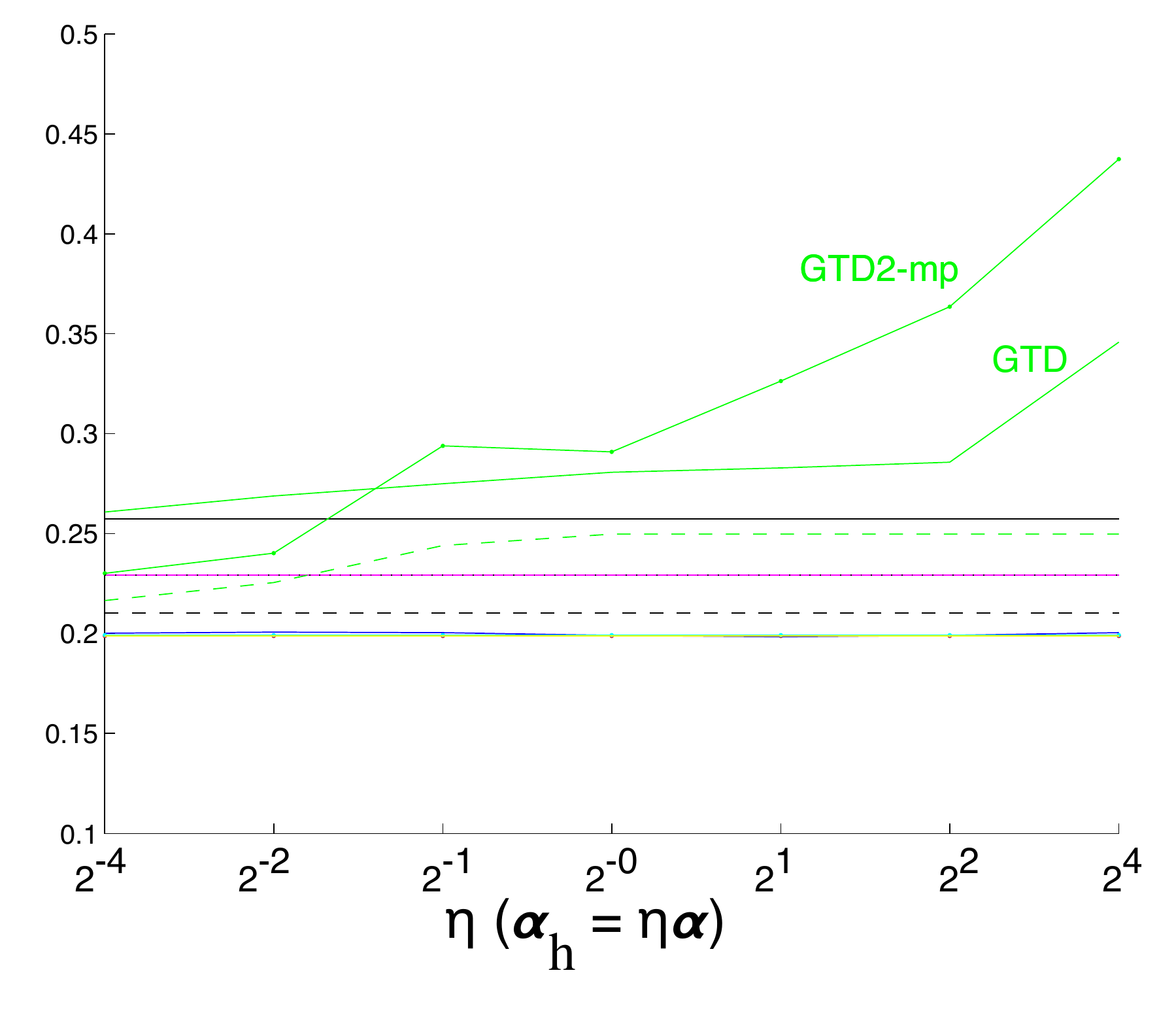} &   \includegraphics[width=\gwidth]{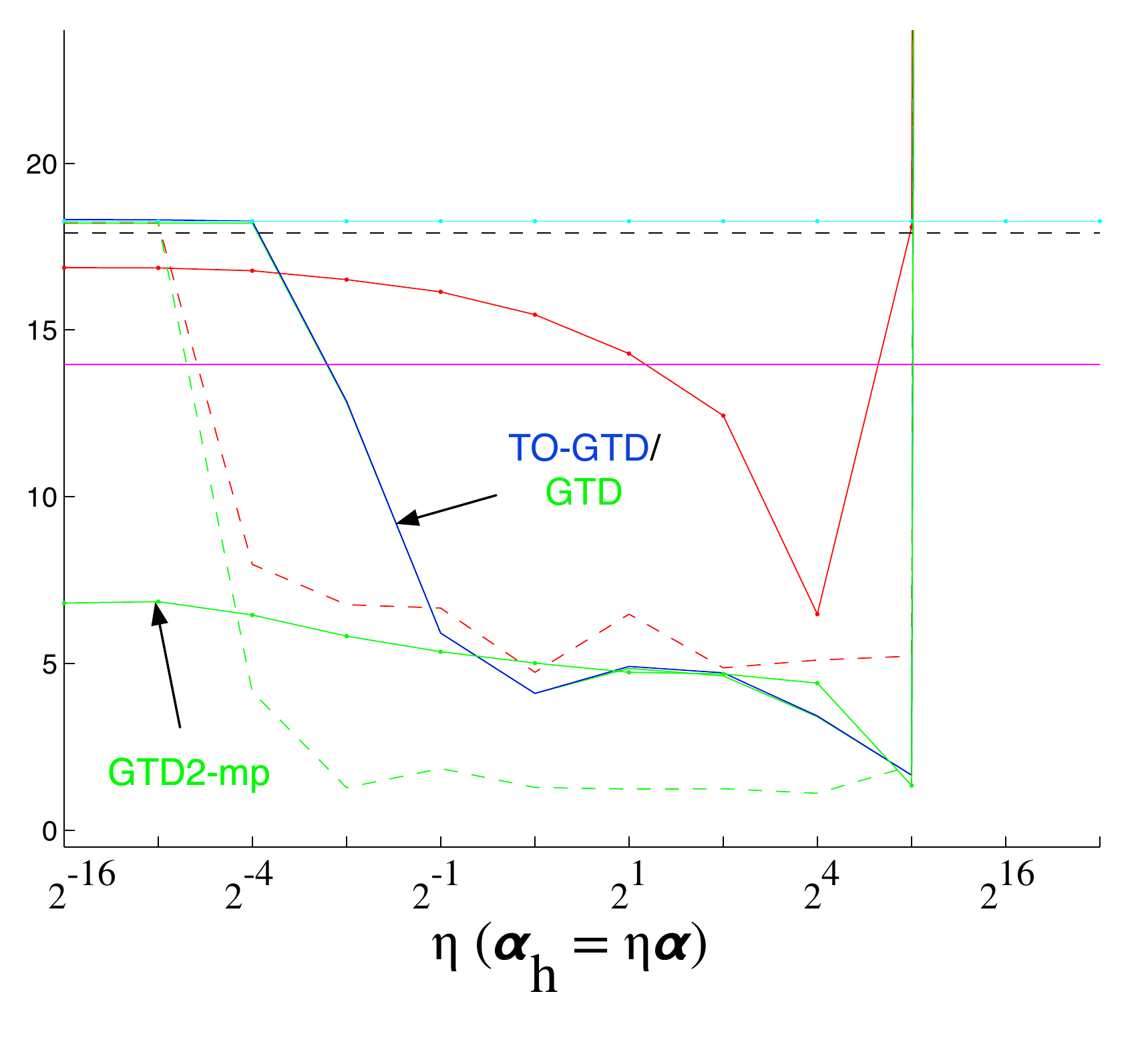} \\
\includegraphics[width=\gwidth]{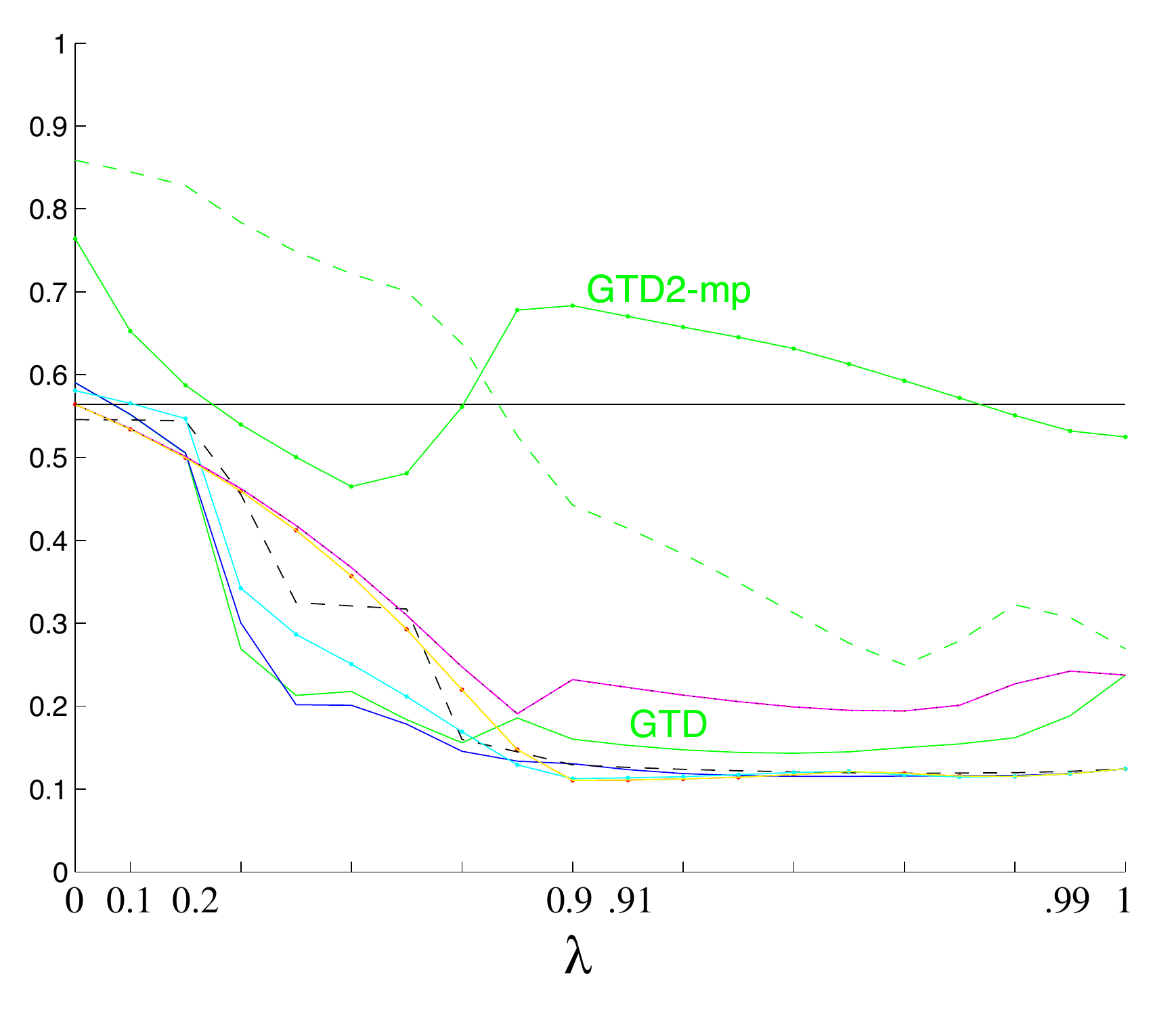}  & \includegraphics[width=\gwidth]{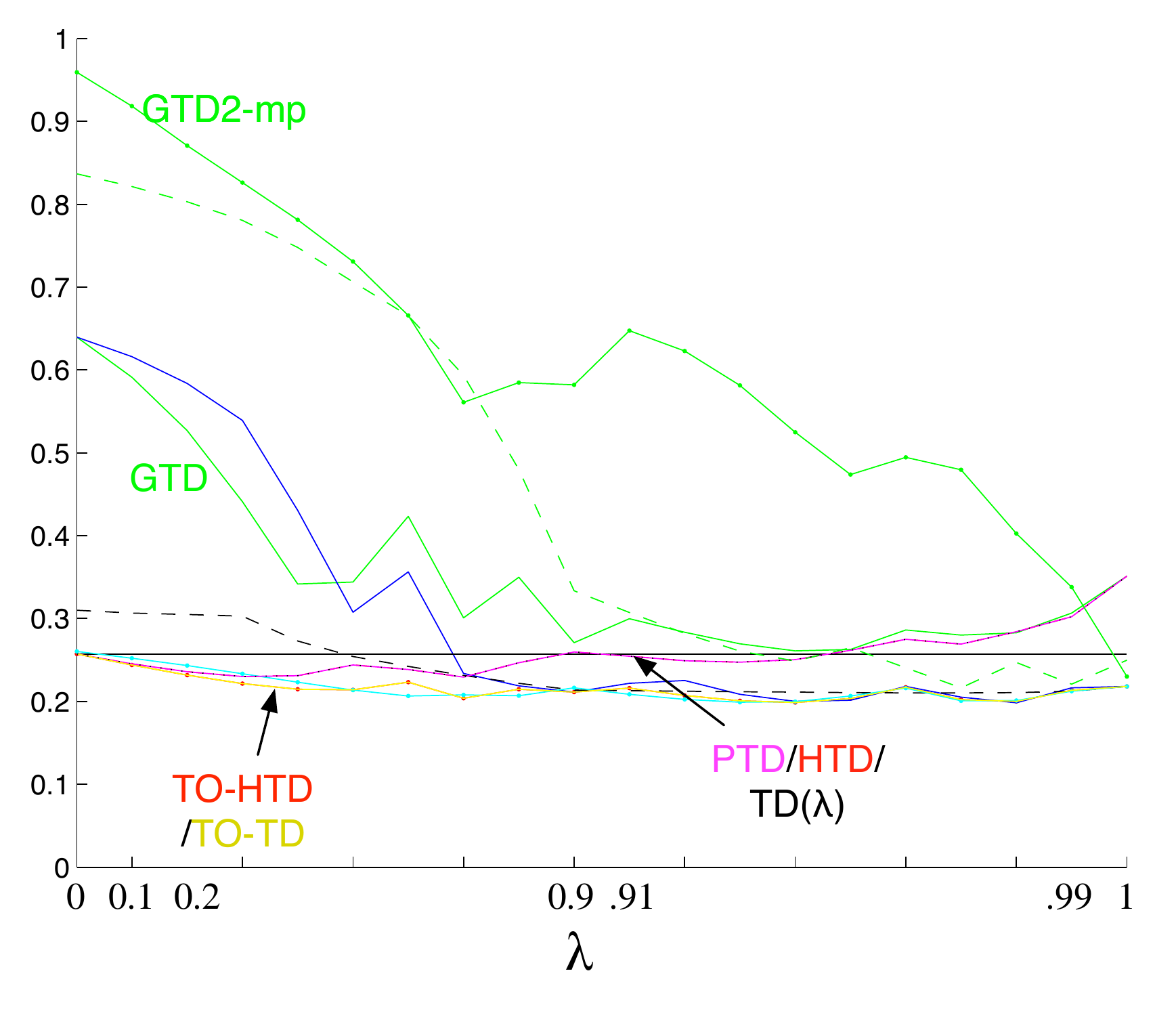}  &  \includegraphics[width=\gwidth]{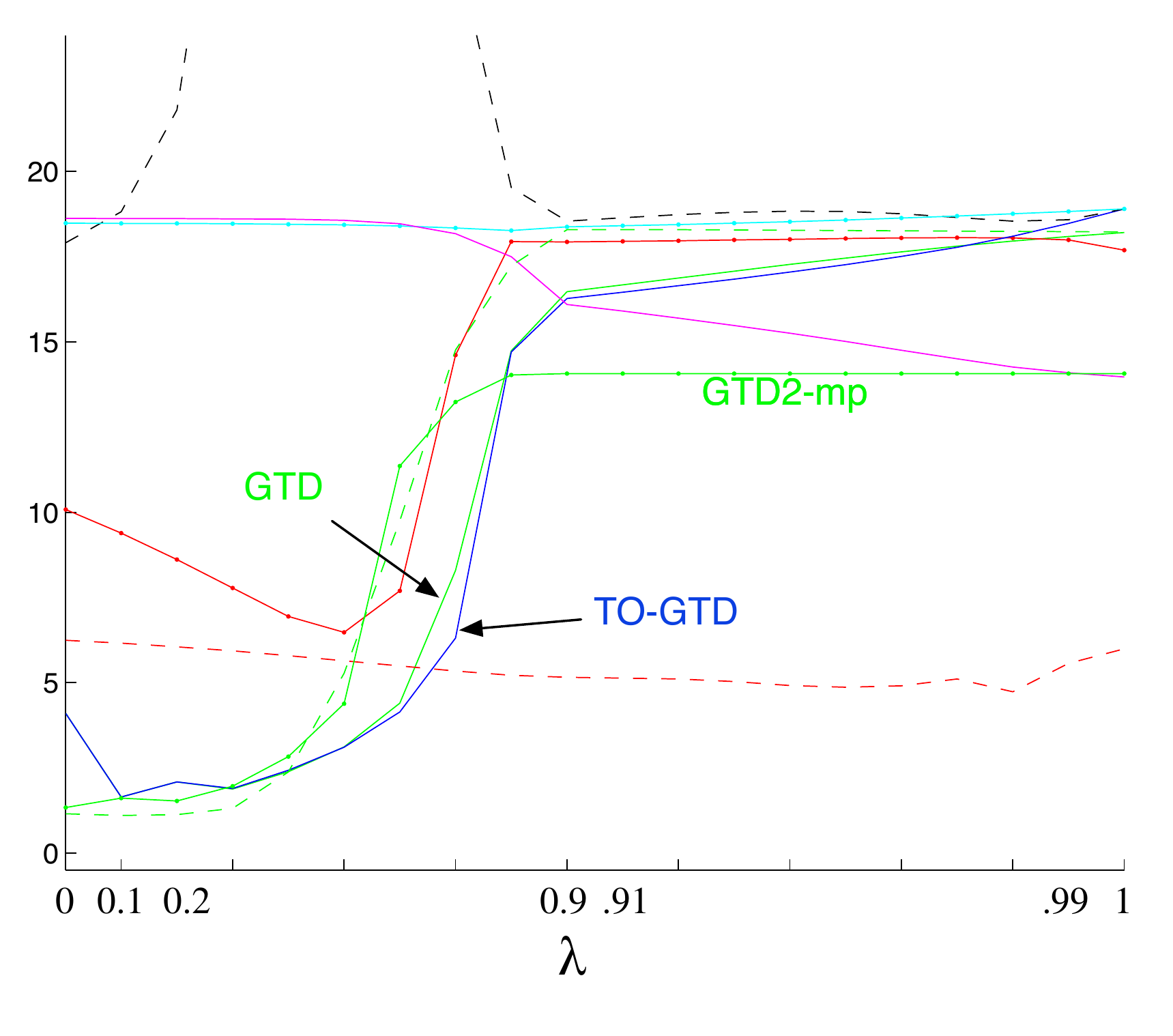}\\
 \end{tabular}
 \caption{\textbf{On-policy} performance on random MDPs with two different representations and \textbf{off-policy} performance on Baird's counterexample. All plots report mean absolute value error averaged over 100 runs and 30 random MDPs, and 500 runs for Baird's. The plots are organized in columns left to right corresponding to results on random MDPs with tabular and binary features, and results on Baird's counterexample. The plots are also organized in rows from top to bottom corresponding to learning curves, $\alpha$,  $\eta$, and $\lambda$ sensitivity.}\label{figure_onpolicy}
 \end{figure*}

\section{Discussion}

There are three broad conclusions suggested by our results. First, we could not clearly demonstrate the supposed superiority of TD($\lambda$) over gradient TD methods in the on-policy setting. In both tabular and aliased feature settings GTD($\lambda$) achieved faster learning and superior parameter sensitivity compared to TD, PTD, and HTD. Notably, the $\eta$-sensitivity of the GTD algorithm was very reasonable in both domains, however, large $\eta$ were required to achieve good performance on Baird's for both GTD($\lambda$) and TO-GTD($\lambda$). Our on-policy experiments with binary features did indicate a slight advantage for TD($\lambda$), PTD, and HTD, and that PTD and HTD exhibit zero sensitivity to the choice of $\alpha_\h$ as expected. In off-policy learning there is little difference between GTD($\lambda$) and PTD and HTD. Our results combined with the prior work of Dann et al. \cite{dann2014policy} suggest that the advantage of conventional TD($\lambda$) over gradient TD methods, in on-policy learning, is limited to specific domains.   

Our second conclusion, is that the new mirror prox methods achieved poor performance in most settings except Baird's counterexample. Both GTD2-mp and TDC-mp achieved the best performance in Baird's counterexample. We hypothesize that the two-step gradient computation more effectively uses the transition to state 7, and so is ideally suited to the structure of the domain\footnote{Baird's counterexample uses a specific initialization of the primary weights: far from one of the true solutions $\qw = \vec{0}$.}. However, the GTD2-MP method performed worse than off-policy TD(0) in all off-policy random MDP domains, while the learning curves of TDC-mp exhibited higher variance than other methods in all but the on-policy binary case and high parameter sensitivity across all settings except Baird's. This does not seem to be a consequence of the extension to eligibility traces because in all cases except Baird's, both TDC-mp and GTD2-mp performed best with $\lambda>0$. Like GTD and HTD, the mirror prox methods would likely have performed better with values of $\alpha_\h > \alpha$, however, this is undesirable because larger $\alpha_\h$ is required to ensure good performance in some off policy domains, such as Baird's (e.g., $\eta = 2^8$). 

Third and finally, several methods exhibited non-convergent behavior on Baird's counterexample. All methods that exhibited reliable error reduce in Baird's did so with $\lambda$ near zero, suggesting that eligibility traces are of limited use in these more extreme off-policy domains. In the case of PTD, non-convergent behavior is not surprising since our implementation of this algorithm does not include gradient correction---a possible extension suggested by the authors \cite{sutton2014anew}---and thus is only guaranteed to converge under off-policy sampling in the tabular case. For the emphatic TD methods the performance on Baird's remains a concern, especially considering how well TO-ETD($\lambda,\beta$) performed in all our other experiments. The addition of the $\beta$ parameter appears to significantly improve TO-ETD in the off-policy domain with binary features, but could not mitigate the large variance in $\rho$ produced by the counterexample. It is not clear if this bad behavior is inherent to emphatic TD methods\footnote{The variance of TO-ETD has been examined before in two state domains \cite{sutton2015anemphatic}. ETD is thought to have higher variance that other TD algorithms due to the emphasis weighting.}, 
or could be solved by more careful specification of the state-based interest function. In our implementation, we followed the original author's recommendation of setting the interest for each state to 1.0 \cite{sutton2015anemphatic}, because all our domains were discounted and continuing. Additionally, both HTD($\lambda$) and TO-HTD($\lambda$) did not diverge on Baird's, but performance was less than satisfactory to say the least.   

Overall, the conclusions implied by our empirical study are surprisingly clear. If guarding against large variance due to off-policy sampling is a chief concern, then GTD($\lambda$) and TO-GTD($\lambda$) should be preferred. Between the two, GTD($\lambda$) should be preferred if computation is at a premium. If poor performance in problems like Baird's is not a concern, then TO-ETD($\lambda,\beta$) was clearly the best across our experiments, and exhibited nearly the best runtime results. TO-ETD($\lambda$) on the other hand, exhibited high variance in off-policy domains, and sharp parameter sensitivity, indicating parameter turnng of emphatic methods may be an issue in practice.

\section{Appendix}
Additional results and analysis can be found in the full version of the paper: http://arxiv.org/abs/1602.08771.
\newpage

\bibliographystyle{abbrv}
\bibliography{paper.bib}

\newpage
 
 \appendix
 
\section{Algorithms}

The original ETD($\lambda$) algorithm as proposed by Sutton et al. (2015) is an not entirely obvious manipulation of the true online ETD($\lambda$) described above and used in our experiments. The difference is in the definition of the eligibility trace and the primary weight update. To achieve the original ETD($\lambda$) algorithm modify the above true-online ETD($\lambda$) algorithm to use $$\e_t \leftarrow \rho_t(\gamma_t\lambda\e_{t-1} + M_t\vx_t),$$ and $$\Delta \qw \leftarrow \alpha\delta_t\e_t.$$ 

In all the algorithms that follow, we assume $\qw_0$, $\h_0$ are initialized arbitrarily, and eligibility traces are initialized to a vector of zeros (e.g., $\e_{-1} = \vec{0}$).\\\\ 
\subsection*{TD($\lambda$)}
\begin{flalign*}
&\delta_t \defeq R_{t+1} + \gamma_{t+1} \vx_{t+1}^\top\qw_t - \vx_{t}^\top\qw_t&\\ 
&\e_t \leftarrow \lambda_t \gamma_t \e_{t-1} + \vx_t&\\
&\Delta \qw \leftarrow \alpha\delta_t\e_t& 
\end{flalign*}

\subsection*{True-online TD($\lambda$)}
\begin{flalign*} 
&v^\prime \defeq \qw_t^\top \vx_{t+1}&\\
&\delta_t \defeq R_{t+1} + \gamma_{t+1}v^\prime - v&\\
&\e_t \leftarrow \gamma_t\lambda_t\e_{t-1} + \alpha \left[ 1-\gamma_t\lambda_t \e_{t-1}^\top\vx_t\right]\vx_t&\\
&\Delta \qw \leftarrow \delta_t\e_t +  \alpha[v - \qw_t^\top\vx_t]\vx_t&\\
&v \leftarrow v^\prime \hspace{1cm}\triangleright v~\text{initialized to}~0 & 
\end{flalign*}

\subsection*{GTD($\lambda$)}
\begin{flalign*} 
&\delta_t \defeq R_{t+1} + \gamma_{t+1} \vx_{t+1}^\top\qw_t - \vx_{t}^\top\qw_t&\\ 
&\e_t \leftarrow \rho_t (\lambda_t \gamma_t \e_{t-1} + \vx_t)&\\
&\Delta \qw \leftarrow \alpha\left[ \delta_t\e_t  - \gamma_{t+1}(1-\lambda_{t+1})(\e_t^\top \h_t) \vx_{t+1} \right]&\\
&\Delta \h \leftarrow \alpha_\h\left[ \delta_t\e_t  - (\vx_t^\top \h_t) \vx_{t} \right]&
\end{flalign*}

\subsection*{True-online GTD($\lambda$)}
\begin{flalign*} 
&\delta_t \defeq R_{t+1} + \gamma_{t+1} \vx_{t+1}^\top\qw_t - \vx_{t}^\top\qw_t&\\ 
&\e_t \leftarrow \rho_t\left[\lambda_t \gamma_t \e_{t-1} + \alpha_t \left(1-\rho_t \gamma_t \lambda_t (\vx_t^\top \e_{t-1})\right) \vx_t\right]&\\
&\emu_t \leftarrow \rho_t (\lambda_t \gamma_t \emu_{t-1} + \vx_t)&\\
&\e^h_t \leftarrow \rho_{t-1}\lambda_t \gamma_t \e^h_{t-1} + \alpha_\h \left(1-\rho_{t-1} \gamma_t \lambda_t (\vx_t^\top \e^h_{t-1})\right) \vx_t&\\
&\mathbf{d} \defeq \delta_t \e_t + (\e_t - \alpha \rho_t \vx_t)(\qw_t - \qw_{t-1})^\top \vx_t&\\
&\Delta \qw_t \leftarrow \mathbf{d}  - \alpha\gamma_{t+1}(1 - \lambda_{t+1}) (\h_t^\top\emu_t) \vx_{t+1}&\\
&\Delta \h_t \leftarrow  \rho_t \delta_t \e^h_t - \alpha_\h (\vx_t^\top\h_t) \vx_t&
 \end{flalign*}
 
\subsection*{PTD($\lambda$)}
\begin{flalign*} 
&\delta_t \defeq R_{t+1} + \gamma_{t+1} \vx_{t+1}^\top\qw_t - \vx_{t}^\top\qw_t&\\ 
&\e_t \leftarrow \rho_t (\lambda_t \gamma_t \e_{t-1} + \vx_t)&\\
&\Delta \qw_t \leftarrow \alpha\delta_t\e_t + (\rho_t-1)\h_t&\\
&\h_{t+1} \leftarrow \gamma_t\lambda_t(\rho_t \h_t + \alpha\bar{\delta}_t\e_t)&\\
&\bar{\delta}_t \defeq R_{t+1} + \qw_t^{\top} \vx_{t+1} - \qw_t^{\top} \vx_t&
 \end{flalign*} 
 
\subsection*{True-Online ETD($\lambda$)}
\begin{flalign*} 
&\delta_t \defeq R_{t+1} + \gamma_{t+1} \vx_{t+1}^\top\qw_t - \vx_{t}^\top\qw_t&\\ 
&F_t \leftarrow \rho_{t-1}\gamma_t F_{t-1} + I_t \hspace{2cm}\triangleright F_{-1} = 0&\\
&M_t \defeq \lambda_{t}I_t + (1-\lambda_t)F_{t}&\\
&\e_t \leftarrow \rho_t\gamma_t\lambda_t\e_{t-1} + \rho_t\alpha M_t(1-\rho_t \gamma_t \lambda_t (\vx_t^\top \e_{t-1})) \vx_t&\\
&\Delta \qw \leftarrow \delta_t\e_t  + (\e_t - \alpha M_t \rho_t \vx_{t})(\qw_t - \qw_{t-1})^\top \vx_t&
 \end{flalign*} 
 
\subsection*{True-Online ETD($\beta,\lambda$)}
\begin{flalign*} 
&\delta_t \defeq R_{t+1} + \gamma_{t+1} \vx_{t+1}^\top\qw_t - \vx_{t}^\top\qw_t&\\ 
&F_t \leftarrow \rho_{t-1}\beta_t F_{t-1} + I_t \hspace{2cm}\triangleright F_{-1} = 0&\\
&M_t \defeq \lambda_{t}I_t + (1-\lambda_t)F_{t}&\\
&\e_t \leftarrow \rho_t\gamma_t\lambda_t\e_{t-1} + \rho_t\alpha M_t(1-\rho_t \gamma_t \lambda_t (\vx_t^\top \e_{t-1})) \vx_t&\\
&\Delta \qw \leftarrow \delta_t\e_t  + (\e_t - \alpha M_t \rho_t \vx_{t})(\qw_t - \qw_{t-1})^\top \vx_t&
 \end{flalign*} 
 
\subsection*{GTD2($\lambda$)-MP}
\begin{flalign*} 
&{\delta}_t \defeq R_{t+1} + \gamma_{t+1}\qw_t^{\top} \vx_{t+1} - \qw_t^{\top} \vx_t&\\
&\e_t \leftarrow \rho_t (\lambda_t \gamma_t \e_{t-1} + \vx_t)&\\
&\h_{t+\frac{1}{2}} \leftarrow \h_t + \alpha_\h\left[ \delta\e_t - (\h_t^\top\vx_t)\vx_t \right]&\\
&\qw_{t+\frac{1}{2}} \leftarrow \qw_t + \alpha\left[ (\h_t^\top\vx_t)\vx_t - \gamma_{t+1}(1-\lambda_{t+1})(\h_t^\top\e_t)\vx_{t+1}\right]&\\
&{\delta}_{t+\frac{1}{2}} \defeq R_{t+1} + \gamma_{t+1}\qw_{t+\frac{1}{2}}^{\top} \vx_{t+1} - \qw_{t+\frac{1}{2}}^{\top} \vx_t&\\
&\Delta \qw \leftarrow \alpha\left[ (\h_{t+\frac{1}{2}}^\top\vx_t)\vx_t - \gamma_{t+1}(1-\lambda_{t+1})(\h_{t+\frac{1}{2}}^\top\e_t)\vx_{t+1}\right]&\\
&\Delta \h \leftarrow \alpha_\h\left[ \delta_{t+\frac{1}{2}}\e_t  - (\vx_t^\top \h_{t+\frac{1}{2}}) \vx_{t} \right]&
\end{flalign*}
 
\subsection*{TDC($\lambda$)-MP}
\begin{flalign*} 
&{\delta}_t \defeq R_{t+1} + \gamma_{t+1}\qw_t^{\top} \vx_{t+1} - \qw_t^{\top} \vx_t&\\
&\e_t \leftarrow \rho_t (\lambda_t \gamma_t \e_{t-1} + \vx_t)&\\
&\h_{t+\frac{1}{2}} \leftarrow \h_t + \alpha_\h\left[ \delta\e_t - (\h_t^\top\vx_t)\vx_t \right]&\\
&\qw_{t+\frac{1}{2}} \leftarrow \qw_t + \alpha\left[ \delta\e_t - \gamma_{t+1} (1-\lambda_{t+1})(\h_t^\top\e_t)\vx_{t+1} \right]&\\
&{\delta}_{t+\frac{1}{2}} \defeq R_{t+1} + \gamma_{t+1}\qw_{t+\frac{1}{2}}^{\top} \vx_{t+1} - \qw_{t+\frac{1}{2}}^{\top} \vx_t&\\
&\Delta \qw \leftarrow \alpha\left[ \delta_{t+\frac{1}{2}}\e_t  - \gamma_{t+1}(1-\lambda_{t+1})(\e_t^\top \h_{t+\frac{1}{2}}) \vx_{t+1} \right]&\\
&\Delta \h \leftarrow \alpha_\h\left[ \delta_{t+\frac{1}{2}}\e_t  - (\vx_t^\top \h_{t+\frac{1}{2}}) \vx_{t} \right]&
\end{flalign*} 

\subsection*{HTD($\lambda$)}
\begin{flalign*} 
&{\delta}_t \defeq R_{t+1} + \gamma_{t+1}\qw_t^{\top} \vx_{t+1} - \qw_t^{\top} \vx_t&\\
&\e_t \leftarrow \rho_t (\lambda_t \gamma_t \e_{t-1} + \vx_t)&\\
&\emu_t \leftarrow \lambda_t \gamma_t \emu_{t-1} + \vx_t&\\
&\Delta \qw_t \leftarrow \alpha\left[ \delta_t\e_t  + (\vx_t - \gamma_{t+1} \vx_{t+1})(\e_t - \emu_t)^\top \h_t \right]&\\
&\Delta \h_t \leftarrow \alpha_\h\left[ \delta_t\e_t  - (\vx_t - \gamma_{t+1} \vx_{t+1})({\emu}_t^\top \h_t) \right]&
 \end{flalign*}
 
\subsection*{True-online HTD($\lambda$)}
\begin{flalign*} 
&{\delta}_t \defeq R_{t+1} + \gamma_{t+1}\qw_t^{\top} \vx_{t+1} - \qw_t^{\top} \vx_t&\\
&\e_t \leftarrow \rho_t (\lambda_t \gamma_t \e_{t-1} + \vx_t)&\\
&\emu_t \leftarrow \lambda_t \gamma_t \emu_{t-1} + \vx_t&\\
&\e^o_t \leftarrow \rho_t(\lambda_t \gamma_t \e^o_{t-1} + \alpha (1-\rho_t \gamma_t \lambda_t \vx_t^\top \e^o_{t-1}) \vx_t)&\\
&\mathbf{d} \defeq \delta_t \e^o_t + (\e^o_t - \alpha \rho_t \vx_t)(\qw_t - \qw_{t-1})^\top \vx_t&\\
&\Delta \qw_t \leftarrow \mathbf{d}  + \alpha (\vx_t - \gamma_{t+1} \vx_{t+1})(\e_t - \emu_t)^\top \h_t&\\
&\Delta \h_t \leftarrow  \mathbf{d}  - \alpha_\h (\vx_t - \gamma_{t+1} \vx_{t+1}){\emu}_t^\top \h_t&
 \end{flalign*}

\newpage

 \section{Additional experiments}
 This section includes additional results further analyzing the relative performance of linear TD-based policy evaluation algorithms. 
 The runtime results are as follows, for Figures \ref{figure_onruntime} and Figure \ref{figure_offruntime}.
The graphs indicate a sample efficiency versus time trade-off. For increasing $c$, the algorithms
are given more time per sample to finish computing. If computation is not done within
the allotted time $c$, then the agent continues to finish computation but has essentially paused interaction.
Several possible iterations may be required by the algorithm, if it is slow, until it is done computing on that 
one sample, at which point it is given a new sample. This simulates a real-time decision making tasks, such as mobile robot control.
New samples cannot be processed or buffered for off-line computation while the previous sample is being process. However, multiple samples can be processed per iteration, where the iteration duration is denoted by $c\in\R$. Typically computationally frugal learning methods perform well when $c$ is smaller, because
more samples can be processed per iteration. For example, for $c = 0.1$ on-policy TD(0) processes 101 samples, 
TD($\lambda$) processes 97 samples, TO-TD processes 85 samples and TO-HTD processes 31 samples.
Even though the TD(0) algorithm was allowed to process more samples per iteration, it did not achieve the best performance trade-off, because TO-TD is both sample efficient and computationally efficient. For larger $c$ more time
is available to each algorithm on each iteration. In this case some
of the other algorithms have better performance trade-offs. 
For example, in off-policy learning, HTD with $c=1.25$ effectively ties PTD for the best performance.  


\newcommand{\gwidthtwo}{0.5\textwidth}


  \begin{figure*}
\begin{tabular}{cc}
   \includegraphics[width=\gwidthtwo]{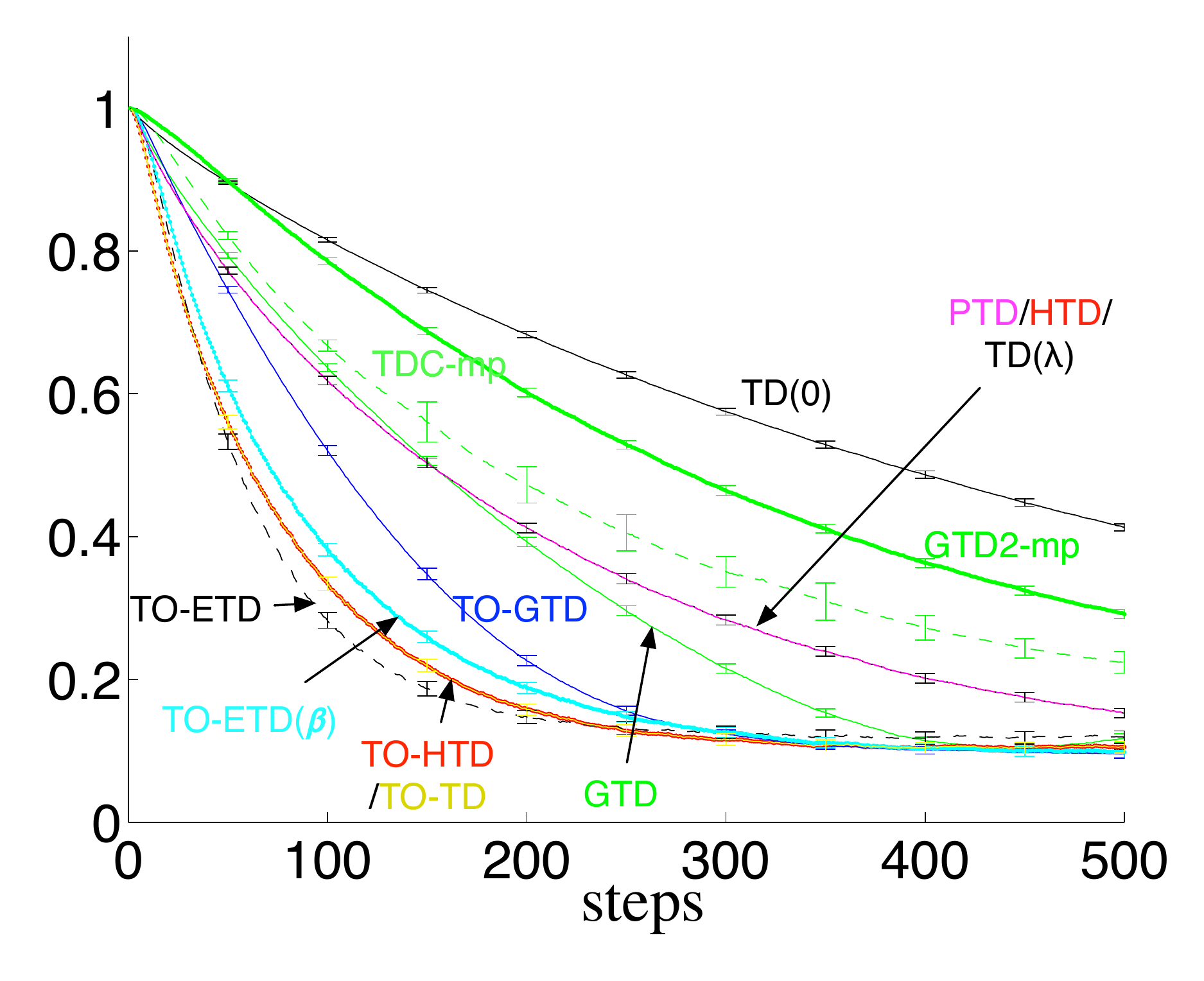} &   \includegraphics[width=\gwidthtwo]{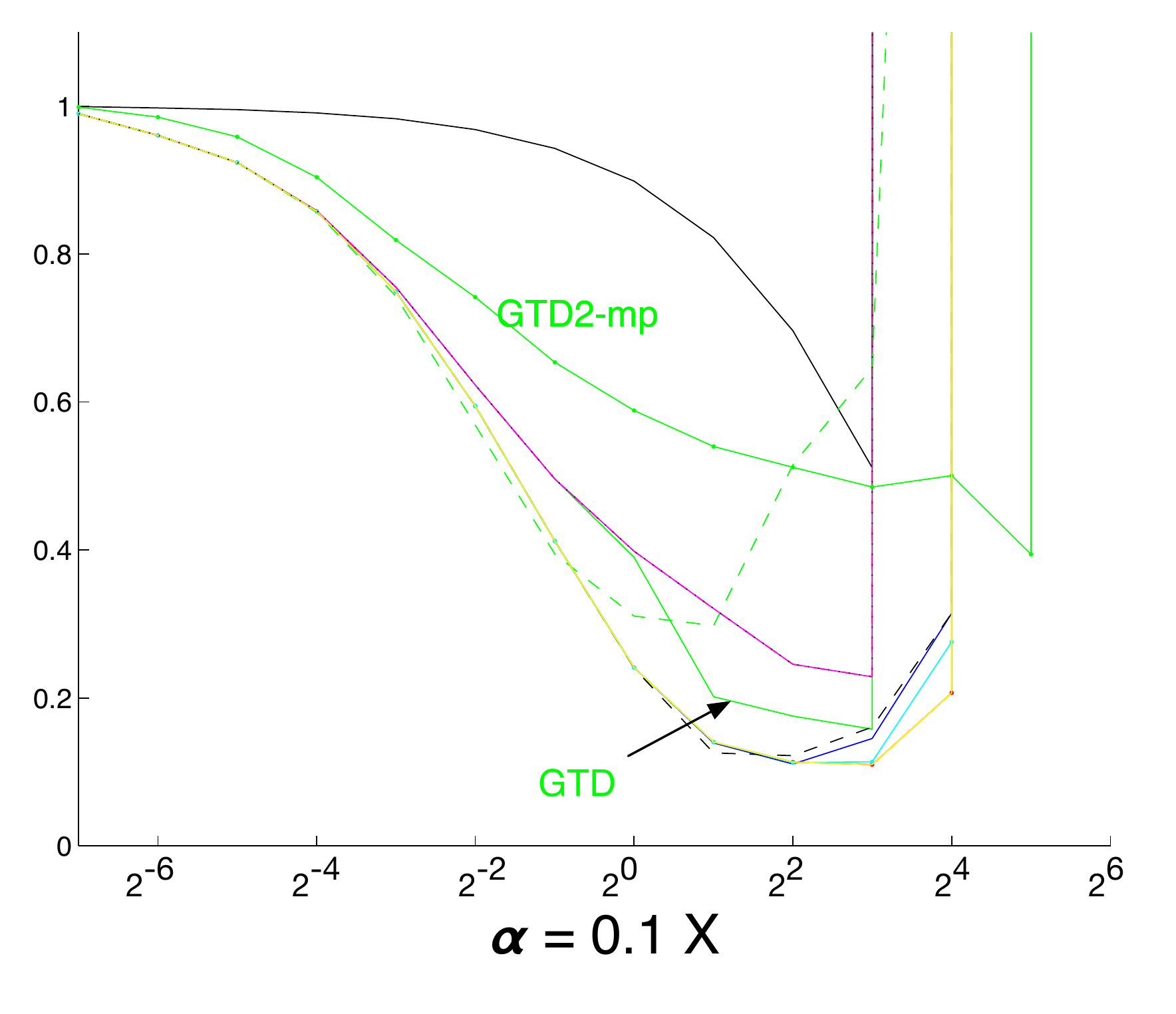} \\
  \includegraphics[width=\gwidthtwo]{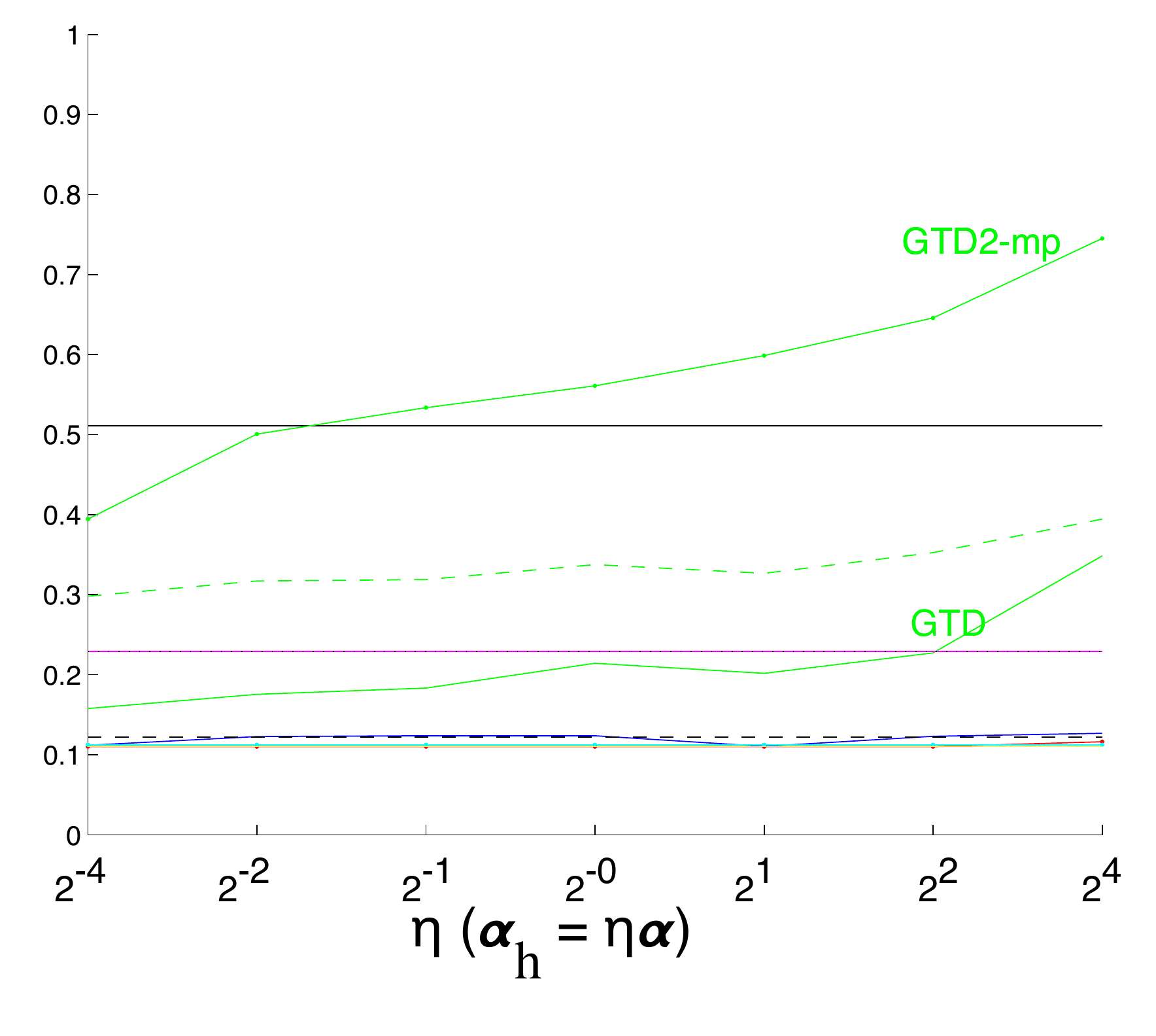} &   \includegraphics[width=\gwidthtwo]{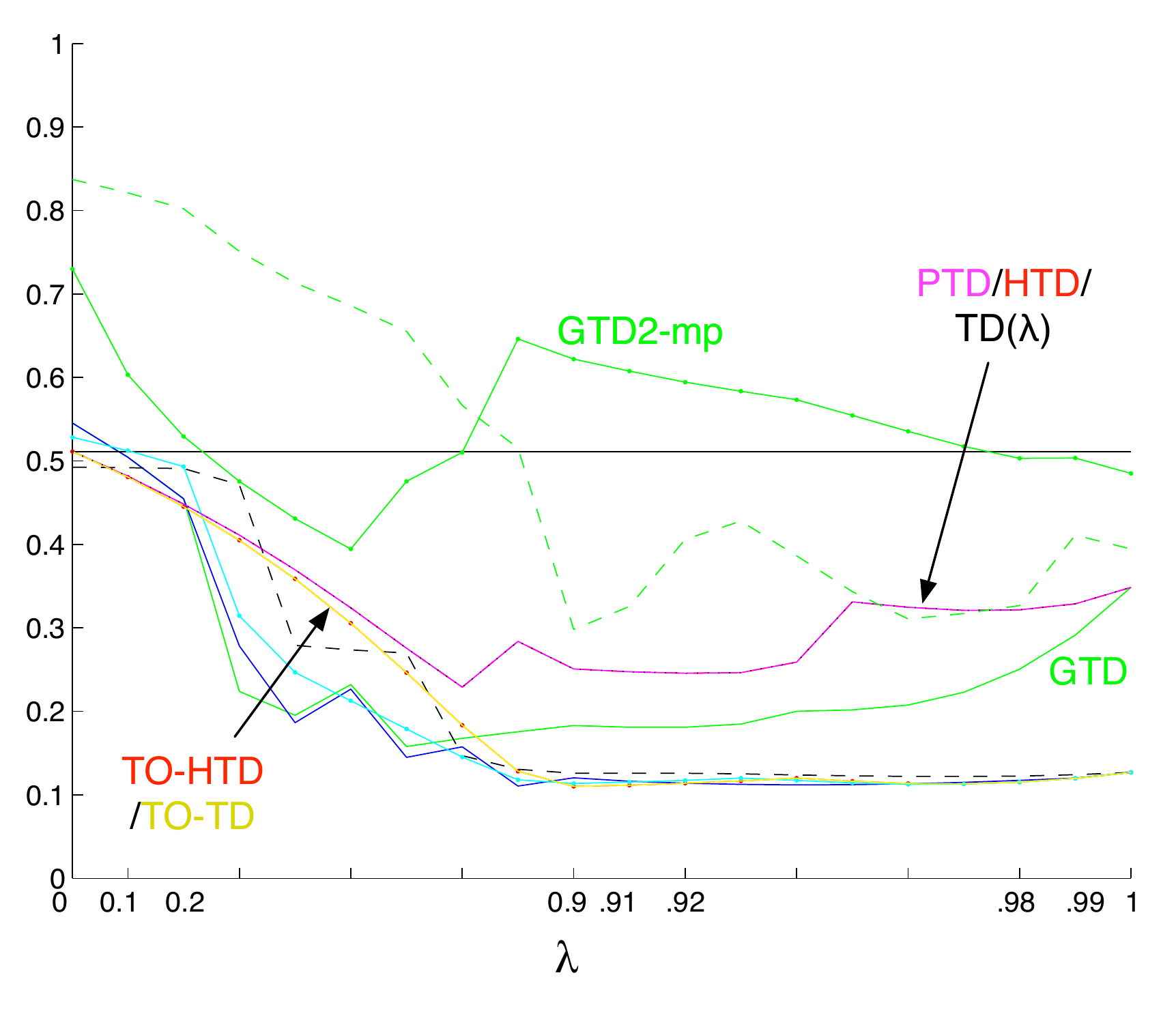} \\
 \end{tabular}
 \caption{\textbf{On-policy} performance on random MDPs with aliased tabular features. All plots report mean absolute value error averaged over 100 runs and 30 MDPs. The top left plot depicts the learning curves with the best parameters found averaged over all random MDP instances. The remaining graphs depict each algorithms parameter sensitivity in mean absolute value error for  $\alpha$,  $\eta$, and $\lambda$.}
 \end{figure*}

  \begin{figure*}
\begin{tabular}{cc}
   \includegraphics[width=\gwidthtwo]{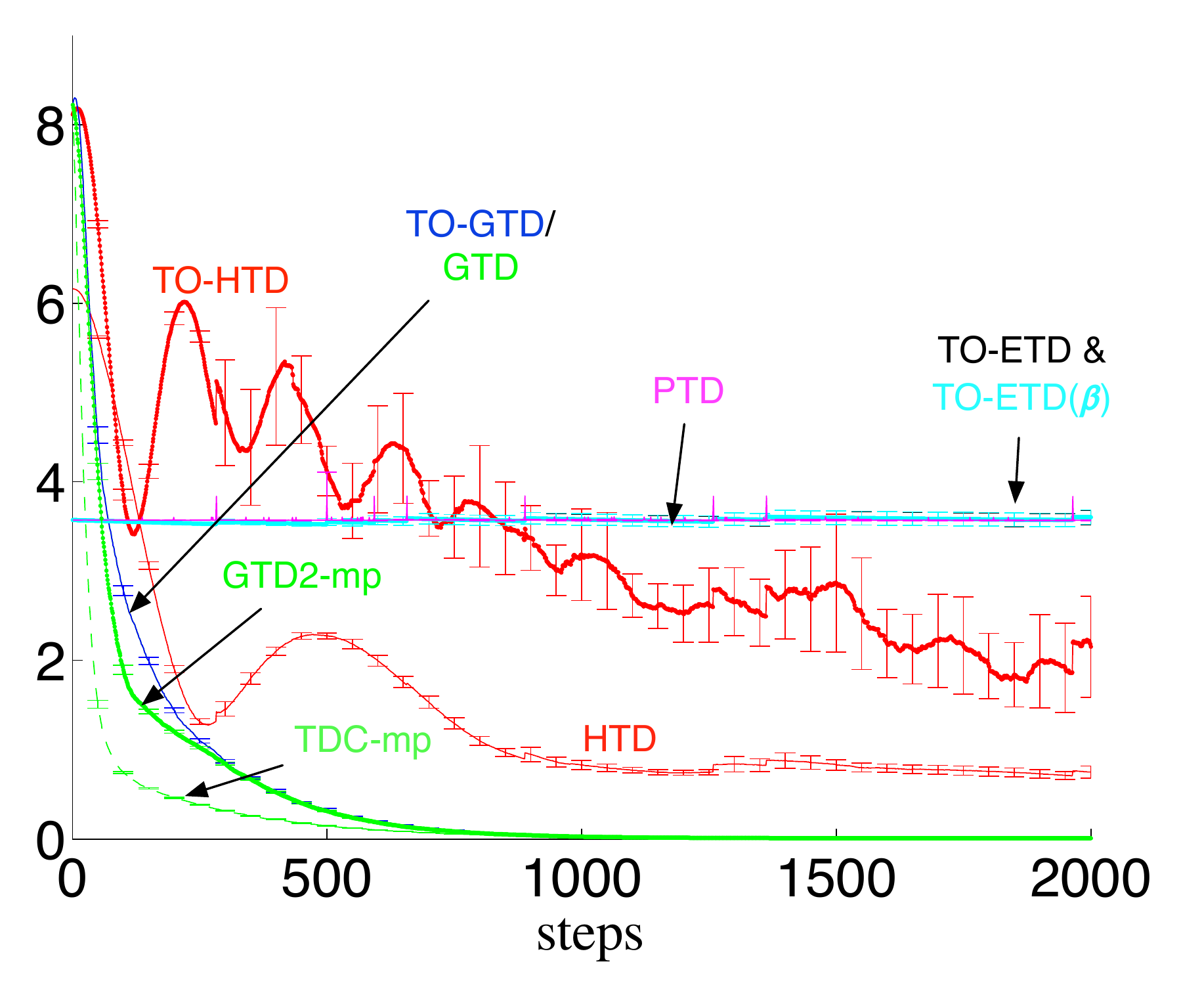} &   \includegraphics[width=\gwidthtwo]{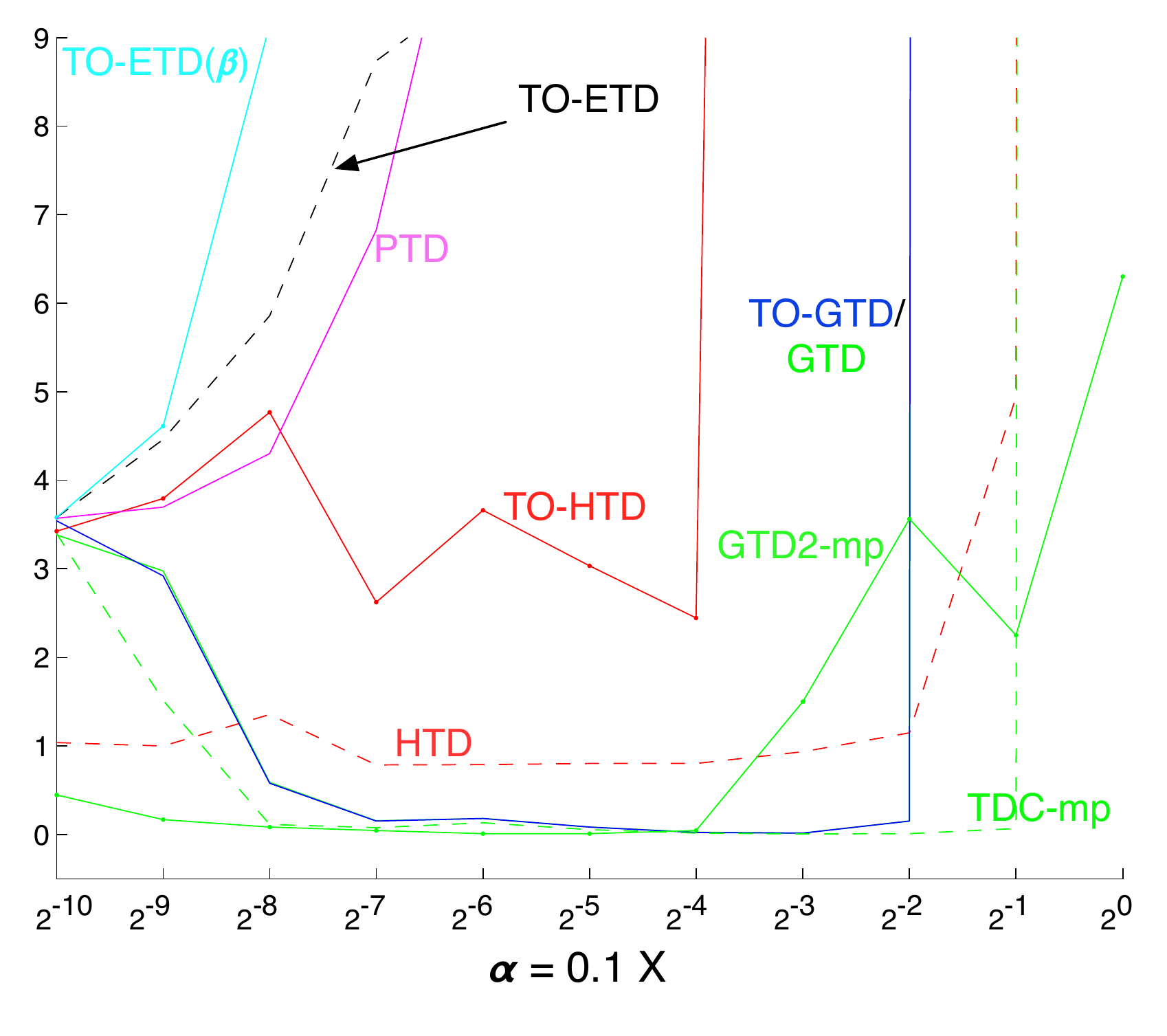} \\
  \includegraphics[width=\gwidthtwo]{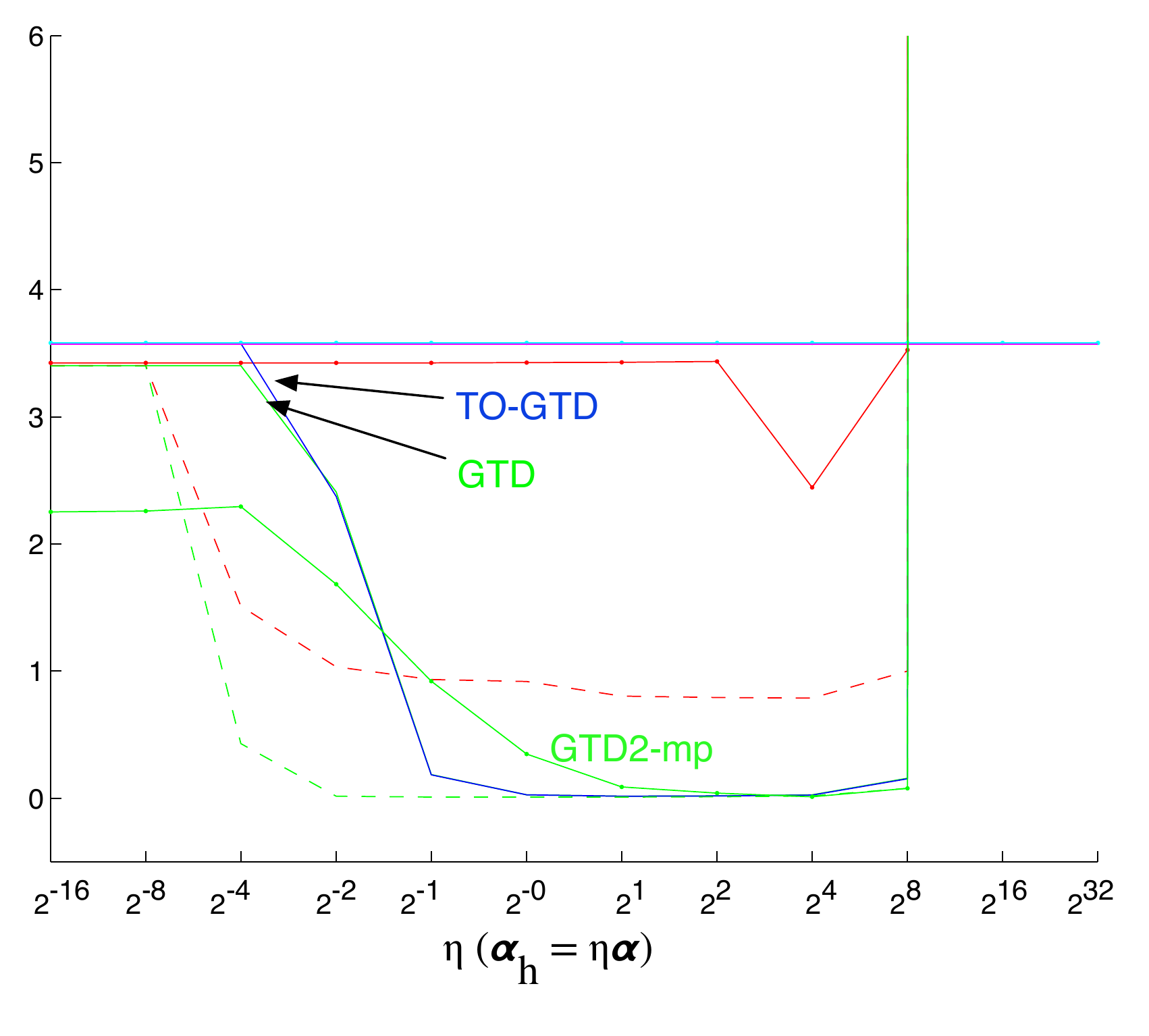} &   \includegraphics[width=\gwidthtwo]{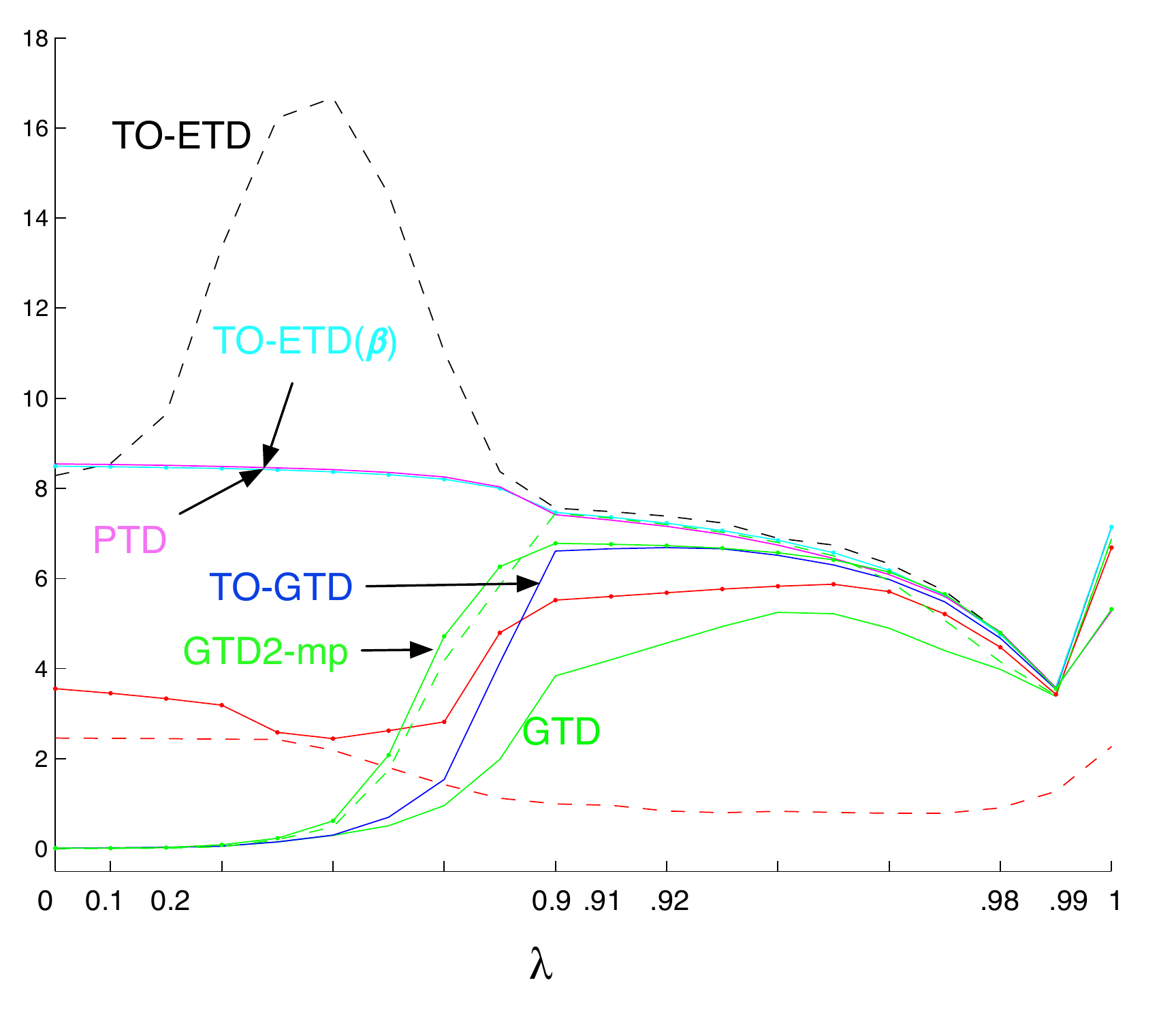} \\
 \end{tabular}
 \caption{{\bf Off-policy} performance on a variant of Baird's 7-state counterexample. All plots report the root mean square projected Bellman error (RMSPBE). See White's thesis (2015) for a detailed explanation of how to compute the MSPBE from the parameters of an MDP. The top left graph reports the RMSPBE is averaged 200 plotted against time, under the best parameter setting found over an extensive sweep. The remaining plots depict the parameter sensitivity, in RMSPBE, of each method with respect to the key algorithm parameters$\alpha$,  $\eta$, and $\lambda$.}
 \end{figure*}
 

 \begin{figure*}
 \hspace*{-0.75cm}
\begin{tabular}{ccc}
   \includegraphics[width=\gwidth]{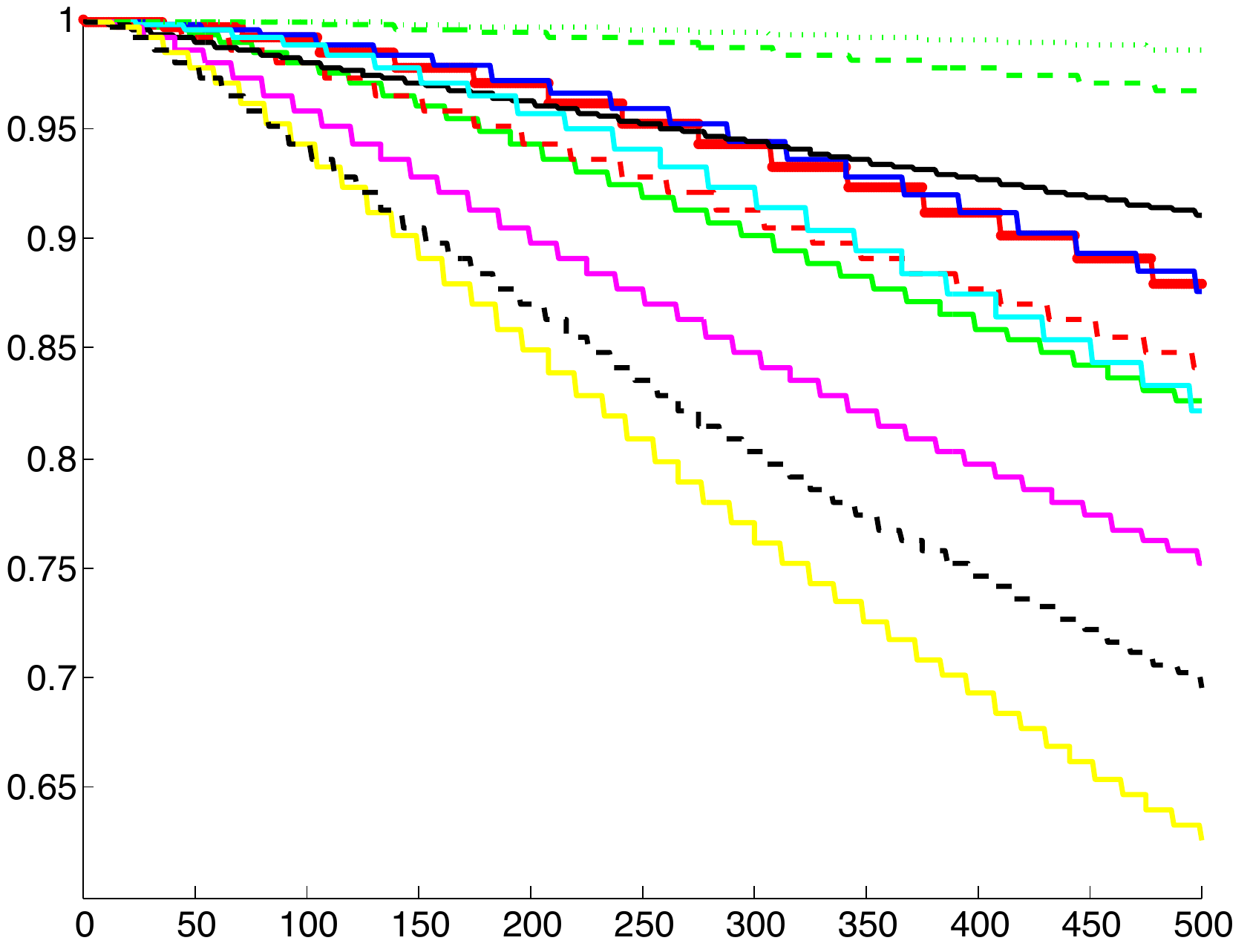} &   \includegraphics[width=\gwidth]{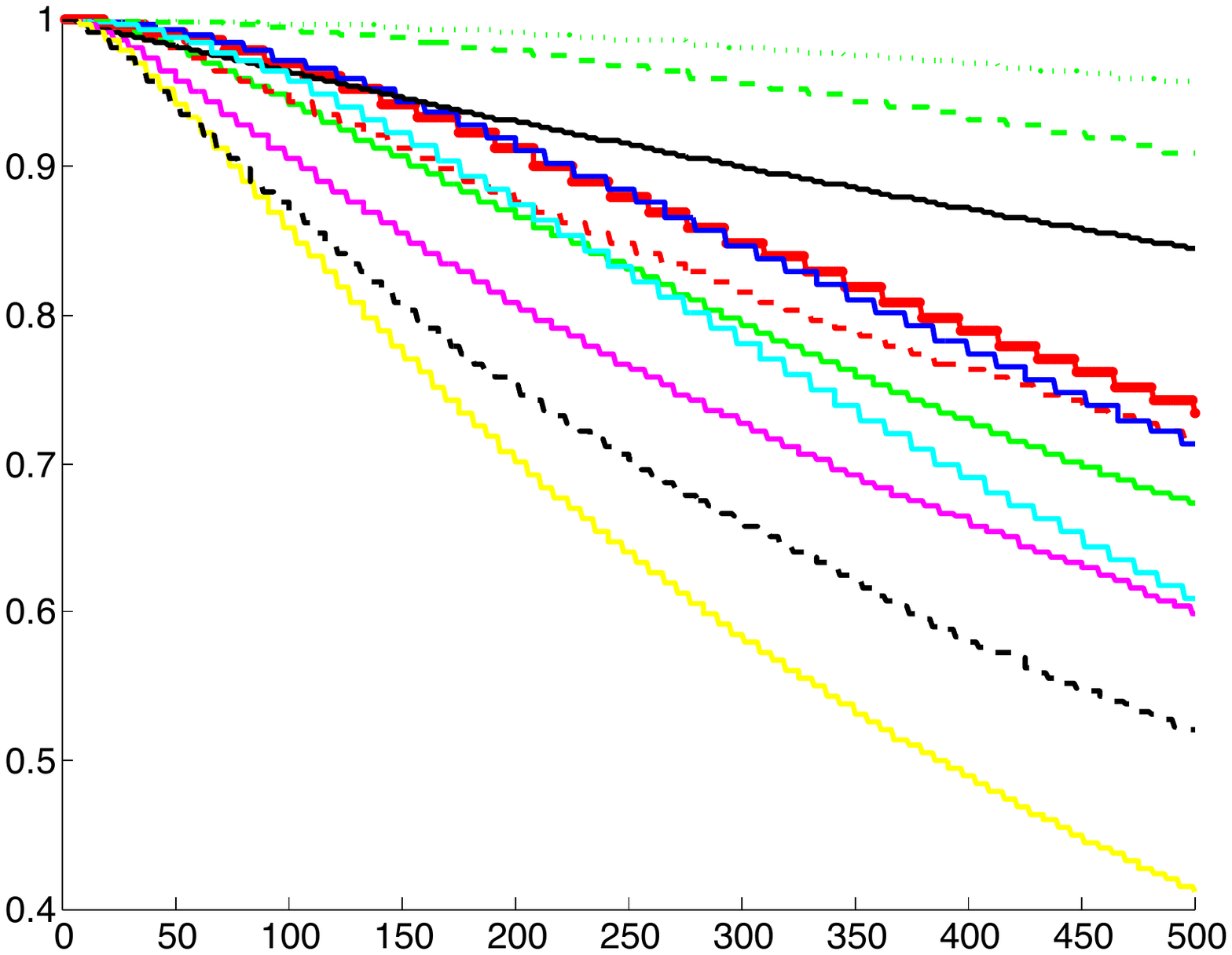} &   \includegraphics[width=\gwidth]{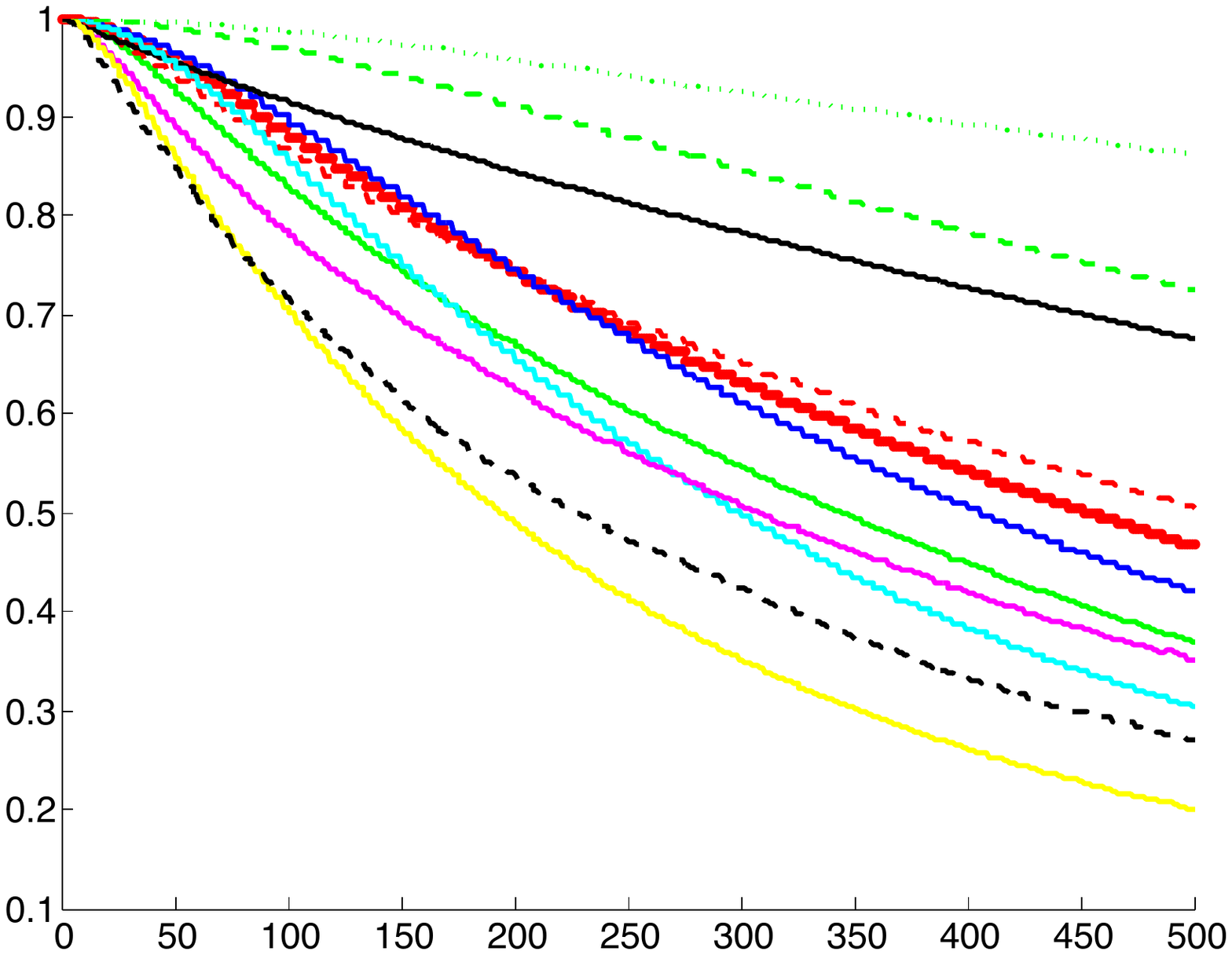} \\
 $c=0.05$ & $c=0.1$ & $c=0.25$ \\
    \includegraphics[width=\gwidth]{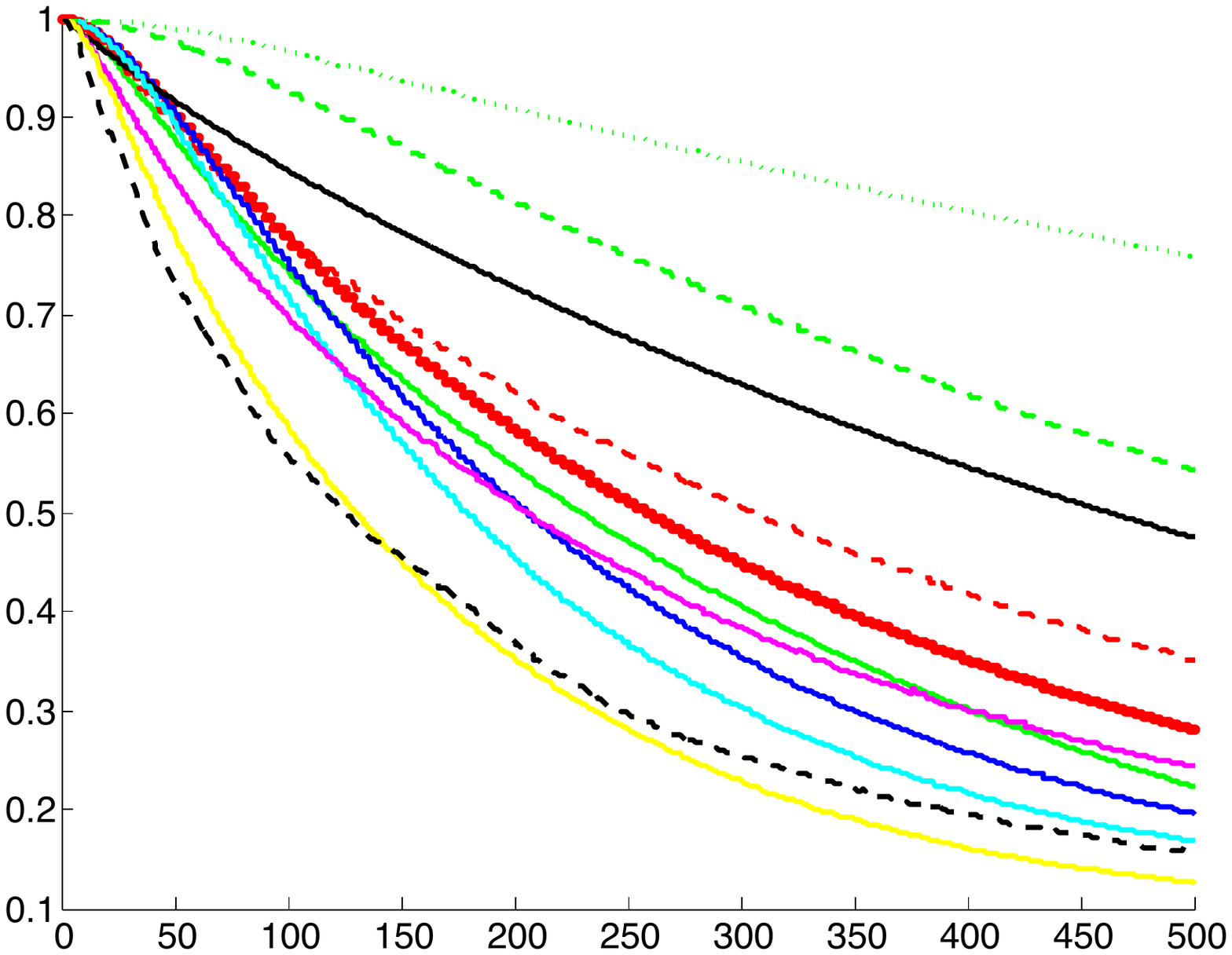} &   \includegraphics[width=\gwidth]{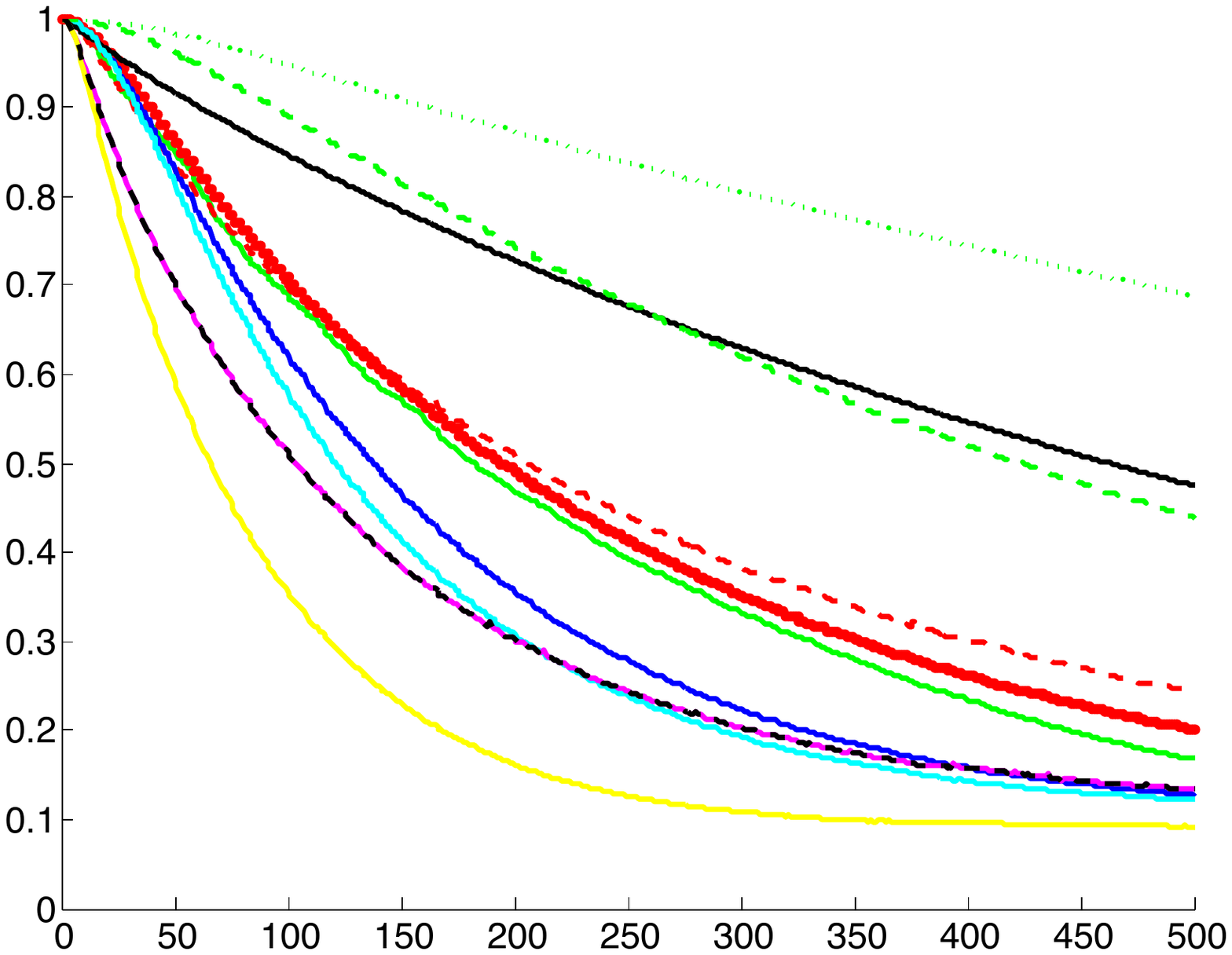} &   \includegraphics[width=\gwidth]{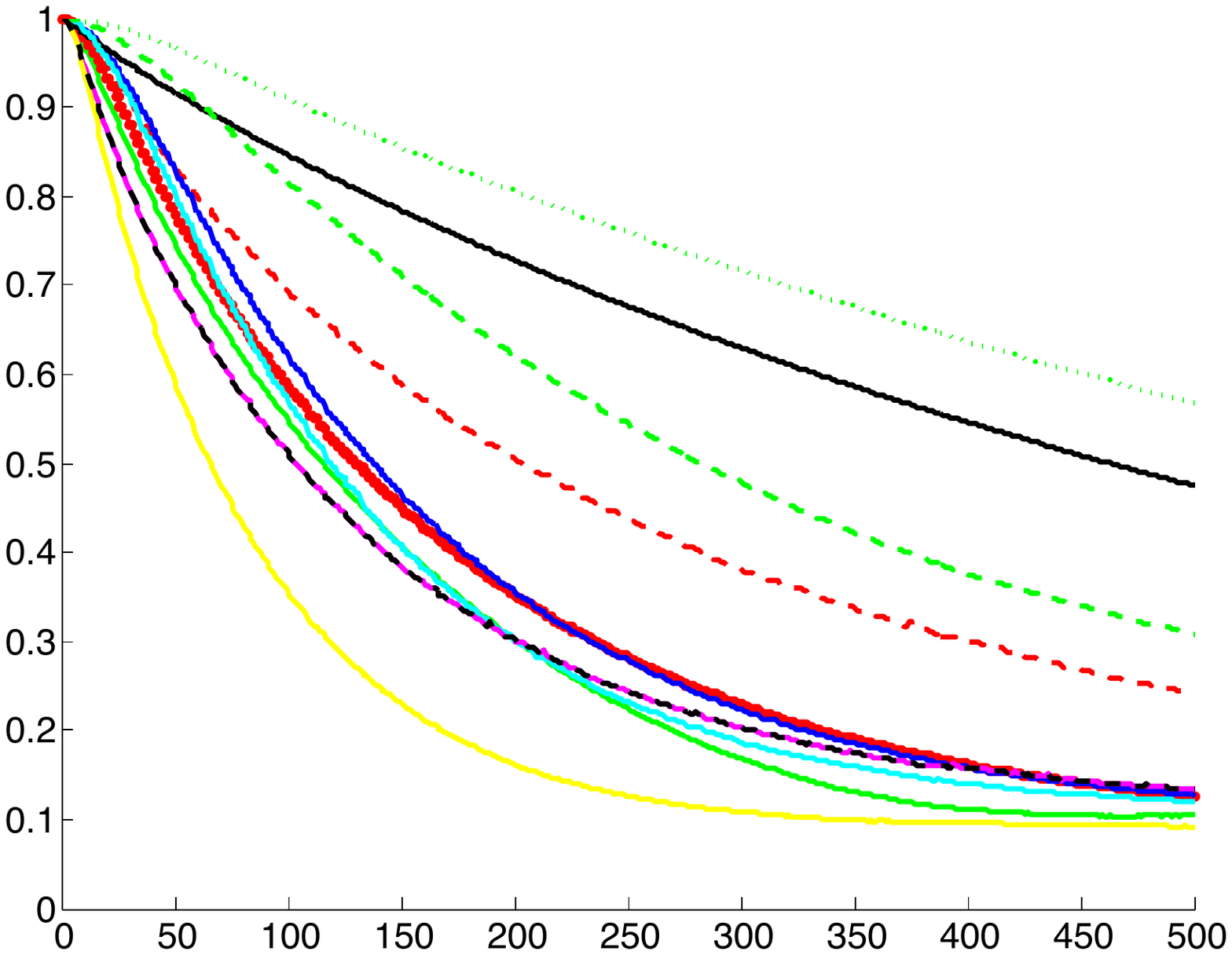}\\
 $c=0.5$ & $c=0.75$ & $c=1.0$ \\
  \includegraphics[width=\gwidth]{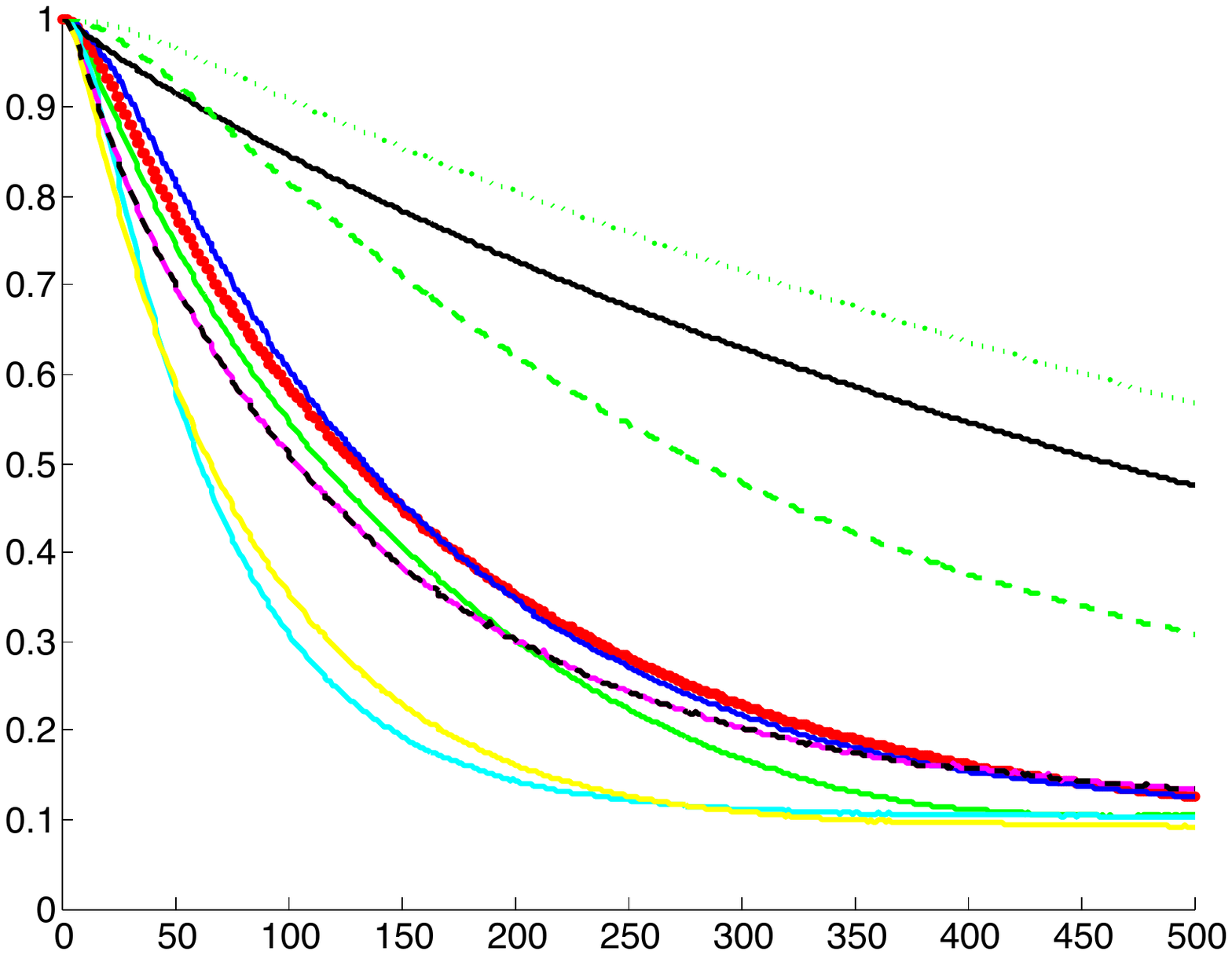} &   \includegraphics[width=\gwidth]{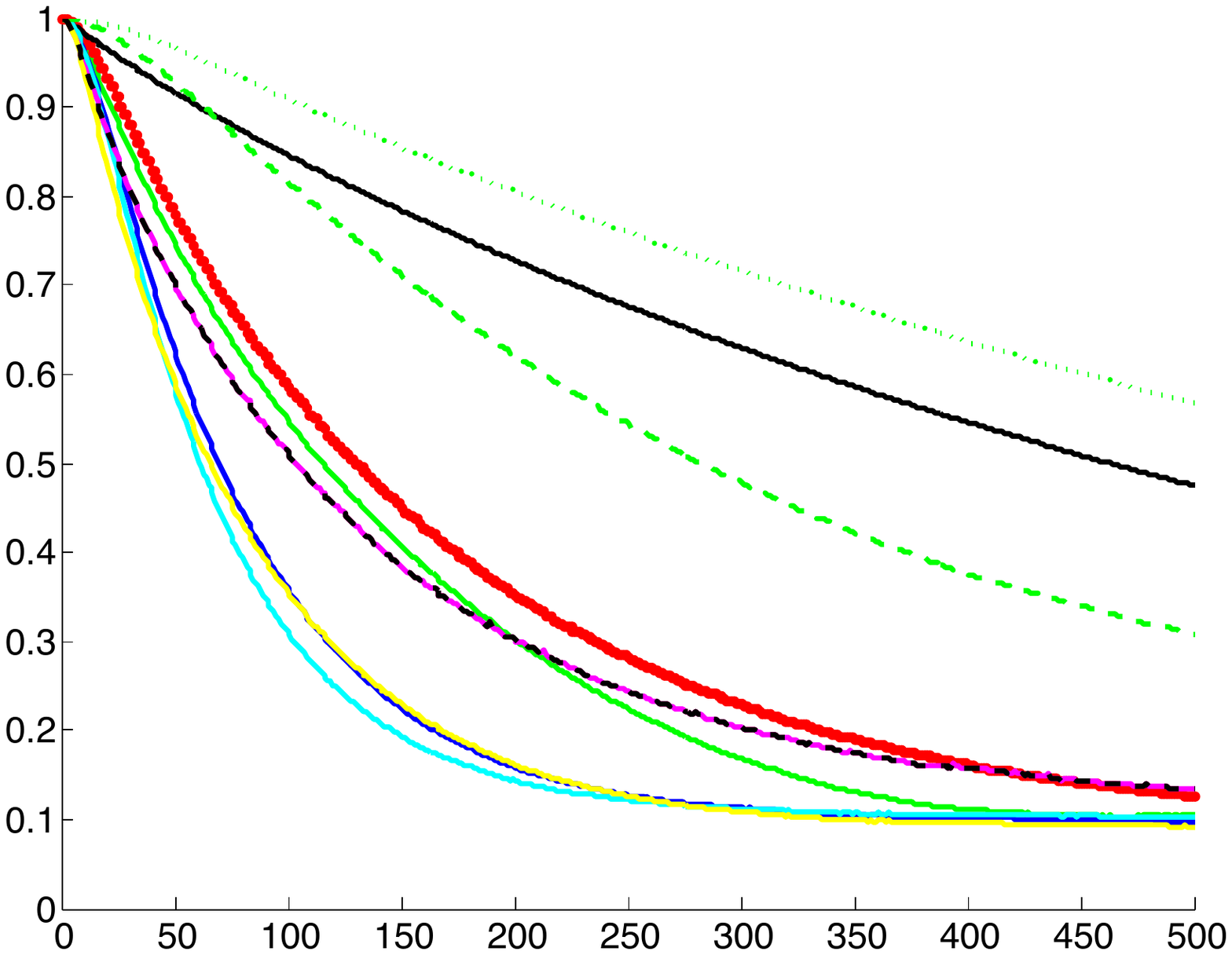} &   \includegraphics[width=\gwidth]{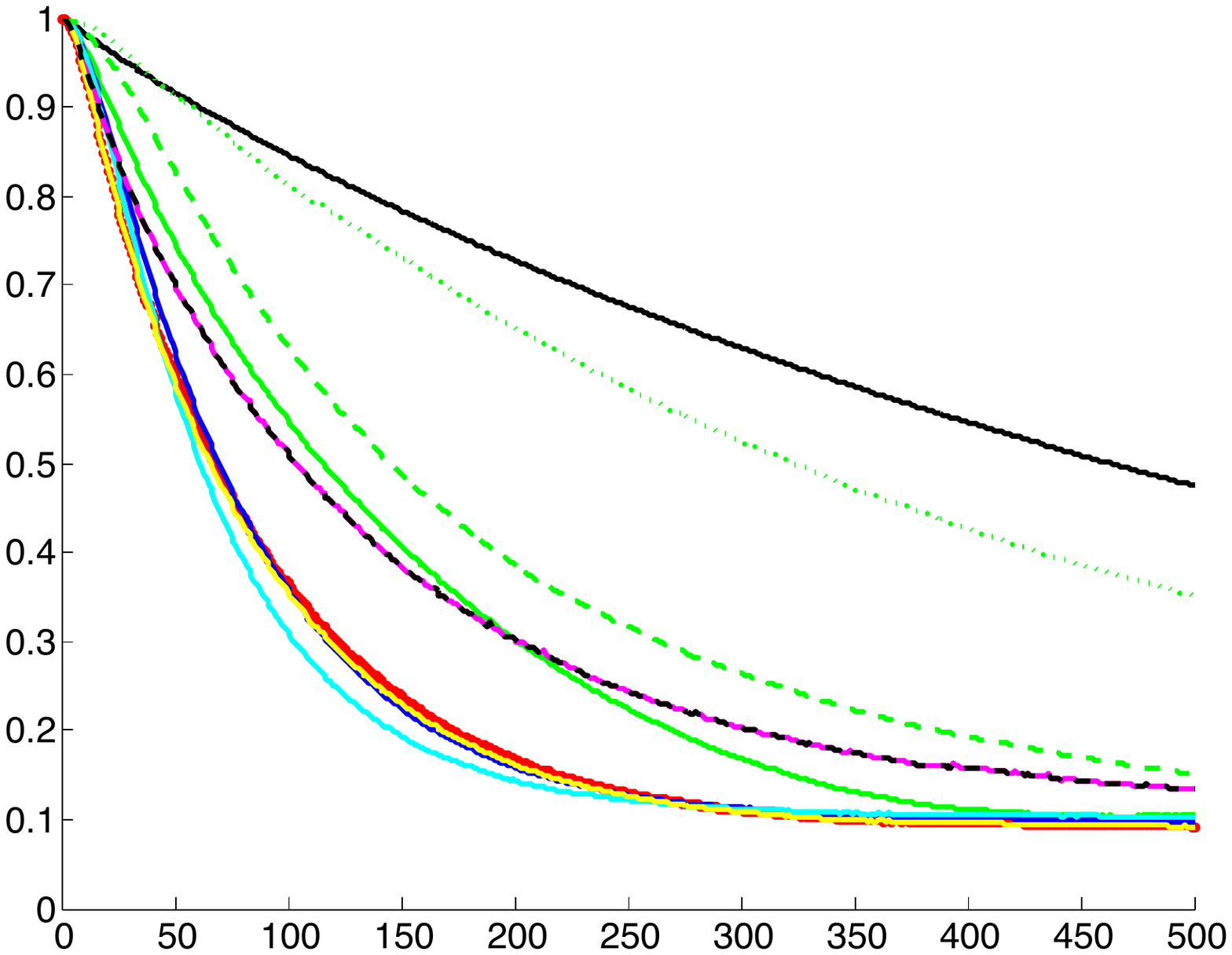}\\
 $c=1.25$ & $c=1.5$ & $c=1.75$ \\
 \end{tabular}
 \caption{Runtime analysis in on-policy random MDPs, with tabular features. Once the time per iteration is increased to 1.75 milliseconds, we obtain the original learning curve graphs: there are no runtime restrictions on the algorithms at that point since they are all fast enough with so much time per second. The line style and colors correspond exactly with the labels in the main paper. For a detailed discussion of the figure see the appendix text.}\label{figure_onruntime}
 \end{figure*}
 
  \begin{figure*}
 \hspace*{-0.75cm}
\begin{tabular}{ccc}
   \includegraphics[width=\gwidth]{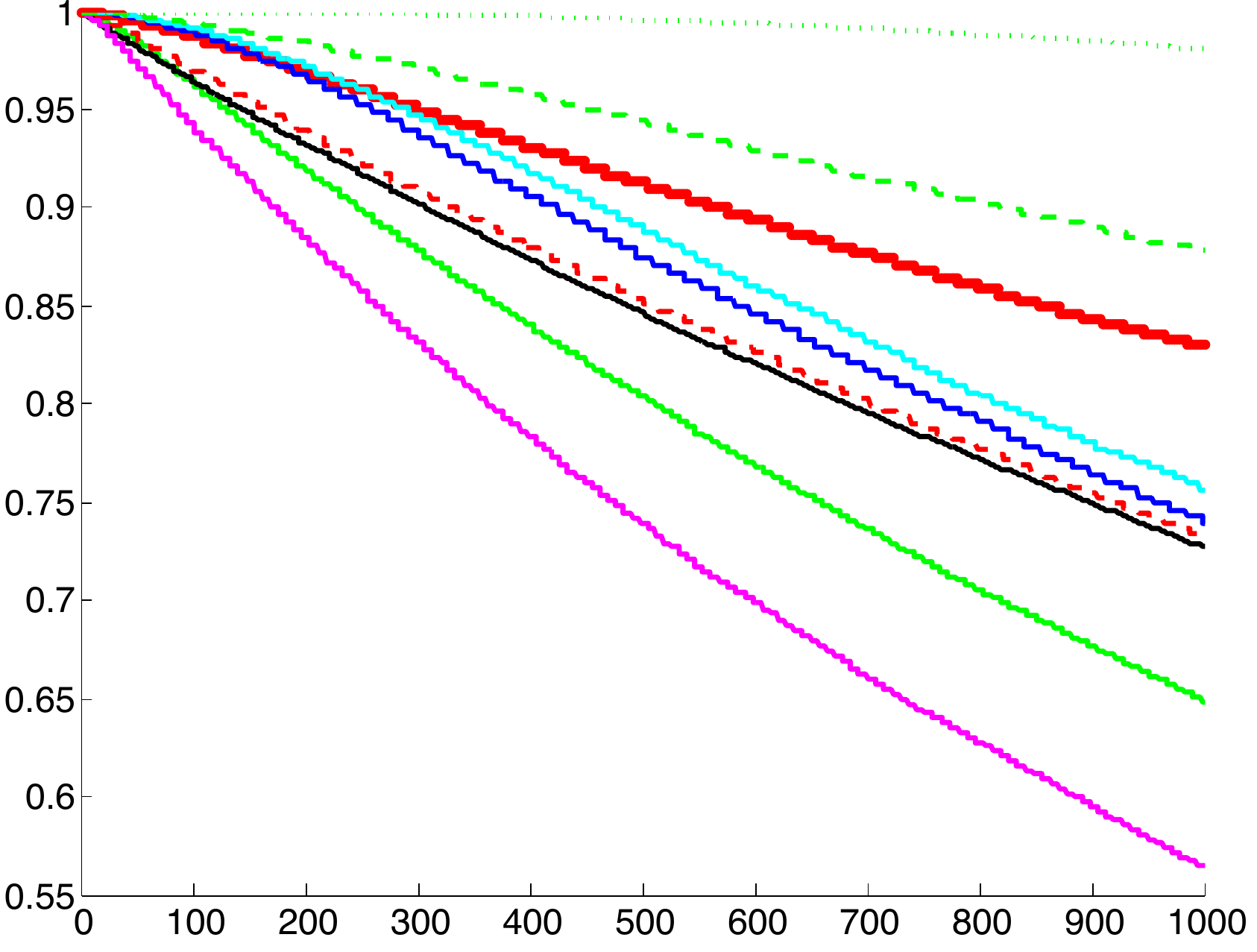} &   \includegraphics[width=\gwidth]{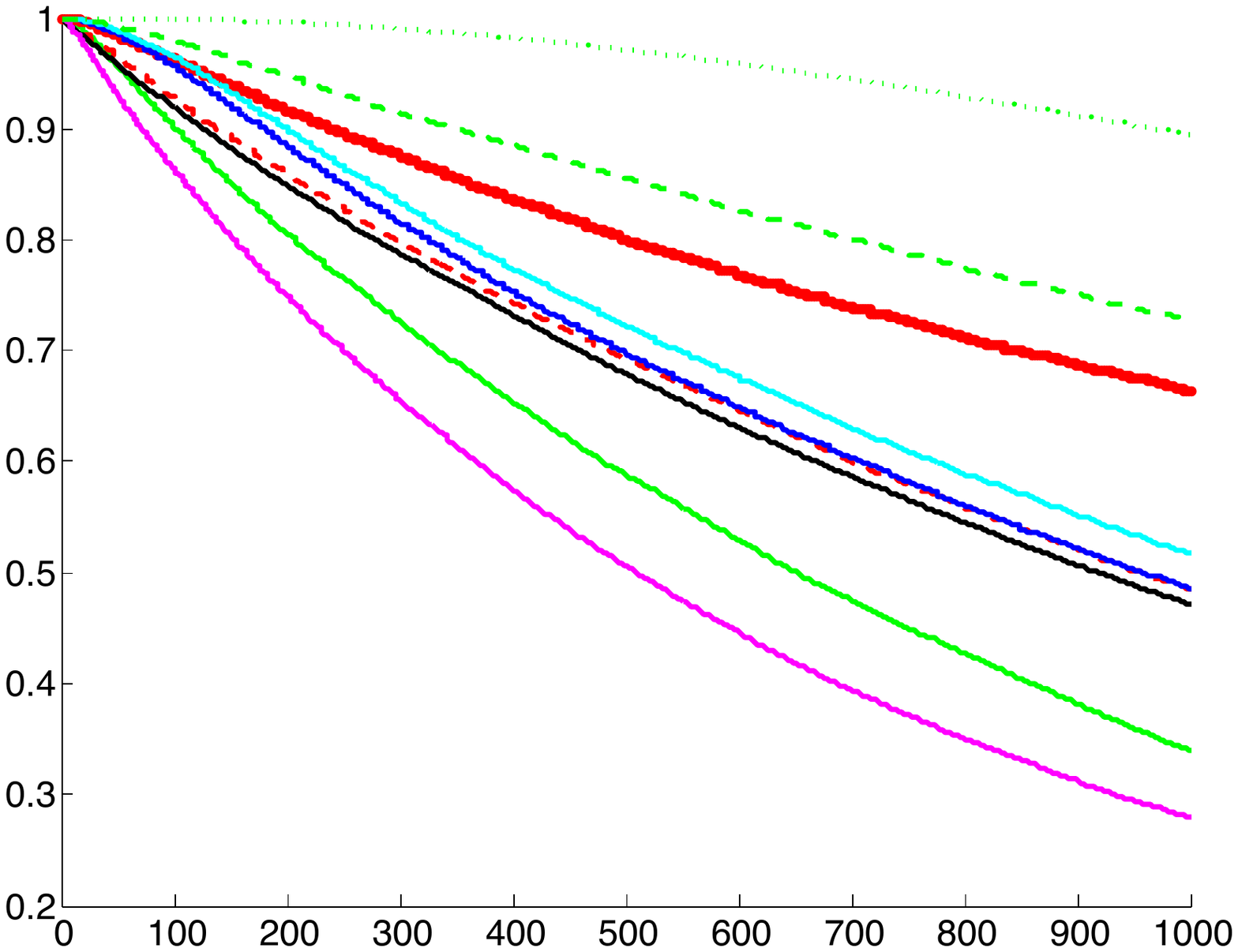}  &    \includegraphics[width=\gwidth]{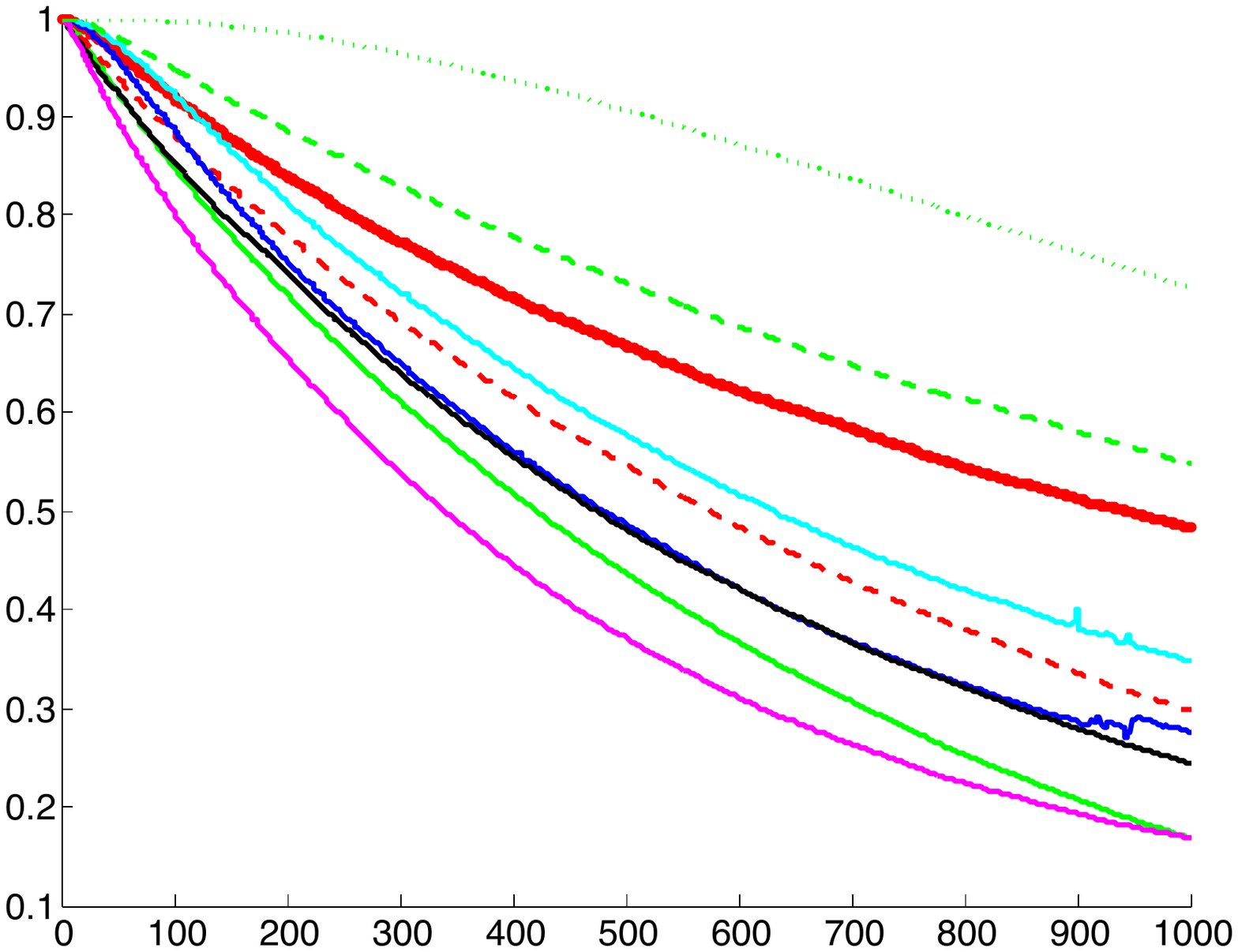} \\
 $c=0.1$ & $c=0.25$ & $c=0.5$\\
   \includegraphics[width=\gwidth]{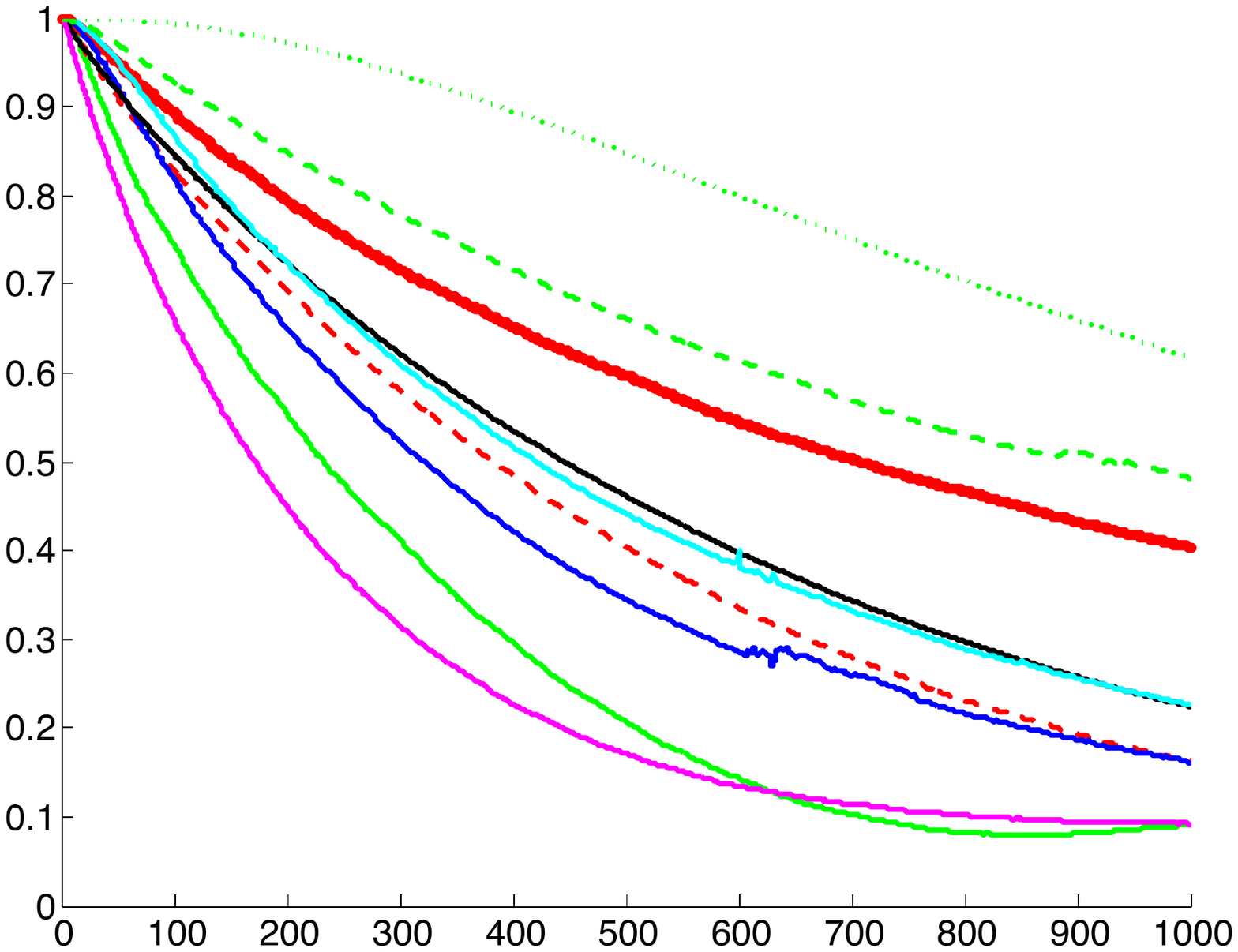} &   \includegraphics[width=\gwidth]{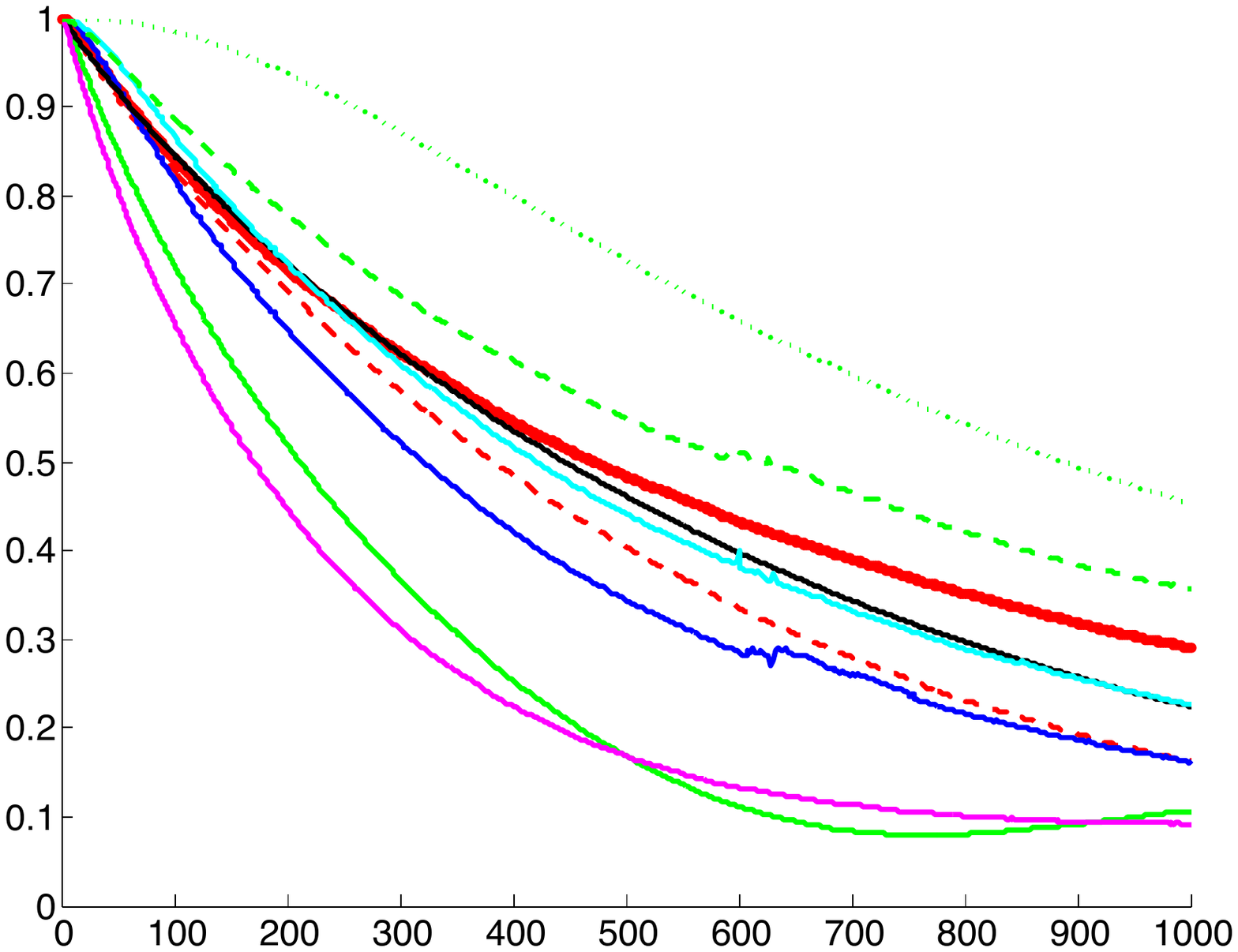} &     \includegraphics[width=\gwidth]{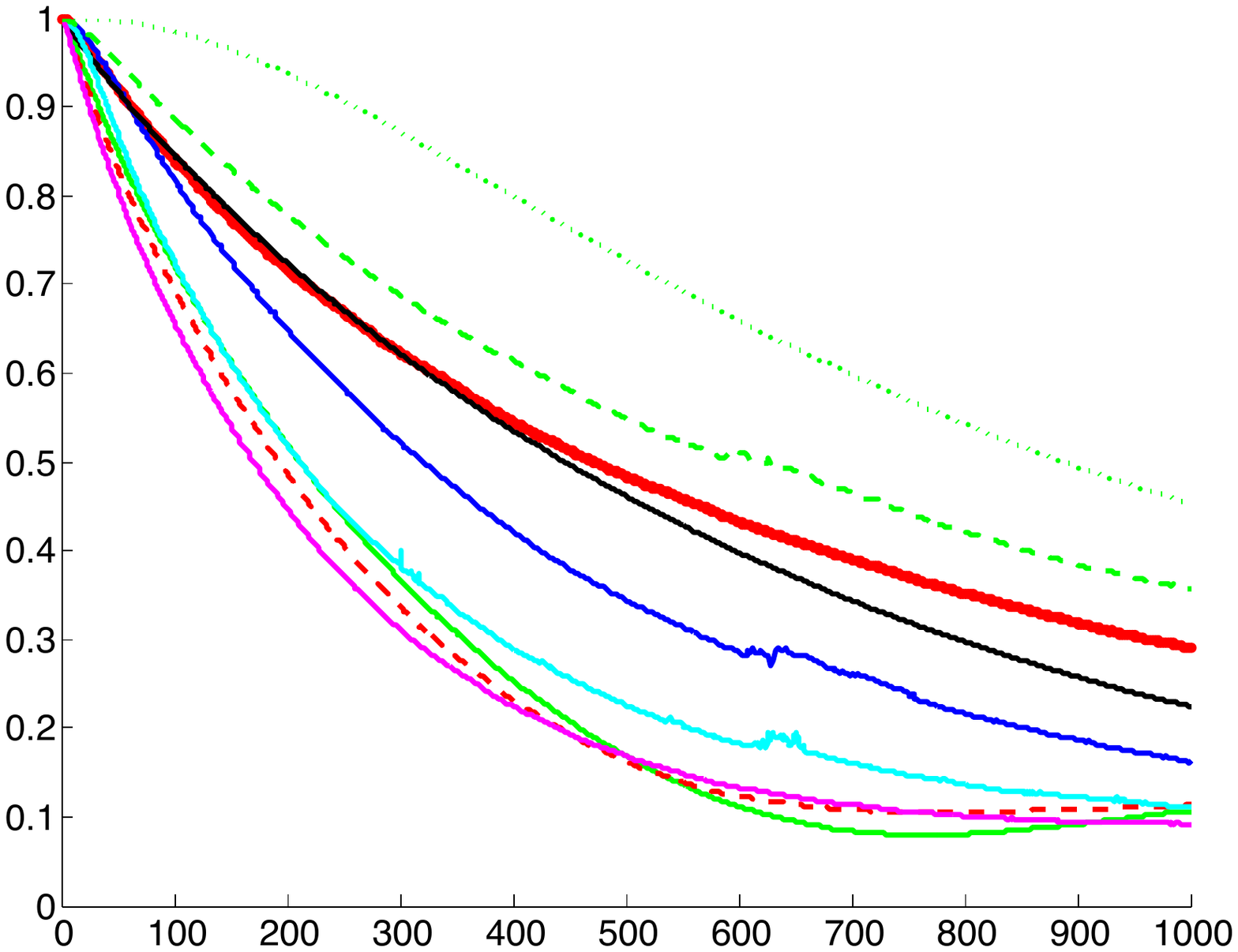} \\
  $c=0.75$ & $c=1.0$&$c=1.25$ \\
     \includegraphics[width=\gwidth]{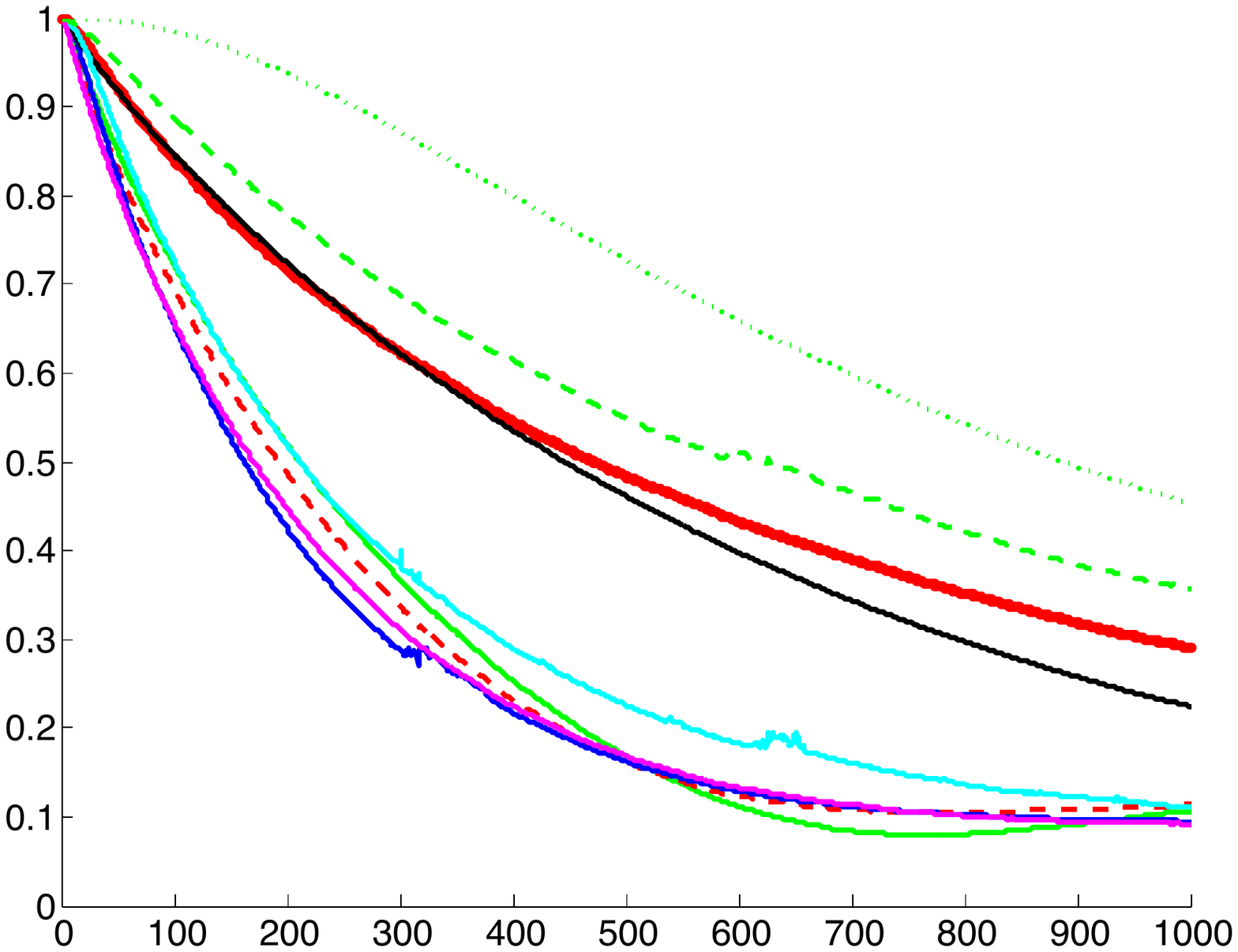} &   \includegraphics[width=\gwidth]{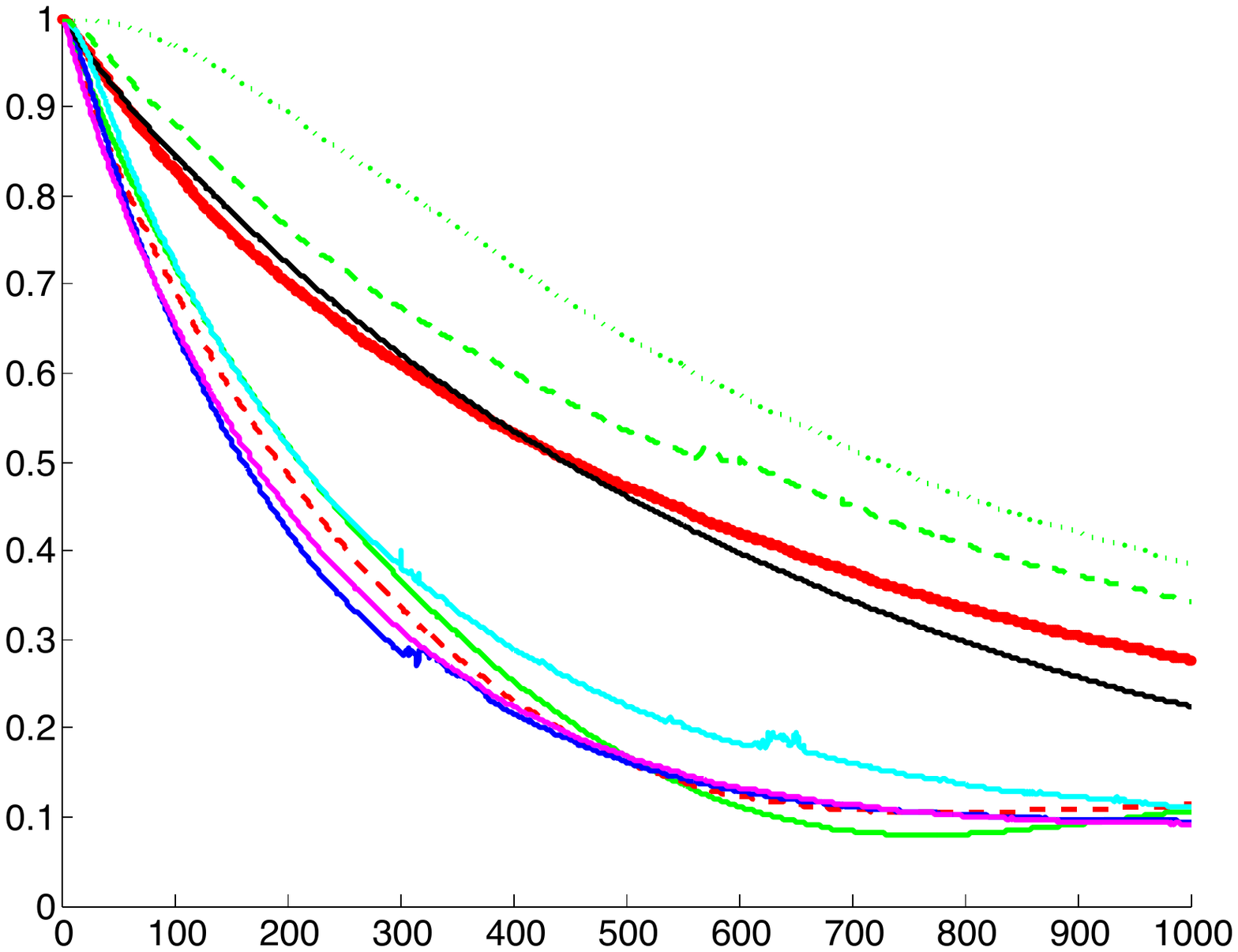} &  \includegraphics[width=\gwidth]{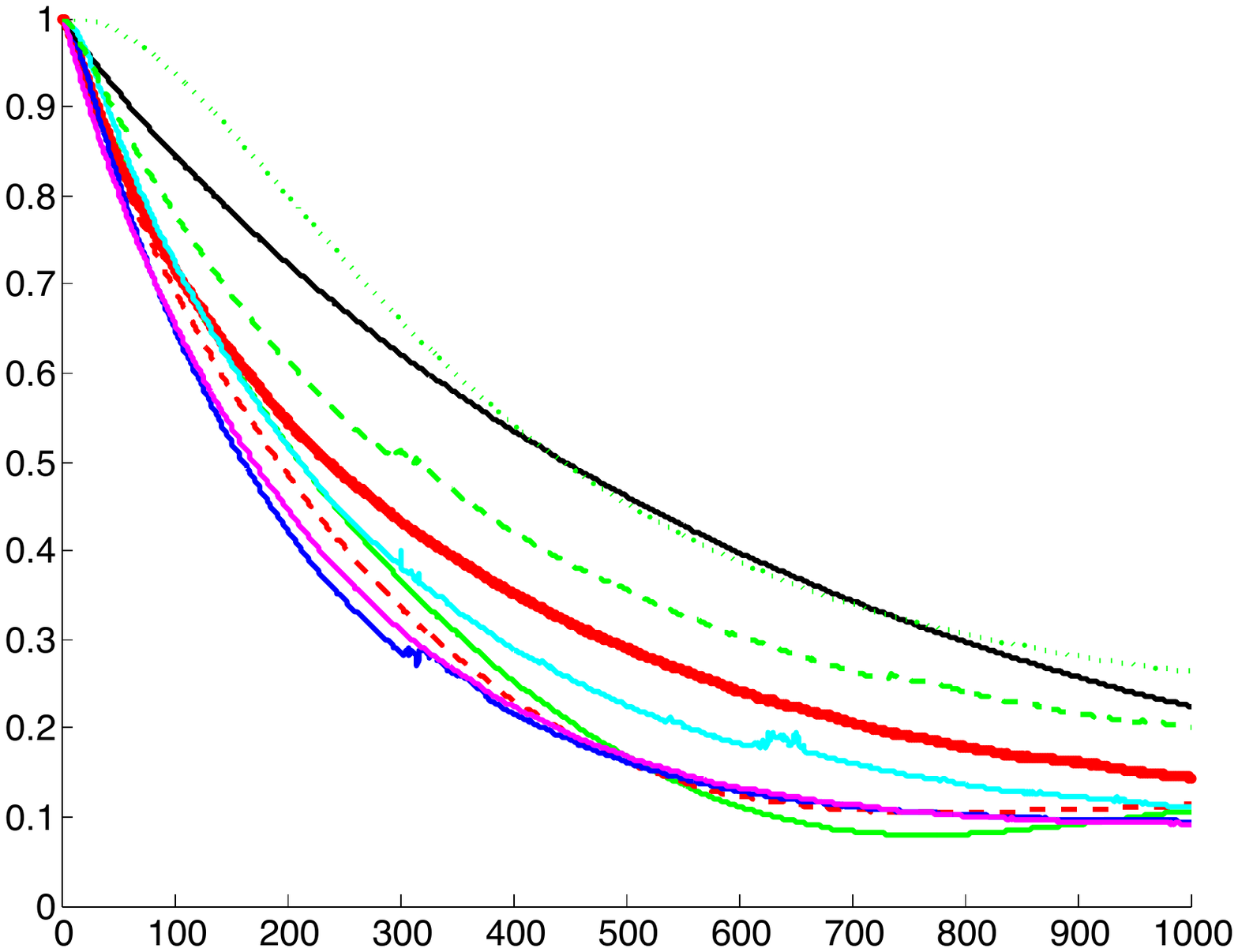}\\
  $c=1.5$ & $c=1.75$ & $c=2.0$ \\
 \end{tabular}
 \caption{Runtime analysis in off-policy random MDPs, with tabular features. Once the time per iteration is increased to 2 milliseconds, we obtain the original learning curve graphs: there are no runtime restrictions on the algorithms at that point since they are all fast enough with the time allotted . The line style and colors correspond exactly with the labels in the main paper. For a detailed discussion of the figure see the appendix text.}\label{figure_offruntime}
 \end{figure*}


\end{document}